\documentclass{article}

\usepackage[T1]{fontenc}    
\usepackage{url}            
\usepackage{booktabs}       
\usepackage{amsfonts}       
\usepackage{nicefrac}       
\usepackage{microtype}      
\usepackage{xcolor}         

\usepackage[utf8]{inputenc}
\usepackage{mycommands}
\usepackage{geometry}
\usepackage{cancel}
\usepackage{hyperref}
\usepackage{enumitem}
\usepackage{wrapfig}

\usepackage{etoc}

\newcommand{\mc}[1]{\mathcal{#1}}

\newcommand{\lalit}[1]{\textcolor{blue}{}}

\newcommand{\pil}{\widehat{\pi}_{l}}
\newcommand{\pill}{\widehat{\pi}_{l-1}}
\newcommand{\1}{\mathbf{1}}
\newcommand{\supp}{\operatorname{supp}}
\newcommand{\inlineeqnum}{\refstepcounter{equation}~~\mbox{(\theequation)}}
\renewcommand{\v}{r} 
\newcommand{\He}{\operatorname{hess}}
\newcommand{\footremember}[2]{%
    \footnote{#2}
    \newcounter{#1}
    \setcounter{#1}{\value{footnote}}%
}

\usepackage{minitoc}

\title{Instance-optimal PAC Algorithms for Contextual Bandits}
\author{%
  Zhaoqi Li\footremember{alley}{Department of Statistics, University of Washington, \texttt{zli9@uw.edu}}%
  \and Lillian Ratliff\footremember{trailer}{Department of Electrical and Computer Engineering, University of Washington, \texttt{ratliffl@uw.edu}}%
  \and Houssam Nassif\footnote{Amazon, \texttt{houssamn@amazon.com}}%
  \and Kevin Jamieson\footnote{Allen School of Computer Science \& Engineering, University of Washington, \texttt{jamieson@cs.washington.edu}}%
  \and Lalit Jain\footnote{School of Business, University of Washington,  \texttt{lalitj@uw.edu}}
}

\begin{document}

\doparttoc 
\faketableofcontents 

\maketitle
\begin{abstract}
In the stochastic contextual bandit setting, regret-minimizing algorithms have been extensively researched, but their instance-minimizing best-arm identification counterparts remain seldom studied.
In this work, we focus on the stochastic bandit problem in the $(\epsilon,\delta)$-\textit{PAC} setting: given a policy class $\Pi$ the goal of the learner is to return a policy $\pi\in \Pi$ whose expected reward is within $\epsilon$ of the optimal policy with probability greater than $1-\delta$. We characterize the first \textit{instance-dependent} PAC sample complexity of contextual bandits through a quantity $\rho_{\Pi}$, and provide matching upper and lower bounds in terms of $\rho_{\Pi}$ for the agnostic and linear contextual best-arm identification settings. We show that no algorithm can be simultaneously minimax-optimal for regret minimization and instance-dependent PAC for best-arm identification. Our main result is a new instance-optimal and computationally efficient algorithm that relies on a polynomial number of calls to an argmax oracle.

\end{abstract}

\section{Introduction}


We consider the stochastic contextual bandit problem in the PAC setting.
Fix a distribution $\nu$ over a potentially countable\footnote{Assuming the set of contexts is countable versus uncountable is for presentation purposes only, since it allow us the notational convenience of letting $\nu_c$ denote the probability of context $c$ arriving.} set of contexts $\mc{C}$.
The action space is $\mc{A}$, and for computational tractability, we assume $|\mc{A}|$ is finite.
We have a set of policies $\Pi$ of interest where each policy $\pi \in \Pi$ is a map from contexts to an action space $\pi: \mc{C} \rightarrow \mc{A}$.
The reward function is $\v: \mc{C} \times \mc{A} \rightarrow \R$.
At each time $t=1,2,\dots$ a context $c_t \sim \nu$ arrives, the learner chooses an action $a_t \in \mc{A}$, and receives reward $r_t:=r_t(c_t,a_t)\in \R$ with $\E[r_t | c_t, a_t] = \v(c_t,a_t) \in \R$. 
The value of a policy $V(\pi)$ is the expected reward from playing action $\pi(c)$ in context $c$: $V(\pi) = \E_{c \sim \nu}[ \v(c, \pi(c)) ]$.
Given a collection of policies $\Pi$, the objective is to identify the optimal policy $\pi_{\ast} := \arg\max_{\pi \in \Pi} V(\pi)$, with high probability. 
Formally, for any $\epsilon > 0$ and $\delta \in (0,1)$, we seek to characterize the sample complexity of identifying a policy $\pi \in \Pi$ such that $V(\pi) \geq V(\pi_{\ast}) - \epsilon$, with probability at least $1-\delta$. 
That is, we wish to minimize the total amount of interactions with the environment to learn an $\epsilon$-optimal policy.

We study both the \emph{agnostic} setting,  where $\Pi$ is an arbitrary set of policies with no assumed relationship with the reward function $\v(c,a)$; and the \emph{realizable} setting, where the policy class and the reward function follow a linear structure, known as the linear contextual bandit problem. 
In both cases, we are interested in \emph{instance-dependent} sample complexity bounds.
That is, the upper and lower bounds we seek do not simply depend on coarse quantities like $|\Pi|$, $|\mc{A}|$, and $1/\epsilon^2$,  but more fine-grained relationships between the context distribution $\nu$, geometry of policies $\Pi$, and the reward function $\v:\mc{C} \times \mc{A} \rightarrow \R$.
Our motivation is that instance-dependent bounds describe the difficulty of a particular problem instance, allowing optimal algorithms to adapt to the true difficulty of the problem, whether easy or hard.
We seek algorithms that take advantage of ``easy'' instances instead of optimizing for the worst-case~\cite{LogisticBound}.

\subsection{Related work}

\paragraph{Minimax regret bounds for general policy classes}
The vast majority of research in contextual bandits focuses on regret minimization.
That is, for a time horizon $T$, the goal of the player is to minimize $\E\left[ \sum_{t=1}^T \v(c_t,\pi_{\ast}(c_t)) - \v(c_t,a_t) \right]$.
The landmark algorithm EXP4 for non-stochastic multi-armed bandits~\cite{auer2002nonstochastic} achieves a regret bound of $\sqrt{ |\mc{A}| T \log(|\Pi|)}$. 
Unfortunately, the running time of EXP4 is linear in $|\Pi|$ which is prohibitive for many problems of interest. The algorithms proposed in \cite{dudik2011efficient,agarwal2014taming}  achieve the same regret bound with a computational complexity that is only polynomial in $T$ and $\log(|\Pi|)$.
Both approaches can be used to obtain an $\epsilon$-optimal policy with probability at least $1-\delta$ using a sample complexity no more than $\frac{ |\mc{A}| \log(|\Pi| /\delta)}{\epsilon^2}$.
None of these works made any assumption on the connection between the reward function $\v$ and the policy class $\Pi$ (i.e. the agnostic setting). \vspace{-10pt}

\paragraph{Instance-dependent regret bounds for general policy classes}
The epoch-greedy algorithm of~\cite{langford2007epoch} achieved the first instance-dependent bounds on regret with a coarse guarantee depending only on the minimum policy gap 
$\Delta_{\sf pol} := V(\pi_{\ast}) - \max_{\pi \neq \pi_{\ast}}V(\pi)$.
In the pursuit of more fine-grained regret bounds achievable by computationally efficient algorithms, many authors resort to the 
\textit{realizability} assumption~\cite{foster2018practical, foster2020beyond,simchi2021bypassing, foster2021instance}.
The learner knows a hypothesis class $\mc{H}$ where each $f \in \mc{H}$ is a map $f: \mc{C} \times \mc{A} \rightarrow \R$,
and there exists an  $f^{\ast}\in\mc{H}$ such that $\v(c,a) = f^{\ast}(c,a)$ for all $(c,a) \in \mc{C} \times \mc{A}$. 
Under this assumption, \cite{foster2021instance} proves lower and upper bounds on the instance-dependent regret. Their bounds are in term of the \textit{uniform gap} $\Delta_{\sf uniform} := \min_{c\in \mc{C}} \min_{a \in \mc{A}} \v(c,\pi_{\ast}(c)) - \v(c,a)$. 
In general, for any policy class, they establish matching minimax lower and upper regret bounds of the form $\min\{\sqrt{|\mathcal{A}|T\log(|\mc{H}|)}, \frac{|\mc{A}| \log( |\mc{H}|)}{\Delta_{\sf uniform}} \mathfrak{C}_{\mc{H}}^{\mathsf{pol}}\}$, 
where $\mathfrak{C}_{\mc{H}}^{\mathsf{pol}}$ is the \textit{policy disagreement coefficient}, a parameter depending on the geometry of $\mc{H}$ and the context distribution $\nu$. 
That is, these bounds hold with respect to a worst-case family of instances parameterized by $\Delta_{\textsf{uniform}}$ and $\mathfrak{C}_{\mc{H}}^{\mathsf{pol}}$.
Using the standard online-to-batch conversion, this translates to a sample complexity (i.e. the time required to find an $\epsilon$-good policy with constant probability) of roughly $\frac{|\mc{A}| \log( |\mc{H}|)}{\epsilon \ \Delta_{\sf uniform}}\mathfrak{C}_{\mc{H}}^{\mathsf{pol}}$. We show in Corollary~\ref{cor:disagree_coeff} that this sample complexity is at least as large as our bounds. 
Further, unlike our bounds below, this sample complexity is unbounded as $\epsilon$ goes to $0$. Recent work refines these kinds of regret bounds further, and provides minimax regret bounds in terms of the \textit{decision-estimation coefficient}~\cite{foster2021statistical}.


\paragraph{Regret bounds for linear contextual bandits}
A special case of the realizable case assumes a linear structure for $\mc{H}$. Assume there exists a known feature map $\phi : \mc{C} \times \mc{A} \rightarrow \R^d$ and an unknown $\theta_{\ast} \in \R^d$ such that the true reward function is given as $\v(c,a) = \langle \phi(c,a) , \theta_{\ast} \rangle$.
For this setting, popular optimism-based algorithms like LinUCB \cite{li2010contextual} and Thompson sampling \cite{russo2016simple,EBNabi} achieve a regret bound of $\min\{ d\sqrt{T}, \frac{d^2}{\Delta_{\sf uniform}} \}$ \cite{abbasi2011improved}.
Appealing to the online-to-batch conversion, this translates to a PAC guarantee of $\frac{d^2}{\epsilon \ \Delta_{\sf uniform}}$.
More precise instance-dependent upper bounds on regret match instance-dependent lower bounds asymptotically as $T\rightarrow \infty$~\cite{hao2020adaptive,tirinzoni2020asymptotically}. 
These works are most similar to our setting and have qualitatively similar style algorithms. 
However, both approaches rely on asymptotics with large problem-dependent terms that may  dominate the bounds in finite time. 
Our work is focused on upper bounds that nearly match lower bounds for all finite times. \vspace{-10pt}

\paragraph{PAC sample complexity for contextual bandits}
As we will describe, all contextual bandits with an arbitrary policy class can be reduced to PAC learning for linear bandits. 
Once we made this reduction, our sample complexity analysis draws inspiration from the nearly instance-optimal algorithm for linear best-arm identification~\cite{fiez2019sequential}. 
PAC sample complexity of linear contextual bandits was also studied in~\cite{zanette2021design}, who shows a minimax guarantee sample complexity that scales with $\frac{d^2}{\epsilon^2}\log(1/\delta)$.
In their approach, \cite{agarwal2014taming} define their action sampling distribution as a convex combination over policies. 
Our sampling distribution, as well as the optimal sampling distribution, cannot be represented this way and is actually derived from the dual of the optimal experimental design objective.

\subsection{Contributions}
In this work, our contributions include:
\begin{enumerate}[wide, labelwidth=0pt, labelindent=0pt]
    \item 
    In the agnostic setting, we introduce a quantity $\rho_{\Pi}$ that characterizes the instance-dependent sample complexity of PAC learning for contextual bandits (see Equation~\ref{eqn:rho_star_combinatorial}).
    We show that $\rho_{\Pi}$ appears in an information theoretic lower bound on the sample complexity of any PAC algorithm as $\epsilon \rightarrow 0$ in Theorem~\ref{thm:lower_bound}. To ground this, we describe it carefully in the setting of the trivial policy class (Section~\ref{sec:trivial}) and linear policy classes (Section~\ref{sec:linear}). To do so, we reduce agnostic contextual bandits to the realizable linear case (also establishing matching upper and lower bounds in this setting).
    \item We construct an instance on which any regret minimax-optimal algorithm necessarily has a sample complexity that scales quadratically with the optimal sample complexity (Theorem~\ref{thm:inefficient-low-regret}). This shows that no algorithm can be both regret minimax-optimal and instance-optimal PAC.
    \item Finally, we propose Algorithm~\ref{alg:evaluation_oracle} whose sample complexity nearly matches the lower bound based on $\rho_{\Pi}$. By appealing to an argmax oracle, this algorithm has a runtime polynomial in $\rho_{\Pi}$, $1/\epsilon$, $\log(1/\delta)$, $|\mc{A}|$, and $\log(|\Pi|)$, assuming a unit cost of invoking the oracle. 
  
\end{enumerate}

\section{Problem statement and main results}\label{sec:results}
More formally, define $\mc{F}_t = \sigma(c_1,a_1,r_1,\dots,c_t,a_t,r_t)$ as the natural $\sigma$-algebra filtration capturing all observed random variables up to time $t$. 
At each time $t$ an \textit{algorithm} defines a \emph{sampling rule} $\mc{F}_t \mapsto \mc{A}$ which defines $a_{t+1}$, an $\mc{F}_t$-measurable stopping time $\tau \in \mathbb{N}$, and a \emph{selection rule} $\mc{F}_t \mapsto \Pi$ that is only called once at the stopping time $t=\tau$.

\begin{definition}
Fix $\epsilon \geq 0$ and $\delta \in (0,1)$.
We say an algorithm is $(\epsilon,\delta)$-PAC for contextual bandits with policy class $\Pi$, if at the stopping time $\tau \in \mathbb{N}$ with $\E[\tau] < \infty$, the algorithm outputs $\widehat{\pi} \in \Pi$ satisfying $\P( V(\widehat{\pi}) \geq \max_{\pi \in \Pi} V(\pi) - \epsilon ) \geq 1-\delta$.  
\end{definition}


\noindent The \emph{sample complexity} of an $(\epsilon,\delta)$-PAC algorithm for contextual bandits is the time at which the algorithm stops and outputs $\widehat{\pi}$.
The following quantity governs the sample complexity :
\begin{align}
    \rho_{\Pi,\epsilon}(\Pi, v) := \min_{p_c \in \triangle_{\mc{A}}, \ \forall c \in \mc{C} }  \max_{\pi \in \Pi \setminus \pi_* } \frac{ \E_{c \sim \nu} \left[ \left(\frac{1}{p_{c,\pi(c)}}+ \frac{1}{p_{c,\pi_*(c)}}\right) \1\{\pi_*(c) \neq \pi(c) \} \right] }{(\E_{c \sim \nu}[ \, \v(c,\pi_*(c)) - \v(c,\pi(c)) \, ] \vee \epsilon)^2}. \label{eqn:rho_star_combinatorial}
\end{align}
Here, for any countable set $\mc{X}$ we have that $\triangle_{\mc{X}} = \{ p \in \R^{|\mc{X}|} : \sum_{x \in \mc{X}} p_x = 1, p_x \geq 0 \ \forall x \in \mc{X} \}$ so that $p_c$ for every $c \in \mc{C}$ defines a probability distribution over actions $\mc{A}$. In addition we use the notation $a \vee b := \max\{a,b\}$. 
We begin with a necessary condition on the sample complexity for the particular case of exact policy identification ($\epsilon=0$).
\begin{theorem}[Lower bound]\label{thm:lower_bound}
Fix $\epsilon=0$ and $\delta \in (0,1)$. 
Moreover, fix a contextual bandit instance $\mu = (\nu,r)$ and a collection of policies $\Pi$.
Then any $(0,\delta)$-PAC algorithm for contextual bandits satisfies $\E_\mu[\tau] \geq \rho_{\Pi,0} \log(1/2.4 \delta)$.
\end{theorem}
\noindent The proof of the lower bound follows from standard information theoretic arguments \cite{kaufmann2016complexity}.
The lower bound implicitly applies to learners that know the distribution $\nu$ precisely.
In practice, such knowledge would never be available however the learner may have a large dataset of offline data.
\begin{assumption}\label{asm:offline_data}
Prior to starting the game, the learning algorithm is given a large dataset of contexts $\mc{D} = \{c_t\}_{t=1}^T$, where each $c_t\overset{i.i.d.}{\sim}\nu$ for all $t \in [T]$, and $T = O(\text{poly}(1/\epsilon,|\mc{A}|,\log(1/\delta),\log(|\Pi|)))$.
\end{assumption}
\noindent The above only assumes access to samples from the context distribution, not rewards or the value function. 
Importantly, since $\mc{C}$ could be uncountable, we do not assume $\mc{D}$ covers the support of $\nu$. 
Assumption~\ref{asm:offline_data} is satisfied, for example, in an e-commerce setting where the context is the demographic information about visitors to the site for which massive troves of historical data may be available.
Other works in PAC learning have made similar assumptions \cite{huang2015efficient}.
We would like our algorithm to be computationally efficient in the sense that it makes a polynomial number of calls to what we refer to as argmax oracle. Such an assumption is common in the contextual bandits literature~\cite{agarwal2014taming,krishnamurthy2017active, dudik2011efficient}.

\begin{definition}[Argmax oracle (\textsf{AMO})]\label{def:csc}
 The oracle $\textsf{AMO}(\Pi, \{(c_t, s_t)\}_{t=1}^n)$ is an algorithm that given contexts and cost vectors $(c_1,s_1), \cdots, (c_n,s_n)\in\C\times \R^{|\mc{A}|}$, returns $\underset{\pi \in \Pi}{\arg\max} \sum_{t=1}^{n} s_t\left(\pi\left(c_t\right)\right)$. The constrained argmax oracle, denoted as \textsf{C-AMO}, given an upper bound $l$ on the loss, returns $\underset{\pi \in \Pi}{\arg\max} \sum_{t=1}^{n} s_t\left(\pi\left(c_t\right)\right)$ subject to $\sum_{t=1}^{n} s_t\left(\pi\left(c_t\right)\right)\leq l$.
\end{definition}

In general we can implement \textsf{AMO} by calling to cost-sensitive classification~ \cite{dudik2011efficient,beygelzimer2005error} and \textsf{C-AMO} through a Lagrangian relaxation and a cost-sensitive classification oracle~\cite{agarwal2018reductions, cotter2019training}.
Our algorithm uses an argmax oracle as a subroutine at most a polynomial number of times in $\epsilon^{-1}, \log(1/\delta), |\mathcal{A}|$ and $\log(|\Pi|)$. In this sense, it is computationally efficient.
The following sufficiency result holds for general $\epsilon \geq 0$.
\begin{theorem}[Upper bound]\label{thm:upper_informal}
Fix $\epsilon \geq 0$ and $\delta \in (0,1)$.
Under Assumption~\ref{asm:offline_data}, there exists a computationally efficient $(\epsilon,\delta)$-PAC algorithm for contextual bandits that satisfies a sample complexity of $\tau \leq \rho_{\Pi,\epsilon} \log( |\Pi|\log_2(1/\epsilon) /\delta ) \log(1/\Delta_\epsilon)$, where $\Delta_\epsilon = \max\{ \epsilon, \min_{\pi \in \Pi \setminus \pi_*}V(\pi_*)-V(\pi) \}$.
Furthermore, this sample complexity never exceeds $\frac{|\mc{A}|(\log(|\Pi|)+\log(1/\delta)) \log(1/\epsilon)}{\epsilon^2}$.
\end{theorem}
The second part of the theorem follows from the first, since $\rho_{\Pi,\epsilon} \leq 2 |\mc{A}| /\epsilon^2$ by taking $p_{c,a} = 1/|\mc{A}|$ for all $(c,a) \in \mc{C}\times\mc{A}$.

\subsection{Inefficiency of low-regret algorithms}
Computationally efficient algorithms are known to exist, such as ILOVETOCONBANDITS \cite{agarwal2014taming}, which achieve a minimax-optimal cumulative regret of $\sqrt{ T |\mc{A}| \log( |\Pi| /\delta) }$.
Inspecting the proof in~\cite{agarwal2014taming}, one can extract a sample complexity of $\epsilon^{-2} |\mc{A}| \log(|\Pi|/\delta)$ from such results (which is also minimax optimal for PAC). 
The previous section showed that the sample complexity of our algorithm, Theorem~\ref{thm:upper_informal}, nearly matches the instance-dependent lower bound of Theorem~\ref{thm:lower_bound}. In other words, our algorithm achieves a nearly optimal instance-dependent PAC sample complexity. 
However, it is natural to wonder if perhaps with a tighter analysis, the minimax regret optimal algorithm in \cite{agarwal2014taming} also obtains the instance-optimal PAC sample complexity.
In this section, we show that this is not the case.
Indeed, we show that \emph{any} algorithm that is minimax regret optimal must have a sample complexity that is at least quadratic in the optimal PAC sample complexity of some instance. 

\begin{definition}[Hard instance]\label{def:hard_instance}
Fix $m \in \mathbb{N}$, $\Delta \in (0,1]$ and let $\mc{C}=[m]$, $\mc{A} = \{0,1\}$. 
For $i=1,\dots,m$, let $\pi_i(j) = \1\{ i = j\}$ and define $\v(i,j) = \Delta \1\{ j = \pi_1(i) \}$. 
Then $V(\pi_1) = \Delta$ and $V(\pi_i)= \Delta( 1 - 2 / m)$ for all $i \in \mc{C} \setminus \{1\}$. 
\end{definition}

\noindent Note that for the hard instance, $m = |\Pi|$. If observations are corrupted by $\mc{N}(0,1)$ additive noise, then a straightforward calculation shows that $\rho_{\Pi,0}(\Pi, v) = \frac{4/m}{(2 \Delta / m)^2} = m \Delta^{-2}$ for the hard instance.

\begin{theorem}\label{thm:inefficient-low-regret}
Fix $\delta \in (0,1)$ and $\Delta\in (0,1]$. 
We say an algorithm is an $\alpha$-minimax regret algorithm  if for some $\alpha > 0$ and all $T \in \mathbb{N}$ :
\begin{equation*}
\textstyle\underset{\mu'}{\max} \E_{\mu'}\Big[ \sum\limits_{t=1}^T ( r_t(c_t,\pi_*(c_t)) - r_t(c_t,a_t) ) \Big] = \underset{\mu}{\max} \sum\limits_{c,a} \E_{\mu'}[ T_{c,a}(T) ] ( r(c,\pi_*(c)) - r(c,a) ) \leq \sqrt{\alpha |\mc{A}| T}
\end{equation*}
where the maximum is taken over all contextual bandit instances $\mu'=(\nu', \v')$ and $T_{c,a}(T) = \sum_{t=1}^T \1\{ c_t = c, a_t = a \}$.
For any $\alpha$-minimax regret algorithm, it is a $(0,\delta)$-PAC algorithm if at a stopping time $\tau$ it outputs the optimal policy $\pi_*$ with probability at least $1-\delta$.
 Any $\alpha$-minimax regret algorithm that is also $(0,\delta)$-PAC satisfies $\E_\mu[\tau] \geq m^2 \Delta^{-2} \log^2(1/2.4\delta) / 4 \alpha$ for the instance $\mu = (\nu, r)$ defined in \ref{def:hard_instance}.
\end{theorem}
\noindent We point out that the minimax regret optimal rate takes $\alpha = \log(m)=\log(|\Pi|)$.
Thus, taking $\Delta = 1$ and $\delta=0.1$, the minimax regret optimal algorithm has a PAC sample complexity of $m^2 / \log(m)$; whereas the PAC sample complexity of our algorithm, Theorem~\ref{thm:upper_informal}, is just $m \log(m)$.
That is, algorithms with optimal minimax regret have a sample complexity that is at least nearly the optimal PAC sample complexity \emph{squared}.
This demonstrates that no algorithm can simultaneously be minimax regret optimal and obtain the optimal PAC sample complexity.


\subsection{Trivial policy class}\label{sec:trivial}

As a warm-up to discussing linear policy classes, let us consider the simplest policy class.
\begin{definition}[Trivial policy class]
Assume $|\mc{C}| < \infty$ and let $\Pi = \{ \pi(c) = a : (c,a) \in \mc{C} \times \mc{A}\}$ so that $|\Pi| = |\mc{A}|^{|\mc{C}|}$.
\end{definition}
\noindent The trivial policy class has the flexibility to predict any action $a \in \mc{A}$ individually for each $c \in \mc{C}$. This allows us to show that $\rho_{\Pi,0}(\Pi, v)\leq \max_c \frac{2}{\nu_c} \sum_{a'} \Delta_{c,a'}^{-2}$ (see Appendix~\ref{sec:supp-trivial-class}). An immediate corollary of Theorem~\ref{thm:upper_informal} is obtained by simply noting that $|\Pi| = |\mc{A}|^{|\mc{C}|}$.


\begin{corollary}[Trivial class, upper]\label{cor:trivial_upper}
Fix $\epsilon >0$ and $\delta \in (0,1)$. Let $\Pi$ be the trivial policy class applied to some fixed $\mc{C},\mc{A}$ spaces. Then under Assumption~\ref{asm:offline_data} there exists a  computationally efficient $(\epsilon,\delta)$-PAC algorithm for contextual bandits satisfying $\tau \leq \min\{ A \epsilon^{-2}, \max_{c} \frac{1}{\nu_c} \sum_{a'} \Delta_{c,a'}^{-2} \} ( |\mc{C}| \log(|\mc{A}|) + \log( 1 /\delta )) \log(1/\Delta_\epsilon)$, where $\Delta_\epsilon = \max\{ \epsilon, \min_{\pi \in \Pi \setminus \pi_*}V(\pi_*)-V(\pi) \}$. 
Furthermore, this sample complexity never exceeds $\frac{|\mc{A}| (|\mc{C}|\log(|\mc{A}|) + \log(1/\delta))}{\epsilon^2} \log(1/\epsilon)$.
\end{corollary}
\noindent 
Ignoring log factors, the minimax sample complexity of the trivial class is just $\epsilon^{-2} |\mc{A}| ( |\mc{C}| + \log(1/\delta))$.
This is actually a somewhat surprising result, because it says $\lim_{\delta \rightarrow 0} \frac{\E[\tau]}{\log(1/\delta)} \rightarrow \epsilon^{-2} |\mc{A}|$ which is \emph{independent} of $|\mc{C}|$.
To see why this result is somewhat remarkable, if we played a best-arm identification algorithm for each of the $|\mc{C}|$ contexts, then this would lead to a sample complexity of $\epsilon^{-2} |\mc{C}| \cdot |\mc{A}| \log(1/\delta)$. 
It is somewhat of a surprise that such a natural strategy is not optimal. 
For intuition for why we can avoid the multiplicative $|\mc{C}|$, note that to identify an $\epsilon$-good policy among just two policies $(\pi,\pi_*)$ using uniform exploration requires just $\epsilon^{-2}|\mc{A}|\log(1/\delta)$ samples. When we have more than two policies, a union bound achieves the claimed result.  

The minimax sample complexity of Corollary~\ref{cor:trivial_upper} (i.e., the second statement) is nearly tight:
\begin{theorem}[Trivial class, lower]\label{thm:trivial-class-lower}
Fix $\epsilon > 0$ and $\delta \in (0,1/6)$.
Let $\Pi$ be the trivial policy class applied to some fixed $\mc{C},\mc{A}$ spaces.
Moreover, fix a contextual bandit instance $\mu = (\nu,r)$ and a collection of policies $\Pi$.
Then any $(0,\delta)$-PAC algorithm for contextual bandits satisfies $\E_\mu[\tau] \geq \max_{c} \frac{1}{\nu_c} \sum_{a} \Delta_{c,a}^{-2} \log(1/2.4 \delta)$.
Furthermore, $\sup_\mu \E_\mu[\tau] \geq \epsilon^{-2}|\mc{A}|( |\mc{C}| +   \log(1/ \delta))$.
\end{theorem}

\subsection{Linear policy class}\label{sec:linear}

A particularly compelling model-class of policies is the set of linear policies. 

\begin{definition}[Linear policy class]\label{def:linear_policy_class}
Fix a feature map $\phi: \mc{C} \times \mc{A} \rightarrow \R^d$ and assume it is known to the learner.
Let $\Pi = \{ \pi(c) = \arg\max_{a \in \mc{A}} \langle \phi(c,a) , \theta \rangle , \forall \theta \in \R^d \}$.
\end{definition}
\noindent We can consider two settings: the agnostic setting and the realizable setting.
In the agnostic setting, there is no assumed relationship between the true reward function $\v(c,a)$ and $\phi: \mc{C} \times \mc{A} \rightarrow \R^d$.
In this case, Theorem~\ref{thm:upper_informal} applies directly by taking a cover of $\Pi$. 
\begin{corollary}[Agnostic, upper bound]\label{thm:linear_agnostic}
Fix $\epsilon \geq 0$ and $\delta \in (0,1)$.
Let $\Pi$ be the linear policy class in $\R^d$.
Under Assumption~\ref{asm:offline_data} there exists a computationally efficient $(\epsilon,\delta)$-PAC algorithm for contextual bandits that satisfies a sample complexity of $\tau \leq \rho_{\Pi,\epsilon} \cdot ( d \log(1/\epsilon) + \log( 1 /\delta )) \log(1/\Delta_\epsilon)$ where $\Delta_\epsilon = \max\{ \epsilon, \min_{\pi \in \Pi \setminus \pi_*}V(\pi_*)-V(\pi) \}$. Furthermore, this sample complexity never exceeds $\frac{|\mc{A}| (d \log(1/\epsilon) + \log(1/\delta))}{\epsilon^2} \log(1/\epsilon)$.
\end{corollary}
\noindent Comparing to the lower bound of Theorem~\ref{thm:lower_bound}, the instance dependent upper bound of Corollary~\ref{thm:linear_agnostic} matches up to a factor of the dimension and negligible $\log$ factors. 
In contrast to the ``model-free'' feel of the agnostic case, we can also consider a ``model-based'' type setting that we refer to as the \text{realizable} setting.
\begin{definition}[Realizable]\label{def:realizable}
We say the linear policy class is \emph{realizable} if there exists a $\theta_* \in \R^d$ such that $\v(c,a) = \langle \phi(c,a) , \theta_* \rangle$ for all $c \in \mc{C}$ and $a \in \mc{A}$.
Thus, for any $\pi \in \Pi$ we have $V(\pi) = \E_{c \sim \nu}[ \v(c, \pi(c)) ] = \E_{c \sim \nu}[ \langle \phi(c,\pi(c)), \theta_* \rangle ] = \langle \phi_\pi , \theta_* \rangle$ with $\phi_\pi := \E_{c \sim \nu}[ \phi(c,\pi(c)) ]$.
Finally, at the start of the game the learner knows this model.
\end{definition}

\noindent The setting in Definition~\ref{def:realizable} is commonly referred to as the linear contextual bandit problem \cite{abbasi2011improved}.
Clearly, we have that $\pi_*(c) = \arg\max_{a \in \mc{A}} \langle \phi(c,a) , \theta_* \rangle$. We begin by defining a quantity fundamental to our sample complexity results:
\begin{align*}
    \rho_{ {\sf lin},\epsilon} := \min_{p_c  \in \triangle_{\mc{A}}, \, \forall c \in \mc{C}} \max_{\pi \in \Pi \setminus \pi_*} \frac{ \| \phi_\pi - \phi_{\pi_*} \|^2_{\E_{c \sim \nu}[ \sum_{a \in \mc{A}} p_{c,a} \phi(c,a) \phi(c,a)^{\top} ]^{-1}} }{ \langle \phi_{\pi_*} - \phi_\pi , \theta_* \rangle^2 \vee \epsilon^2 }. 
\end{align*}

\begin{theorem}[Realizable, lower bound]\label{thm:linear_lower_bound}
Fix $\epsilon=0$ and $\delta \in (0,1)$. 
Let $\Pi$ be the linear policy class in $\R^d$ and assume it is realizable (see Definitions \ref{def:linear_policy_class} and \ref{def:realizable}).
Any $(0,\delta)$-PAC algorithm in this setting satisfies $\E[\tau] \geq \rho_{ {\sf lin},0} \cdot \log(1/2.4\delta)$.
\end{theorem}

\noindent We now state our nearly matching upper bound. However, in this case we note that the algorithm is not  computationally efficient.
\begin{theorem}[Realizable, upper bound]\label{thm:realizable_upper}
Fix $\epsilon \geq 0$ and $\delta \in (0,1)$.
Let $\Pi$ be the linear policy class in $\R^d$ and assume it is realizable (see Definitions \ref{def:linear_policy_class} and \ref{def:realizable}).
Under Assumption~\ref{asm:offline_data} there exists an $(\epsilon,\delta)$-PAC algorithm for this setting that with probability at least $1-\delta$ it satisfies 
\begin{align*}
    \tau \leq \rho_{ {\sf lin},\epsilon} \cdot ( \min\{d\log(1/\epsilon), \log(|\Pi|) \} + \log(1/\delta) ) \log(1/\Delta_\epsilon)
\end{align*}
where $\displaystyle\Delta_\epsilon = \max\{ \epsilon, \min_{\pi \in \Pi \setminus \pi_*} \langle \phi_{\pi_*}-\phi_\pi , \theta_* \rangle  \} = \max\{ \epsilon, \min_{(c,a) \in \mc{C} \times \mc{A} : \pi_*(c) \neq a} \langle \phi(c,\pi_*(c))-\phi(c,a) , \theta_* \rangle  \}$. 
Furthermore, this sample complexity never exceeds $\frac{d(d\log(1/\epsilon)+\log(1/\delta)) \log(1/\epsilon)}{\epsilon^2}$.
\end{theorem}

\begin{proof}
To see the second part of the theorem statement, observe that
\begin{align*}
    \max_{\pi \in \Pi \setminus \pi_*}\| & \phi_\pi  - \phi_{\pi_*} \|^2_{\E_{c \sim \nu}[ \sum_{a \in \mc{A}} p_{c,a} \phi(c,a) \phi(c,a)^{\top} ]^{-1}} \\
    &= \max_{\pi \in \Pi \setminus \pi_*} \| \E_{c \sim \nu}[ \phi(c,\pi(c)) - \phi(c,\pi_*(c)) ] \|^2_{\E_{c \sim \nu}[ \sum_{a \in \mc{A}} p_{c,a} \phi(c,a) \phi(c,a)^{\top} ]^{-1}} \\
    &\leq \max_{\pi \in \Pi \setminus \pi_*} \E_{c \sim \nu}\left[ \|  \phi(c,\pi(c)) - \phi(c,\pi_*(c))  \|^2_{\E_{c \sim \nu}[ \sum_{a \in \mc{A}} p_{c,a} \phi(c,a) \phi(c,a)^{\top} ]^{-1}} \right] \\
    &\leq \max_{\pi \in \Pi} 4 \, \E_{c \sim \nu}\left[ \|  \phi(c,\pi(c))  \|^2_{\E_{c \sim \nu}[ \sum_{a \in \mc{A}} p_{c,a} \phi(c,a) \phi(c,a)^{\top} ]^{-1}} \right] \\
    &= \max_{q \in \triangle_\Pi} 4 \, \E_{c \sim \nu}\left[ \sum_{\pi \in \Pi} q_\pi  \|  \phi(c,\pi(c))  \|^2_{\E_{c \sim \nu}[ \sum_{a \in \mc{A}} p_{c,a} \phi(c,a) \phi(c,a)^{\top} ]^{-1}} \right] \\
    &= \max_{q \in \triangle_\Pi} 4 \, \text{Tr}\left(  \E_{c \sim \nu}\left[ \sum_{\pi \in \Pi} q_\pi   \phi(c,\pi(c))  \phi(c,\pi(c))^\top \right] \E_{c \sim \nu}\left[ \sum_{a \in \mc{A}} p_{c,a} \phi(c,a) \phi(c,a)^{\top} \right]^{-1} \right)\\ &\leq 4 d
\end{align*}
where the last line takes $p_{c,a} = \sum_{\pi \in \Pi} \1\{ \pi(c) = a \} q_\pi$, which is at least as good as the minimizing choice in the theorem.
\end{proof}
We remark that the algorithm that achieves this upper bound is very different than popular optimism-based algorithms for linear contextual bandits e.g., UCB or Thompson sampling \cite{abbasi2011improved}.
Indeed, our algorithm computes an experimental design and is related to instance-dependent linear bandit algorithms developed for best-arm identification \cite{soare2014best,fiez2019sequential,degenne2020gamification} and regret minimization \cite{hao2020adaptive,tirinzoni2020asymptotically}.
To our knowledge, Theorem~\ref{thm:realizable_upper} provides the first instance-dependent sample complexity for the PAC setting of linear contextual bandits. 
The most relevant work to Theorem~\ref{thm:realizable_upper} is the work of \cite{zanette2021design} which demonstrated a minimax sample complexity of $d^2/\epsilon^2 \log(1/\delta)$.

\begin{remark}[Agnostic vs. Realizable]
Contrasting the above results, we note that the sample complexity of the agnostic case is always bounded by $|\mc{A}| d /\epsilon^2$.
whereas it never exceeds $d^2/\epsilon^2$ for the realizable case. 
This matches the intuition that when the number of actions is much larger than the dimension, assuming realizability can significantly reduce the sample complexity.
\end{remark}

\subsection{Comparison to the Disagreement Coefficient}

The work of \cite{foster2021instance} provides regret bounds in terms of instance-dependent quantities inspired by the \textit{disagreement coefficient}, a notion of complexity common in the active learning literature \cite{hanneke2014theory}.  
The following corollary relates our sample complexity to these notions of disagreement coefficients. 

Define the \emph{policy disagreement coefficient} as
\begin{align*}
    \mathfrak{C}_{\Pi}^{\mathsf{pol}}(\epsilon_0)=\sup_{\epsilon\geq \epsilon_0}\frac{\E_{c\sim\nu}[\1\{\exists\pi\in\Pi_\epsilon:\pi(c)\ne\pi_*(c)\}]}{\epsilon}
\end{align*}
where  $\Pi_\epsilon:=\{\pi\in\Pi:\P_\nu(\pi(c)\ne\pi_*(c))\leq\epsilon\}$
and the \emph{cost-sensitive disagreement coefficient} as
\begin{align*}
    \mathfrak{C}_{\Pi}^{\mathsf{csc}}(\epsilon_0)=\sup_{\epsilon\geq \epsilon_0}\frac{\E_{c\sim\nu}[\1\{\exists\pi\in\Pi:\pi(c)\ne\pi_*(c),\E_{c \sim \nu}[ \, \v(c,\pi_*(c)) - \v(c,\pi(c)) \, ]\leq \epsilon\}]}{\epsilon}.
\end{align*}
The \textsf{AdaCB} algorithm of \cite{foster2021instance} achieves a regret of roughly $R_T=O\left(\min_{\delta}\left\{\delta\Delta_{\mathsf{uniform}} T, \frac{|\mc{A}| \log( |\Pi|)\mathfrak{C}_{\Pi}^{\mathsf{pol}}(\delta)}{\Delta_{\mathsf{uniform}}}\right\}\right)$ or $R_T=O\left(\min_{\delta}\left\{\delta T, |\mc{A}| \log( |\Pi|)\mathfrak{C}_{\Pi}^{\mathsf{csc}}(\delta)\right\}\right)$. Observe that at time $T$, given the outputs $\pi_1,\pi_2,\cdots,\pi_T$ from \textsf{AdaCB} algorithm, one could return a  (randomized) policy $\tilde{\pi}$ which on observing a context, samples from the empirical distribution over the outputs. By Markov's inequality we have $\td{\pi}$, $V(\pi_{\ast}) - V(\tilde{\pi}) \leq O(\epsilon)$ with constant probability for $\epsilon=\frac{R_T}{T}$. Therefore, an upper bound on the regret translates to a PAC sample complexity of $\frac{|\mc{A}| \log( |\Pi|)}{\epsilon\Delta_{\mathsf{uniform}}}\mathfrak{C}_{\Pi}^{\mathsf{pol}}(\epsilon/\Delta_{\mathsf{uniform}})$ or $\frac{|\mc{A}| \log( |\Pi|)}{\epsilon}\mathfrak{C}_{\Pi}^{\mathsf{csc}}(\epsilon)$. 

Finally, Corollary \ref{cor:disagree_coeff} shows that this sample complexity bound is at least as large as our upper bound, see Appendix \ref{sec:proof_dis_coeff} for the proof. 
\begin{corollary}\label{cor:disagree_coeff}
Recall that $\displaystyle\Delta_{\sf uniform} := \min_{c\in \mc{C}} \min_{a \in \mc{A}} \v(c,\pi_{\ast}(c)) - \v(c,a)$. 
For any $\epsilon_0 > 0$ we have that 
\begin{enumerate}
    \item $\rho_{\Pi, \epsilon_0}\leq \frac{2|\A|}{\epsilon_0\Delta_{\mathsf{uniform}}}\mathfrak{C}_{\Pi}^{\mathsf{pol}}(\epsilon_0/\Delta_{\mathsf{uniform}})$;
    \item $\rho_{\Pi, \epsilon_0}\leq \frac{2|\A|}{\epsilon_0}\mathfrak{C}_{\Pi}^{\mathsf{csc}}(\epsilon_0)$.
\end{enumerate}
Moreover, for all $\epsilon_0 \geq 0$ we have that $\rho_{\Pi, \epsilon_0} < \infty$ whenever $\Delta_{\sf pol} := V(\pi_{\ast}) - \max_{\pi \neq \pi_{\ast}}V(\pi) > 0$.
\end{corollary}

\section{Optimal Algorithms for Contextual Bandits}

\subsection{Reduction to linear realizability and a simple elimination scheme}\label{sec:reduction_to_linear}
The astute reader may have noticed that if we ignore computation, Theorem~\ref{thm:upper_informal} is actually an immediate corollary of Theorem~\ref{thm:realizable_upper} by taking $\phi(c,a) = \text{vec}(\mb{e}_c \mb{e}_a^\top) \in \R^{|\mc{C}|\cdot |\mc{A}|}$ where $\mb{e}_i$ is a one-hot encoded vector so that $\v(c,a) = \langle \phi(c,a), \theta_* \rangle$ with $\theta_* \in \R^{|\mc{C}|\cdot |\mc{A}|}$.
This observation is key to our sample complexity results.
Recall $\phi_\pi :=  \E_{c\sim \nu}[ \phi(c,\pi(c)) ]$ from Definition~\ref{def:realizable}, we have that $V(\pi) = \E[ \v(c,\pi(c)) ] = \E[ \langle \phi(c,\pi(c)), \theta_* \rangle ] = \langle \phi_\pi, \theta_* \rangle$.
We stress that $\mc{C}$ can be uncountable, and thus we would never actually instantiate any of these vectors.

For notational convenience, define the feasible set of (context, action) probability distributions as $\Omega = \Big\{ w \in \Delta_{\C \times \A} : \nu_c = \sum_{a \in \A} w_{a,c}  \Big\}$. Note that for each context, $p_{c} := \{w_{c,a}/\nu_c\}_{a\in \mc{A}}\in \Delta_{\mc{A}}$ defines a probability distribution over actions. Also define $A(w) := \sum_{c,a} w_{c,a} \phi(c,a) \phi(c,a)^\top$ for any $w \in \Omega$. 
Under this notation, recalling the right hand side from Theorems~\ref{thm:linear_lower_bound} and \ref{thm:realizable_upper} we have
\begin{align*}
    \min_{w\in \Omega }\max_{\pi \in \Pi \setminus \pi_*} \frac{ \| \phi_\pi - \phi_{\pi_*} \|^2_{A(w)^{-1}} }{\langle \phi_{\pi_*}-\phi_\pi, \theta_*\rangle^2 \vee \epsilon^2} &= \min_{p_c  \in \triangle_{\mc{A}}, \, \forall c \in \mc{C}} \max_{\pi \in \Pi \setminus \pi_*} \frac{ \| \phi_\pi - \phi_{\pi_*} \|^2_{\E_{c \sim \nu}[ \sum_{a \in \mc{A}} p_{c,a} \phi(c,a) \phi(c,a)^{\top} ]^{-1}} }{ \langle \phi_{\pi_*} - \phi_\pi , \theta_* \rangle^2 \vee \epsilon^2 } \label{eqn:rho_star_linear}
\end{align*}

To show that the sample complexity of Theorem~\ref{thm:upper_informal} is a corollary of Theorem~\ref{thm:realizable_upper}, it suffices to show that equation \eqref{eqn:rho_star_combinatorial} and the above display are equal. 
To see this, observe
\begin{align*}
    \| \phi_\pi - \phi_{\pi_*} \|^2_{A(w)^{-1}} &= \| \E_{c\sim\nu}[ \vec(\mb{e}_c \mb{e}_{\pi(c)}^\top) - \vec(\mb{e}_c \mb{e}_{\pi_*(c)}^\top)  ] \|_{A(w)^{-1}}^2 \\
    &= \textstyle\sum_{c,a} \frac{\nu_c^2}{w_{c,a}} (\1\{ \pi(c)=a \} + \1\{ \pi_*(c)=a\} - 2 \1\{ \pi(c)=\pi'(c) \} ) \\
    &= \textstyle\E_{c \sim \nu} \left[ \left(\frac{1}{p_{c,\pi(c)}}+ \frac{1}{p_{c,\pi_*(c)}}\right) \1\{\pi_*(c) \neq \pi(c) \} \right].
\end{align*}
Due to this equivalence, the lower bound of Theorem~\ref{thm:lower_bound} is also a corollary of Theorem~\ref{thm:linear_lower_bound}.
The lower bound of Theorem~\ref{thm:linear_lower_bound} follows almost immediately from the lower bound argument in~\cite{fiez2019sequential}.

The conclusion of this section is that from a sample complexity analysis alone, all that is left is to prove Theorem~\ref{thm:realizable_upper}.
In the next section we propose an algorithm that achieves this sample complexity but assumes precise knowledge of the context distribution $\nu$ (this is relaxed in following sections).
While the algorithm is highly impractical for a number of reasons, its analysis provides a great deal of intuition and motivation for our final algorithm.

\subsection{A simple, impractical, elimination-style algorithm}\label{sec:elim_alg}
Algorithm~\ref{alg:elimination} provides an initial elimination based method for the PAC-contextual bandit problem. 
The algorithm runs in stages. 
Before the start of each stage $\ell \in \mathbb{N}$, the algorithm defines a distribution $p^{(\ell)}_c \in \triangle_{\mc{A}}$ for each $c \in \mc{C}$. At each successive time $t\in [n_{\ell}]$, it plays the random action $a_t \sim p^{(\ell)}_{c_t}$ in response to context $c_t \sim \nu$, and receives random reward $r_t$ with $\E[r_t|c_t,a_t] = \langle \phi(c_t,a_t),\theta_* \rangle$. 
Observe that 
\[
   \textstyle \E\left[ \phi(c_t,a_t) r_t \right] = \E\left[  \phi(c_t,a_t) \phi(c_t,a_t)^\top \theta_* \right]
    =  \sum_{c \in \mc{C}, a \in \mc{A}} w^{(\ell)}_{c,a} \phi(c,a) \phi(c,a)^\top \theta_* = A(w^{(\ell)})\theta_*
\]
using the identity $w^{(\ell)}_{c,a} := \nu_c p^{(\ell)}_{c,a}$.
Thus, if we set $O_t = A(w^{(\ell)})^{-1} \phi(c_t,a_t) r_t$ then $\E[O_t] = \theta_*$.
A straightforward calculation also shows that $\text{Cov}( O_t ) = A(w^{(\ell)})^{-1}$ if $r_t$ is perturbed with additive unit variance noise.
Thus, an unbiased estimator of $\Delta(\pi, \pi_{\ast}):= V(\pi_*)-V(\pi) = \langle \phi_{\pi_*} - \phi_\pi, \theta_* \rangle$ is simply $\langle \phi_{\pi_*} - \phi_{\pi}, \frac{1}{n_{\ell}} \sum_t O_t \rangle$ which has variance $\frac{1}{n_{\ell}} \| \phi_{\pi_*} - \phi_{\pi} \|_{A(w^{(\ell)})^{-1}}^2$.
Intuitively,  $\langle \phi_{\pi_*} - \phi_{\pi}, \frac{1}{n_{\ell}} \sum_t O_t \rangle = \langle \phi_{\pi_*} - \phi_\pi, \theta_* \rangle \pm \sqrt{\frac{1}{n_{\ell}} \| \phi_{\pi_*} - \phi_{\pi} \|_{A(w^{(\ell)})^{-1}}^2}$ so we can safely conclude that a policy $\pi$ is sub-optimal (i.e., $\pi \neq \pi_*$) if there exists any policy $\pi'$ such that $\langle \phi_{\pi'} - \phi_{\pi}, \frac{1}{n_{\ell}} \sum_t O_t \rangle \gg \sqrt{\frac{1}{n_{\ell}} \| \phi_{\pi'} - \phi_{\pi} \|_{A(w^{(\ell)})^{-1}}^2}$.
This is the intuition behind Contextual RAGE (Algorithm~\ref{alg:elimination}), which inherits its name from the best-arm identification algorithm of~\cite{fiez2019sequential} that inspired its strategy.

However, while $\langle \phi_{\pi_*} - \phi_{\pi}, \frac{1}{n_{\ell}} \sum_t O_t \rangle$ is unbiased and has controlled variance, it is potentially heavy-tailed because $w^{(\ell)}_{c,a}$ can be arbitrarily small.
Instead of trying to control $w^{(\ell)}_{c,a}$ and appealing to Bernstein's inequality, we use the robust mean estimator of Catoni~\cite{lugosi2019mean}. We can then show:
\begin{lemma}\label{lem:rage_correctness}
$\pi_* \in \Pi_\ell$ and $\max_{\pi \in \Pi_\ell} \langle \phi_{\pi_*} - \phi_\pi, \theta^* \rangle \leq 4 \epsilon_\ell$ for all $\ell > 1$ w.p. at least $1-\delta$.
\end{lemma}
The lemma states that if $\Pi_\ell$ is the active set of policies still under consideration, the optimal policy $\pi_*$ is never discarded from $\Pi_\ell$, and moreover, the quality of all policies remaining in $\Pi_\ell$ is getting better and better. The full proof of this lemma is in Appendix~\ref{sec:proof_rage_correctness}. 
We are now ready to state the main sample complexity result. 

\begin{theorem}\label{thm:contextual_rage}
Fix any policy class $\Pi = \{ \pi : \mc{C} \rightarrow \mc{A} \}_\pi$, distribution over contexts $\nu$, $\delta \in (0,1)$, $\epsilon \geq 0$, and feature map $\phi: \mc{C} \times \mc{A} \rightarrow \R^d$ such that $r(c,a) = \langle \phi(c,a) , \theta_* \rangle$ (this is without loss generality, as one can always take $\phi(c,a) = \vec(\mb{e}_c \mb{e}_a^\top)$). 
With probability at least $1-\delta$, if $\phi_\pi = \E_{c \sim \nu}[\phi(c,\pi(c))]$ and $\pi_* = \arg\max_{\pi} \langle \phi_{\pi}, \theta_* \rangle$ then \emph{Contextual-RAGE} returns a policy $\widehat{\pi} \in \Pi$ such that $V(\widehat{\pi}) \geq V(\pi_*) - \epsilon$  after taking at most
\begin{align*}
    c  \min_{w \in\Omega}\max_{\pi \in \Pi} \frac{\|\phi_\pi-\phi_{\pi_*}\|_{A(w)^{-1}}^2}{(\langle \phi_{\pi_*}-\phi_\pi , \theta^* \rangle \vee \epsilon)^2}  \log( \log((\Delta \vee \epsilon)^{-1}) |\Pi| /\delta)\log((\Delta \vee \epsilon)^{-1})
\end{align*}
samples, where $c$ is an absolute constant and $\Delta = \min_{\pi \in \Pi \setminus \pi_*} V(\pi_*) - V(\pi)$.
\end{theorem}

\begin{proof}
Define $S_\ell = \{ \pi \in \Pi : \langle \phi_{\pi_*} - \phi_\pi , \theta^* \rangle \leq 4 \epsilon_\ell \}$.
The above lemma implies that with probability at least $1-\delta$ we have $\bigcap_{\ell=1}^\infty \{ \Pi_\ell \subseteq S_\ell \}$.
Observe that if for any $\mc{V} \subset \Pi$ we define $f(\mc{V}) = \min_{w \in \Omega} \rho(w,\mc{V})$ then
\begin{align*}
\rho(w^{(\ell)}, \Pi_\ell) &= \min_{w \in \Omega} \max_{\pi,\pi' \in \Pi_{\ell}}  \| \phi_\pi - \phi_{\pi'} \|_{A(w)^{-1}}^2 \leq \min_{w \in \Omega} \max_{\pi,\pi' \in S_{\ell}}  \| \phi_\pi - \phi_{\pi'} \|_{A(w)^{-1}}^2 = \rho(S_\ell).
\end{align*}
For $\ell \geq \lceil \log_2(4\Delta^{-1}) \rceil$ we have that $S_\ell = \{ \pi_* \}$, thus the sample complexity to identify $\pi_*$ is
\begin{align*}
\sum_{\ell=1}^{\lceil \log_2(4\Delta^{-1}) \rceil} \tau_\ell 
&= \sum_{\ell=1}^{\lceil \log_2(4\Delta^{-1}) \rceil} \lceil 4 \epsilon_\ell^{-2} \rho( w^{(\ell)}, \Pi_\ell ) \log(2 \ell^2|\Pi|/\delta) \rceil \\
&\leq \sum_{\ell=1}^{\lceil \log_2(4\Delta^{-1}) \rceil} 4 \epsilon_\ell^{-2} \rho(S_\ell) \log(2 \ell^2 |\Pi| /\delta) +1\\
&\leq c \log( \log(\Delta^{-1}) |\Pi| /\delta) \sum_{\ell=1}^{\lceil \log_2(4\Delta^{-1}) \rceil}  \epsilon_\ell^{-2} \rho(S_\ell) 
\end{align*}
for some absolute constant $c > 0$.
We now note that
\begin{align*}
\min_{w \in\Omega}\max_{\pi \in \Pi} \frac{\|\phi_\pi-\phi_{\pi_*}\|_{A(w)^{-1}}^2}{(\langle \phi_{\pi_*}-\phi_\pi , \theta^* \rangle)^2} 
&= \min_{w \in\Omega}  \max_{\ell \leq \lceil \log_2(4\Delta^{-1})\rceil} \max_{\pi \in S_\ell} \frac{\|\phi_\pi-\phi_{\pi_*}\|_{A(w)^{-1}}^2}{(\langle \phi_{\pi_*}-\phi_\pi , \theta^* \rangle)^2} \\
&\geq \frac{1}{\lceil \log_2(4\Delta^{-1})\rceil} \min_{w \in\Omega} \sum_{\ell=1}^{\lceil \log_2(4\Delta^{-1}) \rceil} \max_{\pi \in S_\ell} \frac{\|\phi_\pi-\phi_{\pi_*}\|_{A(w)^{-1}}^2}{(\langle \phi_{\pi_*}-\phi_\pi , \theta^* \rangle)^2} \\
&\geq \frac{1}{16 \lceil \log_2(4\Delta^{-1})\rceil}  \sum_{\ell=1}^{\lceil \log_2(4\Delta^{-1}) \rceil} \epsilon_\ell^{-2} \min_{w \in\Omega} \max_{\pi \in S_\ell} \|\phi_\pi-\phi_{\pi_*}\|_{A(w)^{-1}}^2 \\
&\geq \frac{1}{64 \lceil \log_2(4\Delta^{-1})\rceil}  \sum_{\ell=1}^{\lceil \log_2(4\Delta^{-1}) \rceil} \epsilon_\ell^{-2} \min_{w \in\Omega} \max_{\pi,\pi' \in S_\ell} \|\phi_\pi-\phi_{\pi'}\|_{A(w)^{-1}}^2 \\
&= \frac{1}{64 \lceil \log_2(4\Delta^{-1})\rceil}  \sum_{\ell=1}^{\lceil \log_2(4\Delta^{-1}) \rceil} \epsilon_\ell^{-2} \rho(S_\ell) 
\end{align*}
where we have used the fact that $\max_{\pi,\pi' \in S_\ell} \|\phi_\pi-\phi_{\pi'}\|_{A(w)^{-1}}^2 \leq 4 \max_{\pi \in S_\ell} \|\phi_\pi-\phi_{\pi_*}\|_{A(w)^{-1}}^2$ by the triangle inequality.
\end{proof}

\subsection{Towards a more efficient algorithm}\label{sec:more_efficient}

One major issue with  Algorithm~\ref{alg:elimination} is that it explicitly maintains a set of policies $\Pi_\ell$ from round to round. Since $\Pi$ could be exponential in $|\mc{A}|$, this is a non-starter for any implementation. 
As a motivation for our approach, we consider a non-elimination algorithm, Algorithm~\ref{alg:naive}, as an intermediate step. It does not maintain $\Pi_{\ell}$ and instead just solves the optimization problem \eqref{eqn:opt_naive} over $\Pi$. The design computed in \eqref{eqn:opt_naive} is chosen to ensure that for all $\pi\in \Pi$, $|\hat{\Delta}_{\ell-1}(\pi, \hat{\pi}_{\ell}) - \Delta(\pi, \pi_{\ast})|\leq 2\epsilon_{\ell-1} + \frac{1}{4}\Delta(\pi, \pi_{\ast})$ with high probability (Lemma~\ref{lem:diff_emp_true_gap}). Equivalently, we estimate gaps up to a constant factor for policies with $\Delta(\pi, \pi_{\ast}) > \epsilon_{\ell}$, while our gap estimates are bounded by $\epsilon_{\ell}$ for those policies satisfying $\Delta(\pi, \pi_{\ast}) \leq \epsilon_{\ell}$. This ensures that our choice of $\hat\pi_{\ell}$ is good enough, i.e. satisfies $V(\pi_{\ast}) - V(\hat{\pi}_{\ell}) \leq \epsilon_{\ell}$ with high probability. The full proof is in Appendix~\ref{sec:proof_rage_correctness}.

Unfortunately, Algorithm~\ref{alg:naive} introduces additional problems. It is not clear whether solving \eqref{eqn:opt_naive} is computationally efficient. Also, we need to find an estimator $\hat{\Delta}_l$ that is computationally efficient even if the policy space $\Pi$ is infinite. In addition, it requires precise knowledge of $\nu$ to even \emph{define} the domain of distributions $\Omega$ optimized over, and store the solution $w \in \mc{C} \times \mc{A}$ explicitly. But in general, such precise knowledge will not be available and is only estimable using past data (Assumption~\ref{asm:offline_data}).

\subsection{An instance-optimal and computationally efficient algorithm}\label{sec:alg}
In this section we provide Algorithm~\ref{alg:evaluation_oracle}, which witnesses the guarantees of Theorem~\ref{thm:upper_informal} for the general agnostic contextual bandit problem. We now address the caveats of the previous approaches.

\paragraph{Access to Offline Data.} By Assumption~\ref{asm:offline_data}, we have access to a large amount of sampled offline contexts $\mathcal{D}$, where each $c_t\in \mathcal{D}$ is drawn IID from $\nu$. Having access to $\mc{D}$ allows us to approximate $\mathbb{E}_{c\sim \nu}[\cdot]$ with expectations over the empirical distribution $\mathbb{E}_{c\sim \nu_\D}[\cdot]$, where $\nu_\D$ is the uniform distribution over historical data $\mc{D}$. The number of offline contexts we need only scales logarithmically over the size of the policy set $\Pi$, more specifically, $\operatorname{poly}(|\A|,\epsilon^{-1}, \log(|\Pi|),\log(1/\delta))$. We quantify the precise number of samples needed in Appendix \ref{sec:offline}. 

\begin{figure}
    \centering
\begin{minipage}[t]{0.49\textwidth}
\begin{algorithm}[H]
    \centering
    \caption{Elimination Contextual RAGE}
    \small
    \begin{algorithmic}[1]\label{alg:elimination}
    \REQUIRE  $\Pi$, $\phi : \mc{C} \times \mc{A} \rightarrow \R^d$, $\delta\in (0,1)$
    \STATE \textbf{Initialize} $\Pi_{1}=\Pi$
    \FOR{$\ell=1,2,\cdots,\lceil\log_2(1/\epsilon)\rceil$}
    \STATE $\epsilon_\ell:=2^{-\ell}, \delta_{\ell} := \delta/(2\ell^2|\Pi|)$
    \STATE {\color{red} Let $n_{\ell}$ be the minimum value s.t.:
    \[\hspace{-10pt}\min_{w \in \Omega}\max_{\pi,\pi' \in \Pi_{\ell}} \frac{\| \phi_\pi - \phi_{\pi'} \|_{A(w)^{-1}}^2\log(1/\delta_{\ell})}{n_{\ell}} \leq \epsilon_{\ell}\]}
    with solution $w^{(\ell)}$.
    \STATE 
    For each $t\in [n_\ell]$, get $c_t\sim \nu$, pull $a_t\sim p^{(\ell)}_{c_t}$, observe reward $r_t$ 
    \STATE Compute $O_t=A(w^{(\ell)})^{-1} \phi(c_t,a_t) r_t$.  
    \STATE For  $\pi,\pi' \in \Pi_\ell$  \[\hat{\Delta}_{\ell}(\pi,\pi') = {\sf Cat}( \{ \inner{\phi_\pi - \phi_{\pi'}}{O_{i}} \}_{i=1}^{n_\ell})\]\vspace{-15pt}
    \STATE 
    {\color{red} Update
    \[
    {\Pi}_{\ell+1}={\Pi}_\ell \setminus \{\pi' \in {\Pi}_l\mid \max_{\pi \in {\Pi}_\ell}: \hat{\Delta}_{\ell}(\pi,\pi')>\epsilon_\ell\}\]} 
    \vspace{-5pt}
    \ENDFOR
    \ENSURE $\Pi_{\ell+1}$
    \end{algorithmic}
\end{algorithm}
\end{minipage}
\hfill
\begin{minipage}[t]{0.5\textwidth}
\begin{algorithm}[H]
\small
    \centering
    \caption{Non-elimination Contextual RAGE}
    \begin{algorithmic}[1]\label{alg:naive}
    \REQUIRE  $\Pi$, $\phi : \mc{C} \times \mc{A} \rightarrow \R^d$, $\delta\in (0,1)$
    \STATE \textbf{Initialize:} $\hat{\pi}_0\in\Pi$ arbitrarily
    \FOR{$\ell=1,2,\cdots,\lceil\log_2(1/\epsilon)\rceil$}
    \STATE $\epsilon_\ell:=2^{-\ell}$, $\delta_\ell:=\delta/(2\ell^2|\Pi|)$
    \STATE {\color{blue} Let $n_\ell$ be the minimum value s.t.: 
    \begin{align}
        &\min_{w\in\Omega}\max_{\pi\in\Pi}-\frac{1}{4}\hat{\Delta}_{l-1}(\pi,\pill)\nonumber\\
        &+\sqrt{ \tfrac{2\| \phi_{\pi} - \phi_{\pill} \|_{A(w)^{-1}}^2 \log(1/\delta_l )}{n_\ell}}\leq \epsilon_\ell.\label{eqn:opt_naive}
    \end{align}}
    with solution $w^{(\ell)}$
    \STATE For each $t\in [n_\ell]$, get $c_t\sim \nu$, pull $a_t\sim p^{(\ell)}_{c_t}$, observe reward $r_t$
    \STATE Compute $O_t = A(w^{(\ell)})^{-1}\phi(c_t,a_t)r_t$.
    \STATE For each $\pi\in\Pi$, let $$\hat{\Delta}_{\ell}(\pi,\hat{\pi}_{\ell}) = \mathsf{Cat}(\{\langle \phi_{\pi} - \phi_{\hat{\pi}_\ell}, O_i \rangle\}_{i=1}^{n_{\ell}}).$$ 
    \STATE {\color{blue} Set 
    $\hat{\pi}_\ell:=\arg\min_{\pi\in\Pi}\hat{\Delta}_{\ell}(\pi,\hat{\pi}_\ell)\hfill \inlineeqnum\label{eqn:est_naive}$} 
    \ENDFOR
    \ENSURE $\hat{\pi}_l$
    \end{algorithmic}
    \footnotesize
\end{algorithm}
\end{minipage}
\end{figure}
\paragraph{Computing the design efficiently.} As described, the context space $\mc{C}$ may be infinite so maintaining a distribution $\omega\in \Omega \subset \Delta_{\mc{C}\times \mc{A}}$ is not possible. To overcome this issue, we consider the dual problem of equation~\eqref{eqn:opt_naive}. 
We can remove the square root by noticing that $2\sqrt{x y} = \min_{\gamma > 0 } \gamma x + \frac{y}{\gamma}$, and introducing an additional minimization over the variable $\gamma_{\pi} ,\pi\in \Pi$. Then, the dual problem becomes
\begin{equation}\label{eqn:dual_naive}
    \textstyle\max_{\lambda\in\Delta_\Pi}\min_{w\in\Omega}\min_{\gamma_\pi\geq 0}\sum_{\pi\in\Pi}\lambda_\pi\left(-\hat{\Delta}_{l-1}(\pi,\pill)+\gamma_{\pi} \norm{\phi_{\pi}-\phi_{\pill}}_{A(w)^{-1}}^2 +  \frac{\log(1/\delta_l)}{2 \gamma_{\pi}n_l}\right). 
\end{equation}
Exchanging the order of the minimums on $\omega$ and $\gamma$, somewhat surprisingly we have the close-form expression (Lemma~\ref{lem:est_design_2})
\begin{equation*}
    \textstyle\min_{\omega\in \Omega} \sum_{\pi\in\Pi} \lambda_\pi \gamma_\pi \|\phi_{\pi} - \phi_{\hat{\pi}_{\ell-1}}\|_{A(w)^{-1}}^2 = \E_{c\sim \nu}\left[\left(\sum_{a\in \mathcal{A}}\sqrt{(\lambda\odot \gamma)^\T t_a^{(c)}(\pill) }\right)^2\right],
\end{equation*}
where for $\pi'\in\Pi$, $t_{a}^{(c)}(\pi') \in \{0,1\}^{|\Pi|}$ with  $[t_{a}^{(c)}(\pi')]_\pi:=\1\{\pi(c)=a,\pi'(c)\ne a\}+\1\{\pi(c)\ne a,\pi'(c)=a\}$  and $[\lambda\odot\gamma]_\pi := \lambda_\pi \gamma_\pi$.
Interestingly, this value is achieved at a sampling distribution $\omega$, which is a \textit{non-linear} function of $\lambda$ rather than a convex combination over policies (as in~\cite{agarwal2014taming}). 
Because we have an expectation over contexts, this expectation can be replaced by an empirical estimate using historical data, thus avoiding any issues with an infinite context space.
The final algorithm utilizing these observations found is in Algorithm~\ref{alg:evaluation_oracle}.

The main challenge is finding a solution to the design problem~\eqref{eqn:opt_1}. For starters, we can reduce it to a saddle point problem over $(\lambda, \gamma)$ by considering only a dyadic sequence of $n \in \{2^k: k \in \mathbb{N}\}$. Our procedure uses an alternating ascent/descent method, with the caveat that $\lambda$ lives in a simplex, and $\gamma$ in a box. Both of the spaces are defined over a potentially infinite set of policies $\Pi$ (which in the worst case may scale exponentially in $|\mathcal{C}|$), so we need to argue that we can find a sparse yet $\epsilon$-good solution to \eqref{eqn:opt_1}. 

To handle this, we use the Frank-Wolfe (FW) method on $\lambda$. Referring to the iterates of FW as $\lambda^t$, FW guarantees that the size of the support of $\lambda^t$ in each iterate grows by at most $1$. Thus, if initialized as a $1$-sparse vector, we only need to maintain a sparse $\lambda^t$ in each iteration. Each iterate of Frank-Wolfe involves computing \[\arg\max_{\pi\in \Pi} [\nabla_{\lambda} h_{\ell}(\lambda, \gamma, n)]_{\pi}.\]
To do so, we show that we can appeal to a constrained argmax oracle (\textsf{AMO}) to run the Frank-Wolfe algorithm, a similar approach to~\cite{agarwal2014taming}. 
At an iterate $t$, we use a gradient descent procedure for $\gamma^t$. We will show that in iterate $t$, the support of $\gamma^t$ is contained in that of $\lambda^t$, and we can quantify the number of steps of gradient descent needed to find an $\epsilon$-good solution.
Though $h_{l}(\lambda, \gamma, n)$ might not be convex in $\gamma$, we nevertheless are able to argue that it has a unique minima and that gradient descent converges to this minima. We introduce our subroutine in Algorithm~\ref{alg:FW} and shows that it is computationally efficient with access to an argmax oracle (Definition~\ref{def:csc}) in Theorem~\ref{thm:computational_efficiency}. 

\begin{algorithm}[!htb]
\small
\caption{\textsf{FW-GD}}
\begin{algorithmic}[1]\label{alg:FW}
\REQUIRE $\Pi$ policy sets, number of actions $|\A|$, $\pill\in\Pi$, $\eta_l>0$, $K\in\N$, threshold $\epsilon_l$, $\gamma_{\min}$, $\gamma_{\max}$
\STATE Initialize $n_1=1$, $L=|\A|^2\frac{((1+\eta_l)\gamma_{\max})^{5/2}}{\eta_l^{3/2}\gamma_{\min}^2}$
\FOR{$r=1,2,\cdots$} 
\STATE Initialize $\lambda^0=\e_{0}\in \mathbb{R}^{\Pi}$, $\gamma^0=\1_{|\Pi|}\cdot\sqrt{\frac{\log(1/\delta_l)}{|\A|^2\eta_l n_r}}\in \mathbb{R}^{|\Pi|}$ \hfill {\color{blue} \texttt{// Never explicitly materialized}}
\FOR{$t=0,1,2,\cdots,K$}
\STATE Compute 
\begin{equation}\label{eqn:csc}
\pi_t=\arg\max_{\pi\in\Pi}\bigbrak{\nabla_{\lambda} h_l(\lambda^{t},\gamma^{t},n_{r})}_\pi
\end{equation}
\STATE Set the FW-gap \[g_t=\left\langle\nabla_{\lambda}  h_l(\lambda^{t},\gamma^{t},n_{r}),\e_{\pi_t}- \lambda^{t}\right\rangle = \bigbrak{\nabla_{\lambda} h_l(\lambda^{t},\gamma^{t},n_{r})}_{\pi_t}-\sum_{\pi\in \text{supp}(\lambda^{t})} \bigbrak{\nabla_{\lambda} h_l(\lambda^{t},\gamma^{t},n_{r})}_\pi\]
\STATE Set $\beta_t=\min \left\{\frac{g_{t}}{L\left\|\lambda^{t}-\e_{\pi_t}\right\|_1^{2}}, 1\right\}$ 
\STATE Set $\kappa_{t}=\frac{\epsilon_l}{(t+1)^2}$
\STATE Set $\lambda^{t+1}=(1-\beta_t)\lambda^{t}+\beta_t \e_{\pi_t}$\hfill {\color{blue} \texttt{// Only 1-sparse updates recorded}}
\STATE Set $\gamma^{t+1}=\mathsf{GD}(\lambda^{t},n_{r},\kappa_{t})$ \hfill {\color{blue} \texttt{// Only differences from $\gamma_0$ recorded}}
\ENDFOR
\IF{$h_l(\lambda^{K+1},\gamma^{K+1},n_r)\leq \epsilon_l$}
\STATE \textbf{break}
\ELSE
\STATE $n_{r+1}=2\cdot n_r$
\ENDIF
\ENDFOR
\ENSURE $\lambda^{K+1}\in\triangle_\Pi,\gamma^{K+1}\in\R_+^{|\Pi|},n_{r}$
\end{algorithmic}
\end{algorithm}

\begin{theorem}\label{thm:computational_efficiency}
Let $K_l$ be the number of iterations for \textsf{FW-GD} in the $l$th round and $\lambda^*,\gamma^*$ be the exact solutions to the optimization problem $\max_{\lambda\in\Delta_\Pi}\min_{\gamma\in[\gamma_{\min},\gamma_{\max}]^{|\Pi|}}h_l(\lambda,\gamma,n)$. Then, $K_l=\operatorname{poly}_1(|\A|,\epsilon_l^{-1},\log(1/\delta))$ and the outputs $\lambda^{K_l+1},\gamma^{K_l+1}$ satisfy $h_l(\lambda^*,\gamma^*,n)-h_l(\lambda^{K_l+1},\gamma^{K_l+1},n)\leq\epsilon_l$ with at most $O(K_l^2|\D|)$ calls to a constrained argmax oracle, where the size of the history $\D$ exceeding $\operatorname{poly}_2(\epsilon^{-1},\log|\Pi|, \gamma_{\max},\gamma_{\min}^{-1}, \eta^{-1},|\A|,\log(1/\delta))$ with probability at least $1-\delta$, where $\operatorname{poly}_1$, $\operatorname{poly}_2$ denote some polynomial. 
\end{theorem}

The full proof is in Appendix~\ref{sec:proof_FW_GD}. It is worth noting that we can bound the suboptimality error $h_l(\lambda^*,\gamma^*,n)-h_l(\lambda^{K_l+1},\gamma^{K_l+1},n)$ by the duality gap, as the primal objective is always at least as large as the optimum. Also, the Frank-Wolfe algorithm directly tackles the duality gap, so standard Frank-Wolfe analysis will show that the Frank-Wolfe output makes the duality gap small \cite{pedregosa2020linearly}. 


\paragraph{Regularized Estimator.}  While Algorithms~\ref{alg:elimination} and \ref{alg:naive} use a robust mean estimator as in equation~\eqref{eqn:est_naive}, this estimator is impractical with a very large number of policies $\Pi$. 
Instead, we use a regularized IPS estimator that can be computed using historical data and an argmax oracle.

\begin{algorithm}[!htb]
\small
\caption{Contextual Oracle-efficient Dualized Algorithm (CODA)}
\begin{algorithmic}[1]\label{alg:evaluation_oracle}
\REQUIRE policies $\Pi = \{ \pi : \mc{C} \rightarrow \mc{A} \}_\pi$, feature map $\phi : \mc{C} \times \mc{A} \rightarrow \R^d$, $\delta\in (0,1)$, historical data $\mc{D} = \{ \nu_s \}_s$
\STATE initiate $\hat{\pi}_0\in\Pi$ arbitrarily, $\lambda_0=\e_{\hat{\pi}_0}$, $\hat{\Delta}_0(\pi)$, $\gamma_0$, $\gamma_{\min}$, $\gamma_{\max}$ appropriately
\FOR{$l=1,2,\cdots$}
\STATE $\epsilon_l=2^{-l}$, $\delta_l=\delta/(l^2|\Pi|^2)$ 
\STATE Define
\begin{align}
    \hspace{-10pt}h_l(\lambda,\gamma,n)=\textstyle\sum_{\pi\in\Pi}\lambda_\pi\left(-\hat{\Delta}_{l-1}^{\gamma_{l-1}}(\pi,\pill)+\tfrac{\log(1/\delta_l)}{\gamma_\pi n}\right)+\E_{c\sim\nu_{\mathcal{D}}}\Big[\Big(\sum_{a\in\A}\sqrt{(\lambda\odot \gamma)^\T t_a^{(c)}(\pill)}\Big)^2\Big].\label{eqn:h_l}
\end{align}
\STATE {\color{blue} Let $\lambda^l, \gamma^l, n_l=\textsf{FW-GD}(\Pi,|\A|,\pill,\epsilon_l)$. These are the solutions to 
\begin{align}
    & n_{\ell} := \min\{n\in \mathbb{N}: \max_{\lambda\in\Delta_\Pi}\min_{\gamma\in[\gamma_{\min},\gamma_{\max}]^{|\Pi|}}h_l(\lambda,\gamma,n)\leq \epsilon_{\ell}\}\label{eqn:opt_1}
\end{align}}
\STATE For $i\in [n_{\ell}]$ get $c_i\sim\nu$, pull $a_i\sim p^{(\ell)}_{c_i}$ where $p^{(\ell)}_{c_s,a_s}\propto \sqrt{(\lambda_l\odot\gamma_l)^\T t_{a_s}^{(c_s)}(\pill)}$, observe rewards $r_s$ 
\STATE For each $\pi\in\Pi$, define the IPS estimator $$\hat{\Delta}_l^{\gamma_l}(\pi,\pill)=\sum_{s=1}^{n_l}\frac{r_s}{p^{(\ell)}_{c_s,a_s}+[\gamma_l]_\pi}(\1\{\pill(c_s)=a_s\}-\1\{\pi(c_s)=a_s\})$$
\STATE {\color{blue} set 
\vspace{-5pt}
\begin{equation}
   \textstyle \hat{\pi}_l
    =\arg\min_{\pi\in\Pi}\hat{\Delta}_l^{\gamma_l}(\pi,\hat{\pi}_{l-1})+\E_{c\sim\nu_{\mathcal{D}}} \left[ \left(\frac{[\gamma_l]_\pi}{p_{c,\pi(c)}^{(\ell)}}+\frac{[\gamma_l]_\pi}{p_{c,\pill(c)}^{(\ell)}}\right) \1\{\pill(c )\ne \pi(c )\} \right]+\frac{\log(1/\delta_l)}{[\gamma_l]_\pi n_l}\label{eqn:opt_2}
\end{equation}
\vspace{-5pt}}
\ENDFOR
\ENSURE $\hat{\pi}_l$
\end{algorithmic}
\end{algorithm}

Algorithm~\ref{alg:evaluation_oracle} puts it all together and 
Theorem~\ref{thm:sample_complexity} shows our main result. 
Note that for exposition purposes, we have omitted some additional regularization terms in the optimization problems that have no effect on the sample complexity, but ensure finite-time convergence. Appendix~\ref{sec:complexity} shows the full algorithm and the proof. In what follows, $\operatorname{poly}_1(|\A|,\epsilon^{-1},\log(1/\delta))\cdot\log(|\Pi|)$ and $\operatorname{poly_2}(|\A|,\epsilon^{-1},\log(1/\delta),\log(|\Pi|))$ are polynomials in their arguments that specified in the appendix. 

\begin{theorem}\label{thm:sample_complexity}
Fix any set of policies $\Pi$, context distribution $\nu$ and reward function $r(c,a)\in [0,1]$. With probability at least $1-\delta$, provided a history $\mc{D}$ whose size exceeds $\operatorname{poly}_1(|\A|,\epsilon^{-1},\log(1/\delta))\cdot\log(|\Pi|)$, Algorithm~\ref{alg:evaluation_oracle} returns a policy $\hat{\pi}$ satisfying $V(\pi_{\ast}) - V(\hat{\pi}_{\ell})\leq \epsilon$ in a number of samples not exceeding $O(\rho_{\Pi, \epsilon}\log(|\Pi|\log_2(1/\Delta_\epsilon)/\delta)\log_2(1/\Delta_{\epsilon}))$ where $\Delta_{\epsilon} := \max\{\epsilon, \min_{\pi\in \Pi} V(\pi_{\ast}) - V(\pi)\}$. 

In addition, Algorithm~\ref{alg:evaluation_oracle} is computationally efficient and requires the amount of calls not exceeding $\operatorname{poly_2}(|\A|,\epsilon^{-1},\log(1/\delta),\log(|\Pi|))$ to a constrained argmax oracle.

\end{theorem}

\section{Conclusion}

This work provides the first instance-dependent lower bounds for the $(\epsilon,\delta)$-PAC contextual bandit problem. 
One limitation of this work is that our analysis of Algorithm~\ref{alg:evaluation_oracle} does not immediately extend to the realizable linear setting. That is, a computationally efficient algorithm that achieves the same bound is not known to exist. 
In all other settings discussed in this work, we proposed a computationally efficient algorithm.
A second limitation is the assumption that we have access to a large pool of offline data. 
Because it seems necessary to plan with some information about the context distribution, it is not clear how one would completely remove such an assumption and achieve the same sample complexity bounds. As with any recommender system, there is the potential for unintended consequences from optimizing just a single metric. Moreover, other potential pitfalls can arise, such as negative feedback loops, if our assumptions fail to hold in real-world environments. Such consequences can be mitigated by tracking a diverse set of metrics.



\clearpage
\bibliography{arxiv_main}
\bibliographystyle{plain}

\clearpage
\appendix

\addcontentsline{toc}{section}{Appendix} 
\part{Appendix} 
\parttoc 

\section*{Appendix}

In the appendix we present algorithms and proofs not included in the main text. Broadly speaking, 
\begin{itemize}
    \item Section A presents proofs for lower bounds; 
    \item Section B presents proofs for the proposed computationally inefficient algorithms~\ref{alg:elimination} and \ref{alg:naive};
    \item Section C presents results to justify the computational efficiency of Algorithm~\ref{alg:evaluation_oracle};
    \item Section D presents arguments for Algorithm~\ref{alg:evaluation_oracle} hitting the sample complexity lower bound;
    \item Section E-F provides technical proofs to argue about convergence of our subroutines.
\end{itemize}
The table below summarises the notations we used in the proof. 
\begin{table}[!htb]
    \centering
    \begin{tabular}{|c|l|}
    \hline
       $t_{a}^{(c)}(\pi')$ & $\{\1\{\pi(c)=a,\pi'(c)\ne a\}+\1\{\pi(c)\ne a,\pi'(c)=a\}\}_{\pi\in \Pi}\in \mathbb{R}^{\Pi}$\\
       $S_\ell$ & $\{\pi\in\Pi: \langle \phi_{\pi_*} - \phi_\pi , \theta^* \rangle=V(\pi_*)-V(\pi)=\Delta(\pi,\pi_*) \leq \epsilon_\ell\}$ \\
       $w(\lambda,\gamma)$ & $[w(\lambda,\gamma)]_{a, c}=\nu_{c} \cdot p_{c,a}=\nu_{c} \cdot \frac{\sqrt{(\lambda\odot\gamma)^{\top} (t_{a}^{(c)}+\eta)}}{\sum_{a'\in \A} \sqrt{(\lambda\odot\gamma)^{\top} (t_{a^{\prime}}^{(c)}+\eta)}}$ \\
       $\hat{\Delta}_{l}^{\gamma}(\pi,\pi')$ & $\sum_{s=1}^{n_l}\frac{r_s}{p^{(\ell)}_{c_s,a_s}+\gamma_\pi}(\1\{\pi'(c_s)=a_s\}-\1\{\pi(c_s)=a_s\})$\\
       $h_l(\lambda, \gamma, n)$ &  $\sum_{\pi\in\Pi}\lambda_\pi\cdot\bigsmile{-\hat{\Delta}_{l-1}^{\gamma^{l-1}}(\pi,\hat{\pi}_{l-1})+\frac{\log(1/\delta_l)}{\gamma_\pi n}}$\\
       &\qquad$+\gamma_\pi\E_{c\sim\nu_\D}\bigbrak{\bigsmile{\sum_{a\in\A}\sqrt{(\lambda\odot\gamma)^{\top} (t_{a}^{(c)}+\eta_l)}}^2}$\\
       $\PP_l(w,\gamma)$&$\max_{\pi\in\Pi} \left(-\hat{\Delta}_{l-1}^{\gamma^{l-1}}(\pi,\pill)+\gamma\norm{\phi_\pi-\phi_{\pill}}_{A(w)^{-1}}^2+\frac{\log(1/\delta)}{\gamma n_l}\right)$\\
    \hline   
    \end{tabular}
    \caption{Glossary}
    \label{tab:glossary}
\end{table}

\section{Proof for Results in Section \ref{sec:results}}
\subsection{Proof of Theorem~\ref{thm:lower_bound}}
We quickly point out that the proof of Theorem~\ref{thm:lower_bound} is identical to the proof of the linear policy class case proof of Theorem~\ref{thm:linear_lower_bound}. Please see that argument below.

\subsection{Proof of Theorem~\ref{thm:inefficient-low-regret}}

\begin{proof}[Proof of Theorem~\ref{thm:inefficient-low-regret} ]
To relate the random stopping time to the regret bound, note that
\begin{align*}
    \sum_{c,a} \E_\mu[ T_{c,a}(\tau) ] ( r(c,\pi_*(c)) - r(c,a) ) \leq \E_\mu\left[ \sqrt{\alpha \, |\mc{A}| \, \tau} \right] \leq \sqrt{\alpha |\mc{A}| \E_\mu[\tau]} 
\end{align*}
where the last inequality follows by Jensen's inequality.
Since $\pi_1 := \pi_*$ for our particular instance, if $\bar{c} = \arg\min_{c \in [m]} \E_\mu[ T_{c,\pi_c(c)}(\tau) ]$ then
\begin{align*}
    \sum_{c,a} \E_\mu[ T_{c,a}(\tau) ] ( r(c,\pi_1(c)) - r(c,a) ) &= \sum_{c,a} \E_\mu[ T_{c,a}(\tau) ] \Delta \1\{ a \neq \pi_1(c) \} \\
    &\geq \sum_{c} \max_a \E_\mu[ T_{c,a}(\tau) ] \Delta \1\{ a \neq \pi_1(c) \} \\
    &\geq m \min_{c} \max_{a} \E_\mu[ T_{c,a}(\tau) ] \Delta \1\{ a \neq \pi_1(c) \} \\
    &= m \E_\mu[ T_{\bar{c},\pi_{\bar{c}}(\bar{c})}(\tau) ] \Delta.
\end{align*}
Combining the two equations above, and rearranging, we observe that
\begin{align*}
    \E_\mu[ T_{\bar{c},\pi_{\bar{c}}(\bar{c})}(\tau) ] &\leq \frac{1}{m \Delta} \sqrt{\alpha |\mc{A}| \E_\mu[\tau]} .
\end{align*}
Define an instance $\mu' = (\nu,r')$ such that $r'(c,a) = r(c,a)$ for all $(c,a) \in [m] \times \{0,1\} \setminus (\bar{c},1)$, and set $r'(\bar{c},1) = r'( \bar{c}, \pi_{\bar{c}}(\bar{c}) ) = 2 \Delta$ under  $\mu'$ (instead of $r( \bar{c}, \pi_{\bar{c}}(\bar{c}) ) = 0$ under $\mu$).
Note that under $\mu'$, we now have that $\pi_{\bar{c}}$ is the unique optimal policy.
If the algorithm is $(0,\delta)$-PAC then by \cite[Lemma 1]{kaufmann2016complexity} we have that
\begin{align*}
   \log(1/2.4 \delta) &\leq \sum_{c,a}  KL( \mc{N}(r(c,a),1) | \mc{N}(r'(c,a),1) ) \cdot \E_\mu[ T_{c,a}(\tau) ] \\
   &= KL( \mc{N}(0,1) | \mc{N}(2 \Delta,1) ) \cdot \E_\mu[ T_{\bar{c},\pi_{\bar{c}}(\bar{c})}(\tau) ] = 2 \Delta^2 \cdot \E_\mu[ T_{\bar{c},\pi_{\bar{c}}(\bar{c})}(\tau) ] \\
   &\leq 2 \Delta^2 \cdot \frac{1}{m \Delta} \sqrt{\alpha |\mc{A}| \E_\mu[\tau]}  = \sqrt{\frac{ 4 \alpha \E_\mu[\tau] }{m^2 \Delta^{-2}}}.
\end{align*}
The result follows by rearranging.
\end{proof}

\subsection{Trivial Class: Proof of Theorem~\ref{thm:trivial-class-lower}}\label{sec:supp-trivial-class}
Firstly note that \begin{align*}
    \rho_{\Pi,0}(\Pi, v) &= \min_{p_c \in \triangle_{\mc{A}}, \ \forall c \in \mc{C} }  \max_{\pi \in \Pi \setminus \pi_* } \frac{ \E_{c \sim \nu} \left[ \left(\frac{1}{p_{c,\pi(c)}}+ \frac{1}{p_{c,\pi_*(c)}}\right) \1\{\pi_*(c) \neq \pi(c) \} \right] }{(\E_{c \sim \nu}[ \, \v(c,\pi_*(c)) - \v(c,\pi(c)) \, ] )^2} \\
    &= \min_{p_c \in \triangle_{\mc{A}}, \ \forall c \in \mc{C} }  \max_{\pi \in \Pi \setminus \pi_* } \frac{ \sum_{c \in \mc{C}} \nu_c \left(\frac{1}{p_{c,\pi(c)}}+ \frac{1}{p_{c,\pi_*(c)}}\right) \1\{\pi_*(c) \neq \pi(c) \}  }{(\sum_{c \in \mc{C}} \nu_c \Delta_{c,\pi(c)} \1\{\pi_*(c) \neq \pi(c) \})^2} \\
    &= \min_{p_c \in \triangle_{\mc{A}}, \ \forall c \in \mc{C} }  \max_{\substack{\alpha \in \{0,1\}^{|\mc{C}| \times |\mc{A}|} \setminus \mathbf{0}:\\
    \sum_a \alpha_{c,a} \in \{0,1\}} } \frac{\sum_{c,a} \alpha_{c,a}\nu_c \left(\frac{1}{p_{c,\pi(c)}}+ \frac{1}{p_{c,\pi_*(c)}}\right) \1\{\pi_*(c) \neq a \}   }{(\sum_{c,a} \alpha_{c,a} \nu_c \Delta_{c,\pi(c)} \1\{\pi_*(c) \neq \pi(c) \} )^2} \\
    &= \min_{p_c \in \triangle_{\mc{A}}, \ \forall c \in \mc{C} }  \max_{c,a : \pi_*(c) \neq a} \frac{ \nu_c \left(\frac{1}{p_{c,a}}+ \frac{1}{p_{c,\pi_*(c)}}\right)  }{(\nu_c \Delta_{c,a} )^2} \\
    &\leq \max_c \frac{2}{\nu_c} \sum_{a'} \Delta_{c,a'}^{-2}
\end{align*}
where the last equality follows from repeated application of the inequality $\frac{a_1 + a_2}{(b_1+b_2)^2} \leq \frac{a_1}{b_1^2} \vee \frac{a_2}{b_2^2}$.

\begin{proof}[Proof of Theorem~\ref{thm:trivial-class-lower}]
The proof of the instance-dependent lower bound for $\epsilon=0$ follows directly from Theorem~\ref{thm:lower_bound}.
The second minimax statement is, to our best knowledge, novel. 

First, note that $\sup_\mu \E_\mu[\tau] \geq \epsilon^{-2}|\mc{A}| \log(1/ \delta)$ by a reduction to multi-armed bandits by just setting $\nu_1 = 1$ and $\nu_c =0$ for all $c \neq 1$ \cite{mannor2003lower,kaufmann2016complexity}.
If $U$ denotes the set of instances that achieves this supremum, and $V$ is another set of instances, we note that
$\sup_\mu \E_\mu[\tau] = \sup_{P} \E_{\mu \sim P} \E_\mu[\tau] \geq \frac{1}{2} \sup_{\mu \in U} \E_\mu[\tau] + \frac{1}{2} \sup_{\mu \in V} \E_\mu[\tau]$ for some other set of instances $V$.
Thus, it remains to show that $\sup_\mu \E_\mu[\tau] \geq \epsilon^{-2}|\mc{A}| \cdot |\mc{C}|$.

Consider the following construction of $|\Pi| = |\mc{A}|^{|\mc{C}|}$ instances.
For each context $c \in \mc{C}$ let $\nu_c = 1/|\mc{C}|$, and for each $\pi \in \Pi$ let $r_\pi(c,a) = \alpha \epsilon \1\{ \pi(c) = a \}$ for some $\alpha >0$ to be determined later. 
Clearly, policy $\pi$ is the unique optimal policy under the reward function $r_\pi(s,a)$. 
Assume that observations are perturbed by Gaussian $\mc{N}(0,1)$ noise.

Fix $p \in (1/2,1)$ to be determined later. 
Let $S := \{ c \in \mc{C}: \P_{\mu_\pi}( \pi(c) = \widehat{\pi}(c) ) > p \}$ and suppose $|S| \leq |\mc{C}|/8$.
Then 
\begin{align*}
    \P_{\mu_\pi}( V(\pi) - V(\widehat{\pi}) \leq \epsilon ) &= \P_{\mu_\pi}( \frac{1}{|\mc{C}|} \sum_{c \in \mc{C}} \alpha \epsilon \1\{ \widehat{\pi}(c) \neq \pi(c) \} \leq \epsilon ) \\
    &= \P_{\mu_\pi}(  \sum_{c \in \mc{C}} \1\{ \widehat{\pi}(c) \neq \pi(c) \} \leq |\mc{C}| / \alpha )  \\
    &= \P_{\mu_\pi}(  \sum_{c \in \mc{C}} \1\{ \widehat{\pi}(c) = \pi(c) \} \geq |\mc{C}|(1-1 / \alpha) )  \\
    &\leq \P_{\mu_\pi}(  \sum_{c \in \mc{C} \setminus S} \1\{ \widehat{\pi}(c) = \pi(c) \} \geq |\mc{C}| (1 - 1/ \alpha - 1/8) ) \\
    &\leq \frac{\sum_{c \in \mc{C} \setminus S} \P_{\mu_\pi}(\widehat{\pi}(c) = \pi(c)) }{|\mc{C}|(1 - 1/ \alpha - 1/8)} \leq \frac{p}{1 - 1/ \alpha  - 1/8} \leq 5/6
\end{align*}
with $p = 5/8$ and $\alpha=8$.  
This implies that for $\delta \in (0,1/8)$, any $(\epsilon,\delta)$-PAC algorithm must satisfy $\min_\pi |\{ c \in \mc{C}: \P_{\mu_\pi}( \pi(c) = \widehat{\pi}(c) ) > p \}| \geq |\mc{C}| /8$.

Assume the algorithm is permutation invariant (note that any reasonable algorithm satisfies this, including UCB, Thompson Sampling, elimination, etc.).
Let $\mu_\pi^{(i)} = (\nu,r_0)$ where $r_\pi^{(i)}(c,i) = r_\pi^{(i)}(c,\pi(c)) = \alpha \epsilon$, and $r_\pi^{(i)}(c,j) = 0$ for $j \not\in \{i,\pi(c)\}$.
Note that $\P_{\mu_\pi}( \pi(c) = \widehat{\pi}(c) ) \geq p = 5/6$ and also by the symmetric algorithm assumption that $\P_{\mu_\pi^{(i)}}( \pi(c) = \widehat{\pi}(c) ) \leq 1/2$ because there are two identical best-arms. 
Note that $\sum_{j \in \mc{A}} \E_{\mu_\pi}[T_{c,j}] KL(\mu_\pi(j),\mu^{(i)}_\pi(j) ) = \E_{\mu_\pi}[T_{c,i}] \alpha^2 \epsilon^2 /2$ for $i \neq \pi(c)$.
Putting these two pieces together and applying Lemma~1 of~\cite{kaufmann2016complexity}, we have:
\begin{align*}
    \E_{\mu_\pi}[T_{c,i}] \alpha^2 \epsilon^2 /2 &= \sum_{j \in \mc{A}} \E_{\mu_\pi}[T_{c,j}] KL(\mu_\pi(j),\mu^{(i)}_\pi(j) ) \\
    &\geq d( \P_{\mu_\pi}( \pi(c) = \widehat{\pi}(c)) , \P_{\mu_\pi^{(i)}}( \pi(c) = \widehat{\pi}(c) ) )\\
    &\geq d(5/6,1/2) = \frac{1}{6}\log(5^5/3^6) \geq 1/10.
\end{align*}
Thus, $ \E_{\mu_\pi}[\sum_{i \neq \pi_*(c)} T_{c,i}] \geq \frac{1}{5}\alpha^{-2} \epsilon^{-2} (|\mc{A}|-1)$ and this must occur on at least $|\mc{C}| /8$ contexts. 
Pick one context $c$ of these arbitrarily. 
Then
\begin{align*}
    \frac{1}{5}\alpha^{-2} \epsilon^{-2} (|\mc{A}|-1) \leq \E_{\mu_\pi}[\sum_{i \neq \pi_*(c)} T_{c,i}] = \E_{\mu_\pi}[\sum_{t =1}^\tau \1\{ c_t = c\} ] = \E_{\mu_\pi}[ \tau ] \nu_c = \E_{\mu_\pi}[ \tau ] /|\mc{C}|.  
\end{align*}
Consequently, $\E[\tau] \geq  \frac{1}{5}\alpha^{-2} \epsilon^{-2} (|\mc{A}|-1) |\mc{C}|$.

\end{proof}

\subsection{Proofs of Linear Policy Class}
 We begin by defining a quantity fundamental to our sample complexity results:
\begin{align}\label{eqn:linepsilon}
    \rho_{ {\sf lin},\epsilon} := \min_{p_c  \in \triangle_{\mc{A}}, \, \forall c \in \mc{C}} \max_{\pi \in \Pi \setminus \pi_*} \frac{ \| \phi_\pi - \phi_{\pi_*} \|^2_{\E_{c \sim \nu}[ \sum_{a \in \mc{A}} p_{c,a} \phi(c,a) \phi(c,a)^{\top} ]^{-1}} }{ \langle \phi_{\pi_*} - \phi_\pi , \theta_* \rangle^2 \vee \epsilon^2 }. 
\end{align}
\lalit{
For the special case of $\epsilon=0$, we have that $\rho_{ {\sf lin},\epsilon}$ simplifies to \lalit{We should prove this - but maybe we can rip it out.}
\begin{align*}
    \rho_{ {\sf lin},0}
    = \min_{p_c  \in \triangle_{\mc{A}}, \, \forall c \in \mc{C}} \max_{(c',a) \in \mc{C} \times \mc{A} : \pi_*(c') \neq a} \frac{ \| \phi(c',a) - \phi(c',\pi_*(c')) \|^2_{\E_{c \sim \nu}[ \sum_{a \in \mc{A}} p_{c,a} \phi(c,a) \phi(c,a)^{\top} ]^{-1}} }{ \langle  \phi(c',\pi_*(c'))-\phi(c',a) , \theta_* \rangle^2  }.
\end{align*}
This essentially says you have to solve the best-arm identification problem for each context \cite{soare2014best,fiez2019sequential,degenne2020gamification}, which is very intuitive.
This simplification is proven as part of the following theorem:
}
We quickly point out that the proof of Theorem~\ref{thm:lower_bound} is identical to the proof of the linear policy class case proof of Theorem~\ref{thm:linear_lower_bound}.

\begin{proof}[\textbf{Proof of Theorem~\ref{thm:linear_lower_bound}}]
For any $\theta \in \R^d$ let $\P_{\theta}( \cdot )$ and $\E_{\theta}[ \cdot ]$ denote the probability and expectation laws under $\theta$ and $\nu$ such that $c_t \sim \nu$ and playing action $a_t \in \mc{A}$ results in reward $r_t \sim \mc{N}( \langle \phi(c_t,a_t), \theta\rangle , 1 )$.
If an algorithm is $(0,\delta)$-PAC then $\sup_{\theta \in \R^d} \P_\theta( V(\widehat{\pi}(c)) < V(\pi_*(c))  ) \leq \delta$.
Now, of course, under $\theta$ we have that 
\begin{align*}
    V(\widehat{\pi}(c)) < V(\pi_*(c)) &\iff \E_{c \sim \nu}[ \langle \theta, \phi(c,\widehat{\pi}(c)) - \phi(c,\pi_*(c)) \rangle] < 0 \\
    &\iff \langle \theta, \phi_{\widehat{\pi}} -  \phi_{\pi_*} \rangle < 0\\
    &\iff \exists c : \nu_c \langle \theta, \phi(c,\widehat{\pi}(c)) - \phi(c,\pi_*(c)) \rangle  < 0.
\end{align*}
Fix $\theta_* \in \R^d$ and recall that under $\theta$ we have that $\pi_*(c) = \arg\max_{a \in \mc{A}} \langle \phi(c,a) , \theta \rangle$.
Fix any $\theta \in \R^d$ and $\max_{c,a} \nu_c \langle \theta, \phi(c,a) - \phi(c,\pi_*(c)) \rangle > 0$.
Then by \cite[Lemma 1]{kaufmann2016complexity} we have that
\begin{align*}
    &d( \P_{\theta_*}( V(\widehat{\pi}) = V(\pi_*) ), \P_{\theta}( V(\widehat{\pi}) = V(\pi_*) ) ) \\
    &\leq \sum_{c',a'} \E_{\theta_*}[ T_{c',a'} (\tau) ] KL( \mc{N}(\langle \theta_*, \phi(c',a') \rangle,1) | \mc{N}(\langle \theta, \phi(c',a') \rangle,1) ) \\
    &= \sum_{c',a'} \E_{\theta_*}[ T_{c',a'} (\tau) ]  \| \theta_* - \theta \|_{\phi(c',a') \phi(c',a')^\top}^2 /2 \\ 
    &= \E_{\theta_*}[\tau] \sum_{c',a'} \frac{\E_{\theta_*}[ T_{c',a'} (\tau) ]}{\E_{\theta_*}[\tau]}  \| \theta_* - \theta \|_{\phi(c',a') \phi(c',a')^\top}^2 /2 \\ 
    &\leq \max_{p_c \in \triangle_{\mc{A}}, \forall c \in \mc{C}} \E_{\theta_*}[\tau] \sum_{c',a'} \nu_{c'} p_{c',a'} \| \theta_* - \theta \|_{\phi(c',a') \phi(c',a')^\top}^2 /2 \\
    &= \max_{p_c \in \triangle_{\mc{A}}, \forall c \in \mc{C}} \E_{\theta_*}[\tau]  \| \theta_* - \theta \|_{\E_{c \sim \nu}[ \sum_{a} p_{c,a} \phi(c,a) \phi(c,a)^\top]}^2 /2
\end{align*}
where the last inequality follows from Wald's identity:
\begin{align*}
    \sum_{a' \in \mc{A}} \E_{\theta_*}[ T_{c',a'} (\tau) ] &= \sum_{a' \in \mc{A}} \E_{\theta_*}\left[ \sum_{t=1}^\tau \1\{ a_t = a', c_t = c' \} \right] = \E_{\theta_*}\left[ \sum_{t=1}^\tau \1\{ c_t = c' \} \right] = \E_{\theta_*}[\tau] \nu_{c'}.
\end{align*}
Noting that $d( \P_{\theta_*}( V(\widehat{\pi}) = V(\pi_*) ), \P_{\theta}( V(\widehat{\pi}) \geq d(1-\delta,\delta) \geq \log(1/2.4 \delta)$ and we can minimize over $\theta$, given the conditions, we have that \lalit{Need a sentence about why this is an alternative - basically $\arg\max_{\pi\in\Pi} V_\theta(\pi) = \pi_{\theta}$}
\begin{align*}
    \log(1/2.4\delta) &\leq \max_{p_c \in \triangle_{\mc{A}}, \forall c \in \mc{C}} \min_{\theta: \exists c :\nu_c \langle \theta, \phi(c,a) - \phi(c,\pi_*(c))\rangle > 0} \E_{\theta_*}[\tau]  \| \theta_* - \theta \|_{\E_{c \sim \nu}[ \sum_{a} p_{c,a} \phi(c,a) \phi(c,a)^\top]}^2 /2 \\
    =& \E_{\theta_*}[\tau] \max_{p_c \in \triangle_{\mc{A}}, \forall c \in \mc{C}}\min_{\substack{c,a\in \mathcal{C}\times \mathcal{A}\\ \pi_{\ast}(c)\neq a}} \frac{\langle \phi(c,\pi_*(c)) - \phi(c,a), \theta_* \rangle^2 }{2 \| \phi(c,a) - \phi(c,\pi_*(c)) \|_{\E_{c \sim \nu}[ \sum_{a} p_{c,a} \phi(c,a) \phi(c,a)^\top]^{-1}}}.
\end{align*}
After rearranging we conclude that 
\begin{align*}
    \E_{\theta_*}[\tau] \geq  \min_{p_c \in \triangle_{\mc{A}}, \forall c \in \mc{C}}\max_{\substack{ c,a\in \mathcal{C}\times \mathcal{A}\\ \pi_{\ast}(c)\neq a}} \frac{2 \| \phi(c,a) - \phi(c,\pi_*(c)) \|_{\E_{c \sim \nu}[ \sum_{a} p_{c,a} \phi(c,a) \phi(c,a)^\top]^{-1}}}{\langle \phi(c,\pi_*(c)) - \phi(c,a), \theta_* \rangle^2 } \log(1/2.4 \delta).
\end{align*}
To see that equation~\eqref{eqn:linepsilon} is a lower bound, follow the exact same sequence of steps but taking any $\theta \in \R^d$ and $\max_{\pi\in\Pi} \E_{c \sim \nu}[ \langle \theta, \phi(c,\pi(c)) - \phi(c,\pi_*(c)) \rangle ] > 0$.
\end{proof}

\subsection{Proof for Corollary~\ref{cor:disagree_coeff}}\label{sec:proof_dis_coeff}

\begin{proof}
Observe that
\begin{align*}
    \rho_{\Pi,\epsilon_0} &:= \min_{p_c \in \triangle_{\mc{A}}, \ \forall c \in \mc{C} }  \max_{\pi \in \Pi \setminus \pi_* } \frac{ \E_{c \sim \nu} \left[ \left(\frac{1}{p_{c,\pi(c)}}+ \frac{1}{p_{c,\pi_*(c)}}\right) \1\{\pi_*(c) \neq \pi(c) \} \right] }{(\E_{c \sim \nu}[ \, \v(c,\pi_*(c)) - \v(c,\pi(c)) \, ] \vee \epsilon_0)^2}\\
    &= \min_{p_c \in \triangle_{\mc{A}}, \ \forall c \in \mc{C} } \max_{\epsilon\geq \epsilon_0} \max_{\pi \in \Pi \setminus \pi_*:\Delta(\pi)\leq \epsilon }\frac{ \E_{c \sim \nu} \left[ \left(\frac{1}{p_{c,\pi(c)}}+ \frac{1}{p_{c,\pi_*(c)}}\right) \1\{\pi_*(c) \neq \pi(c) \} \right] }{\epsilon^2}\\
    &= \min_{p_c \in \triangle_{\mc{A}}, \ \forall c \in \mc{C} } \max_{\epsilon\geq \epsilon_0} \max_{\pi \in \Pi \setminus \pi_*:\Delta(\pi)\leq \epsilon }\frac{ \E_{c \sim \nu} \left[ \left(\frac{1}{p_{c,\pi(c)}}+ \frac{1}{p_{c,\pi_*(c)}}\right) \1\{\pi_*(c) \neq \pi(c),\Delta(\pi)\leq\epsilon \} \right] }{\epsilon^2}\\
    &\leq \min_{p_c \in \triangle_{\mc{A}}, \ \forall c \in \mc{C} } \max_{\epsilon\geq \epsilon_0} \max_{\pi \in \Pi \setminus \pi_*:\Delta(\pi)\leq \epsilon }\frac{ \E_{c \sim \nu} \left[ \left(\frac{1}{p_{c,\pi(c)}}+ \frac{1}{p_{c,\pi_*(c)}}\right) \1\{\exists\pi\in\Pi:\pi_*(c) \neq \pi(c),\Delta(\pi)\leq\epsilon \} \right] }{\epsilon^2}\\
    &\overset{(i)}{\leq} \max_{\epsilon\geq \epsilon_0} \max_{\pi \in \Pi \setminus \pi_*:\Delta(\pi)\leq \epsilon }\frac{ \E_{c \sim \nu} \left[ \left(|\A|+|\A|\right) \1\{\exists\pi\in\Pi:\pi_*(c) \neq \pi(c),\Delta(\pi)\leq\epsilon \} \right] }{\epsilon^2}\\
    &= \max_{\epsilon\geq \epsilon_0}\frac{ 2|\A|\E_{c \sim \nu} \left[  \1\{\exists\pi\in\Pi:\pi_*(c) \neq \pi(c),\Delta(\pi)\leq\epsilon \} \right] }{\epsilon^2}\leq \frac{2|\A|}{\epsilon_0} \mathfrak{C}_{\Pi}^{\mathsf{csc}}(\epsilon_0),
\end{align*}
where $(i)$ follows from taking $p_c\in\triangle_\A$ to be the uniform distribution over all actions for each $c\in\C$. 
To relate this to the policy disagreement coefficient, note that \begin{align*}
    \Delta(\pi)=\E_{c\sim\nu}[r(c,\pi_*(c))-r(c,\pi(c))]&\geq \E_{c\sim\nu}[\1\{\pi(c)\ne\pi_*(c)\}(\min_{c\in\C}\min_{a\in\A}r(c,\pi_*(c))-r(c,a))]\\
    &= \P_\nu(\pi(c)\ne\pi_*(c))\Delta_{\mathsf{uniform}}.
\end{align*}
Therefore,
\begin{align*}
    &\max_{\epsilon\geq \epsilon_0}\frac{ 2|\A|\E_{c \sim \nu} \left[  \1\{\exists\pi\in\Pi:\pi_*(c) \neq \pi(c),\Delta(\pi)\leq\epsilon \} \right] }{\epsilon^2}\\
    &\leq \max_{\epsilon\geq \epsilon_0}\frac{ 2|\A|\E_{c \sim \nu} \left[  \1\{\exists\pi\in\Pi:\pi_*(c) \neq \pi(c),\P_\nu(\pi(c)\ne\pi_*(c))\leq\frac{\epsilon}{\Delta_{\mathsf{uniform}}} \} \right] }{\epsilon^2}\\
    &\leq \frac{2|\A|}{\epsilon_0\Delta_{\mathsf{uniform}}}\mathfrak{C}_{\Pi}^{\mathsf{pol}}(\epsilon_0/\Delta_{\mathsf{uniform}}).
\end{align*}
\end{proof}

\section{Proof for sample complexity of Algorithm \ref{alg:elimination} and~\ref{alg:naive}}\label{sec:proof_rage_correctness}

\begin{proof}[Proof of Lemma~\ref{lem:rage_correctness}]
For any $\mc{V} \subseteq \Pi$ and $\pi \in \mc{V}$ define the event
\begin{align*}
\mc{E}_{\pi,\ell}( \mc{V} ) = \{ |\hat{o}_{\pi_*,\pi,\ell}(\mc{V}) - \langle \phi_{\pi_*} - \phi_{\pi},\theta_* \rangle | \leq \epsilon_\ell \} 
\end{align*}
where it is implicit that $\hat{o}_{\pi_*,\pi,\ell}:=\hat{o}_{\pi_*,\pi,\ell}(\mc{V})$ is the resulting estimate after round $\ell$ if $\Pi_\ell$ had been equal to $\mc{V}$. 
Define $w_\ell(\mc{V})$ and $\tau_\ell(\mc{V})$ analogously.
By the properties of the Catoni estimator, we have for any $\mc{V} \subset \Pi$ with probability at least $1-\frac{\delta}{2 \ell^2 |\Pi|}$ that 
\begin{align*}
|\hat{o}_{\pi_*,\pi,\ell}(\mc{V}) - \langle \phi_{\pi_*} - \phi_{\pi},\theta_* \rangle |
&\leq \| \phi_{\pi_*} - \phi_{\pi} \|_{A(w_\ell(\mc{V}))^{-1}} \sqrt{ \frac{2 \log(2 \ell^2 |\Pi|/\delta )}{\tau_\ell(\mc{V}) - \log(2 \ell^2 |\Pi|/\delta )}} \\
&\leq \sqrt{ \frac{\| \phi_{\pi_*} - \phi_{\pi} \|^2_{A(w_\ell(\mc{V}))^{-1}}}{2 \epsilon_\ell^{-2} \rho(w_\ell(\mc{V}),\mc{V}) \log(2 \ell^2 |\Pi|/\delta)}} \sqrt{2 \log(2 \ell^2 |\Pi|/\delta)} = \epsilon_\ell.
\end{align*}
Consequently, 
\begin{align*}
\P\left( \bigcup_{\ell=1}^\infty \bigcup_{\pi \in \Pi_{\ell}} \{ \mc{E}^c_{\pi,\ell}( \Pi_\ell ) \} \right) &\leq \sum_{\ell=1}^\infty \P\left( \bigcup_{\pi \in \Pi_{\ell}} \{ \mc{E}^c_{\pi,\ell}( \Pi_{\ell} ) \} \right) \\
&= \sum_{\ell=1}^\infty \sum_{\mc{V} \subseteq \Pi} \P\left( \bigcup_{\pi \in \mc{V}} \{ \mc{E}^c_{\pi,\ell}( \mc{V} ) \} , \Pi_\ell = \mc{V}\right) \\
&= \sum_{\ell=1}^\infty \sum_{\mc{V} \subseteq \Pi} \P\left(\bigcup_{\pi \in \mc{V}} \{ \mc{E}^c_{\pi,\ell}( \mc{V} ) \} \right) \P( \Pi_\ell = \mc{V}) \\
&\leq \sum_{\ell=1}^\infty  \sum_{\mc{V} \subseteq \Pi} \tfrac{\delta |\mc{V}|}{2\ell^2 |\Pi|} \P( \Pi_\ell = \mc{V})  \leq \delta.
\end{align*}
Thus, assume $\bigcap_{\ell=1}^\infty \bigcap_{\pi \in \Pi_{\ell}} \{ \mc{E}_{\pi,\ell}( \Pi_\ell ) \} $ holds.
For any $\pi \in \Pi_\ell$ we have
\begin{align*}
\hat{o}_{\pi,\pi_*,\ell} &= \hat{o}_{\pi,\pi_*,\ell} - \langle \phi_{\pi} - \phi_{\pi_*},\theta_* \rangle + \phi_{\pi_*},\theta_* \rangle  \\
&\leq \epsilon_\ell + \langle \phi_{\pi} - \phi_{\pi_*},\theta_* \rangle \leq \epsilon_\ell
\end{align*}
which implies that $\pi_*$ would survive to round $\ell+1$.
And for any $\pi' \in \Pi_\ell$ such that $\langle \phi_{\pi_*}-\phi_{\pi'}, \theta^* \rangle > 2 \epsilon_\ell$ we have 
\begin{align*}
\max_{\pi \in \Pi_\ell} \hat{o}_{\pi,\pi',\ell} &\geq \hat{o}_{\pi_*,\pi',\ell} \\
&= \langle \phi_{\pi'} - \phi_{\pi_*},\theta_* \rangle - \hat{o}_{\pi',\pi_*,\ell} + \langle \phi_{\pi_*} - \phi_{\pi'},\theta_* \rangle \\
&> - \epsilon_\ell +  2 \epsilon_\ell = \epsilon_\ell
\end{align*}
which implies this $\pi'$ would be kicked out.
Note that this implies that $\max_{\pi \in \Pi_{\ell+1}} \langle \phi_{\pi_*} - \phi_\pi, \theta^* \rangle \leq 2 \epsilon_\ell = 4 \epsilon_{\ell+1}$.
\end{proof}

In the remaining of this section we provide a proof for the sample complexity of Algorithm~\ref{alg:naive}. 

\begin{theorem}
Under $\EE$, for all $\ell\in\N$, the following holds: 
\begin{enumerate}
    \item $\hat{\pi}_\ell\in S_\ell:=\{\pi\in\Pi: V(\pi_*)-V(\pi)\leq \epsilon_\ell\}$; 
    \item $n_\ell\lesssim \min_{w\in\Omega}\max_{\pi\in\Pi}\frac{\norm{\phi_{\pi_*}-\phi_{\pi}}_{A(w)^{-1}}^2\log(1/\delta_l)}{\epsilon_l^2+\Delta(\pi)^2}$. 
\end{enumerate}
\end{theorem}

Without loss of generality, we assume that $\forall t$, the reward $r_t\in[0,1]$. Note that by the result about Catoni estimator in \cite{lugosi2019mean}, we have 
for all $\ell \in \mathbb{N}$ and $\pi,\pi' \in \Pi$, that
\begin{align*}
|{\sf Cat}( \{ \inner{\phi_\pi - \phi_{\pi'}}{O_{t}} \}_{t=1}^{n_\ell}) - \langle \phi_{\pi} - \phi_{\pi'},\theta_* \rangle |
&\leq \| \phi_{\pi} - \phi_{\pi'} \|_{A(w^{(\ell)})^{-1}} \sqrt{ \frac{2 \log(2 \ell^2 |\Pi|/\delta )}{n_\ell - \log(2 \ell^2 |\Pi|/\delta )}}.
\end{align*}
Therefore, in the $\ell$th round, we have for any $\pi,\pi'\in\Pi$,
\begin{align}
    \left|\hat{\Delta}_l(\pi,\pi')-\Delta(\pi,\pi')\right|&=\left|\mathsf{Cat}(\{\langle \phi_{\pi} - \phi_{\pi'}, O_i \rangle\}_{i=1}^{n_{\ell}})-\inner{\phi_\pi-\phi_{\pi'}}{\theta_*}\right|\nonumber\\
    &\leq\sqrt{ \frac{2\| \phi_{\pi} - \phi_{\pi'} \|_{A(w^{(\ell)})^{-1}}^2 \log(2 \ell^2 |\Pi|/\delta )}{n_\ell}}.\label{eqn:11}
\end{align}
Then, let $\delta_l=\frac{\delta}{2l^2|\Pi|}$ we define the event \[\EE_{l}=\bigcap_{\pi,\pi'\in\Pi}\left\{\left|\hat{\Delta}_{l}(\pi,\pi')-\Delta(\pi,\pi')\right|\leq \sqrt{ \frac{2\| \phi_{\pi} - \phi_{\pi'} \|_{A(w^{(\ell)})^{-1}}^2 \log(1/\delta_l )}{n_\ell}}\right\},\]

and $\EE=\bigcap_{l=0}^\infty\EE_l$. First, by equation~\ref{eqn:11}, we have that $\EE$ happens with probability at least $1-\delta$. 
In order to show the sample complexity lower bound, we use proof by induction. Note that in a step of Lemma~\ref{lem:C_44}, we can show that $n_l\lesssim \min_{w\in\Omega}\max_{\pi\in\Pi}\frac{\norm{\phi_{\pill}-\phi_{\pi}}_{A(w)^{-1}}^2\log(1/\delta_l)}{\epsilon_l^2+\Delta(\pi)^2}$, so we induct on this result. Assume in round $l-1$, $\pill\in S_{l-1}=\{\pi\in\Pi:\Delta(\pi,\pi_*)\leq \epsilon_{l-1}\}$ and $n_{l-1}\lesssim \min_{w\in\Omega}\max_{\pi\in\Pi}\frac{\norm{\phi_{\hat{\pi}_{l-2}}-\phi_\pi}_{A(w)^{-1}}^2\log((l-1)^2|\Pi|^2/\delta)}{\epsilon_{l-1}^2+\Delta(\pi)^2}$. Then, the following lemma gives us an upper bound on the UCB. 
\begin{lemma}\label{lem:1.8}
We have for any $\pi\in\Pi$,
\[\sqrt{\frac{\norm{\phi_{\hat{\pi}_l}-\phi_\pi}_{A(w^{(\ell)})^{-1}}^2\log(1/\delta_l)}{n_{l}}}\leq \frac{1}{28}\bigsmile{4\epsilon_{l}+\hat{\Delta}_{l-1}(\pi,\hat{\pi}_{l-1})}.\]
\end{lemma}
\begin{proof}
By definition of $n_l$ and $w^{(\ell)}$ and $\pi^{(\ell)}$ being the saddle point, we have
\begin{align*}
    &-\frac{1}{4}\hat{\Delta}_{l-1}(\pi^{(\ell)},\pill)+28\sqrt{ \frac{2\| \phi_{\pi^{(\ell)}} - \phi_{\pill} \|_{A(w^{(\ell)})^{-1}}^2 \log(1/\delta_l )}{n_\ell}}\\
    &=\max_{\pi\in\Pi}-\frac{1}{4}\hat{\Delta}_{l-1}(\pi,\pill)+\sqrt{ \frac{1568\| \phi_{\pi} - \phi_{\pill} \|_{A(w^{(\ell)})^{-1}}^2 \log(1/\delta_l )}{n_\ell}}\leq \epsilon_l.
\end{align*}
Solving for $n_l$ gives us 
\begin{equation*}
    n_l\geq \max_{\pi\in\Pi}\frac{1568\norm{\phi_{\pi}-\phi_{\hat{\pi}_{l-1}}}_{A(w^{(\ell)})^{-1}}^2\log(1/\delta_l)}{(4\epsilon_l+\hat{\Delta}_{l-1}(\pi,\pill))^2}.
\end{equation*}
We have for any $\pi\in\Pi$, 
\begin{align*}
    2n_l &\geq 3136 \max_{\pi\in\Pi}\frac{\norm{\phi_{\pill}-\phi_{\pi}}_{A(w^{(\ell)})^{-1}}^2\log(1/\delta_l)}{(4\epsilon_l+\hat{\Delta}_{l-1}(\pi,\hat{\pi}_{l-1}))^2}\\
    &\geq 1568\frac{\norm{\phi_{\pill}-\phi_{\pi}}_{A(w^{(\ell)})^{-1}}^2\log(1/\delta_l)}{(4\epsilon_{l}+\hat{\Delta}_{l-1}(\pi,\hat{\pi}_{l-1}))^2}\\
    &\qquad+1568\frac{\norm{\phi_{\pill}-\phi_{\pil}}_{A(w^{(\ell)})^{-1}}^2\log(1/\delta_l)}{(4\epsilon_{l}+\hat{\Delta}_{l-1}(\pil,\pill))^2}\\
    &\overset{(i)}{\geq} 1568\frac{\Big(\norm{\phi_{\pill}-\phi_{\pi}}_{A(w^{(\ell)})^{-1}}^2+\norm{\phi_{\pill}-\phi_{\hat{\pi}_l}}_{A(w^{(\ell)})^{-1}}^2\Big)\log(1/\delta_l)}{\max\{(4\epsilon_{l}+\hat{\Delta}_{l-1}(\pil,\pill))^2, (4\epsilon_{l}+\hat{\Delta}_{l-1}(\pi,\pill))^2\}}\\
    &\overset{(ii)}{\geq} 1568\frac{\norm{\phi_{\hat{\pi}_{l}}-\phi_{\pi}}_{A(w^{(\ell)})^{-1}}^2\log(1/\delta_l)}{\max\{(4\epsilon_{l}+\hat{\Delta}_{l-1}(\pil,\pill))^2, (4\epsilon_{l}+\hat{\Delta}_{l-1}(\pi,\pill))^2\}}.
\end{align*}
where $(i)$ holds by lower bounding the ratio with a larger denominator, and $(ii)$ holds by triangular inequality. Therefore, using the fact that $\hat{\Delta}(\pi,\pill)\geq 0$ for any $\pi\in\Pi$ since $\pill=\arg\max_{\pi\in\Pi}\hat{V}_{l-1}(\pi)$, we have $\textstyle\sqrt{\max\{(4\epsilon_{l}+\hat{\Delta}_{l-1}(\pil,\pill))^2, (4\epsilon_{l}+\hat{\Delta}_{l-1}(\pi,\pill))^2\}}=\max\{4\epsilon_{l}+\hat{\Delta}_{l-1}(\pil,\pill), 4\epsilon_{l}+\hat{\Delta}_{l-1}(\pi,\pill)\},$ so we have
\[\sqrt{\frac{\norm{\phi_{\pil}-\phi_{\pi}}_{A(w^{(\ell)})^{-1}}^2\log(1/\delta_l)}{n_{l}}}\leq \frac{1}{28}\bigsmile{4\epsilon_{l}+\max\{\hat{\Delta}_{l-1}(\pi,\pill),\hat{\Delta}_{l-1}(\pil,\pill)\}}.\]
\end{proof}

With the above results, the following lemma controls the difference between the empirical gap and the true gap. 
\begin{lemma}\label{lem:diff_emp_true_gap}
With inductive hypotheses, we have for any $\pi\in\Pi$,
$$
|\widehat{\Delta}_{l-1}\left(\pi, \pill\right)-\Delta\left(\pi, \pi_*\right)| \leq 2 \epsilon_{l-1}+ \frac{1}{4}\Delta(\pi,\pi_*).
$$
\end{lemma}
\begin{proof}
We prove this by induction. First, in round $l=0$, this holds by choosing a sufficiently large $n_0$. Then, in round $l-1$,  
\begin{align*}
    &|\widehat{\Delta}_{l-1}\left(\pi, \hat{\pi}_{l-1}\right)-\Delta\left(\pi, \pi_*\right)|\\
    &= |\widehat{\Delta}_{l-1}\left(\pi, \pill\right)-\Delta\left(\pi, \pill\right)-\Delta\left(\pill, \pi_*\right)|\\
    &\leq \sqrt{\frac{2\norm{\phi_{\pi}-\phi_{\pill}}_{A(w^{(\ell-1)})^{-1}}^2\log(1/\delta_{l-1})}{n_{l-1}}}+ \epsilon_{l-1}\\
    &\overset{(i)}{\leq} \frac{\sqrt{2}}{28}\bigsmile{4\epsilon_{l-1}+\max\{\hat{\Delta}_{l-2}(\pi,\hat{\pi}_{l-2}),\hat{\Delta}_{l-2}(\pill,\hat{\pi}_{l-2})\}}+\epsilon_{l-1}\\
    &\overset{(ii)}{\leq} \frac{\sqrt{2}}{28}\bigsmile{4\epsilon_{l-1}+2\epsilon_{l-2}+\frac{5}{4}\Delta(\pi,\hat{\pi}_{l-2})+2\epsilon_{l-2}+\frac{5}{4}\Delta(\pill,\hat{\pi}_{l-2})}+\epsilon_{l-1}\\
    &\leq \frac{\sqrt{2}}{28}\bigsmile{4\epsilon_{l-1}+4\epsilon_{l-2}+\frac{5}{4}\Delta(\pi,\pi_*)+\frac{5}{4}\Delta(\pill,\pi_*)}+\epsilon_{l-1}\\
    &\leq \frac{\sqrt{2}}{28}\bigsmile{4\epsilon_{l-1}+4\epsilon_{l-2}+\frac{5}{4}\Delta(\pi,\pi_*)+\frac{5}{4}\epsilon_{l-1}}+\epsilon_{l-1}\\
    &\leq 2\epsilon_{l-1}+\frac{1}{4}\Delta(\pi,\pi_*),
\end{align*}
where $(i)$ follows from the preceding lemma and $(ii)$ follows from the inductive hypothesis that 
\[|\hat{\Delta}_{l-2}(\pi,\hat{\pi}_{l-2})-\Delta(\pi,\pi_*)|\leq 2 \epsilon_{l-2}+ \frac{1}{4}\Delta(\pi,\pi_*).\]
\end{proof}

We make use of these two lemmas to state a lower bound on $n_l$. 

\begin{lemma}\label{lem:C_44}
Under $\EE$, the choice for $n_l$ in the algorithm satisfies \[n_l\lesssim \min_{w\in\Omega}\max_{\pi\in\Pi}\frac{\norm{\phi_{\pi_*}-\phi_{\pi}}_{A(w)^{-1}}^2\log(1/\delta_l)}{\epsilon_l^2+\Delta(\pi)^2}.\]
\end{lemma}
\begin{proof}
By inductive hypothesis on $n_{l-1}$ and under $\EE_l$, we have for any $\pi\in\Pi$, 
\begin{align*}
    \Delta(\pi,\pi_*)&=\Delta(\pi,\pill)+\Delta(\pill,\pi_*)\\
    &\overset{(i)}{\leq} \hat{\Delta}_{l-1}(\pi,\hat{\pi}_{l-1})+\sqrt{\frac{2\norm{\phi_{\pill}-\phi_{\pi}}_{A(w^{(\ell-1)})^{-1}}^2\log((l-1)^2|\Pi|^2/\delta)}{n_{l-1}}}+\epsilon_{l-1}\\
    &\overset{(ii)}{\leq} \hat{\Delta}_{l-1}(\pi,\hat{\pi}_{l-1})+\frac{\sqrt{2}}{28}\bigsmile{4\epsilon_{l-1}+\hat{\Delta}_{l-2}(\pi,\hat{\pi}_{l-2})}+\epsilon_{l-1}\\
    &\leq \hat{\Delta}_{l-1}(\pi,\hat{\pi}_{l-1})+\frac{\sqrt{2}}{28}\bigsmile{4\epsilon_{l-1}+\frac{5}{4}\Delta(\pi,\pi_*)+2\epsilon_{l-2}}+\epsilon_{l-1}\\
    &\leq \hat{\Delta}_{l-1}(\pi,\hat{\pi}_{l-1})+\frac{1}{4}\Delta(\pi,\pi_*)+2\epsilon_{l-1}.
\end{align*}
where $(i)$ follows from $\EE_{l-1}$ and $(ii)$ follows from Lemma~\ref{lem:1.8}.
Therefore, 
\begin{align*}
    &\min_{w\in\Omega}\max_{\pi\in\Pi}-\frac{1}{4}\hat{\Delta}_{l-1}(\pi,\pill)+28\sqrt{\frac{2\norm{\phi_{\pi}-\phi_{\pill}}_{A(w)^{-1}}^2\log(1/\delta_l)}{n_l}}\\
    &\leq\min_{w\in\Omega}\max_{\pi\in\Pi}-\frac{3}{16}\Delta(\pi,\pi_*)+\frac{1}{2}\epsilon_l+28\sqrt{\frac{2\norm{\phi_{\pi}-\phi_{\pill}}_{A(w)^{-1}}^2\log(1/\delta_l)}{n_l}}\\
    &\leq\min_{w\in\Omega} \max_{\pi\in\Pi}\Big(-\frac{3}{16}\Delta(\pi,\pi_*)+28\sqrt{\frac{2\norm{\phi_{\pi_*}-\phi_\pi}_{A(w)^{-1}}^2\log(1/\delta_l)}{n_l}}\\
    &\qquad+28\sqrt{\frac{2\norm{\phi_{\pi_*}-\phi_{\pill}}_{A(w)^{-1}}^2\log(1/\delta_l)}{n_l}}\Big)+\frac{1}{2}\epsilon_l\\
    &\leq\min_{w\in\Omega} \max_{\pi\in\Pi}\left(-\frac{3}{16}\Delta(\pi,\pi_*)+28\sqrt{\frac{2\norm{\phi_{\pi_*}-\phi_\pi}_{A(w)^{-1}}^2\log(1/\delta_l)}{n_l}}\right.\\
    &\qquad\left.+28\sqrt{\max_{\pi'\in S_{l-1}}\frac{2\norm{\phi_{\pi_*}-\phi_{\pi'}}_{A(w)^{-1}}^2\log(1/\delta_l)}{n_l}}\right)+\frac{1}{2}\epsilon_l
\end{align*}
which is less than $\epsilon_l$
whenever
\[n_l\gtrsim\min_{w\in\Omega}\max_{\pi\in\Pi}\frac{\norm{\phi_{\pi_*}-\phi_{\pi}}_{A(w)^{-1}}^2\log(1/\delta_l)}{\epsilon_l^2+\Delta(\pi,\pi_*)^2}.\]
\end{proof}
Then we finish our first goal. The next goal is to show that $\hat{\pi}_l\in S_l$. 

\begin{lemma}
Under $\EE_l$, we have $\Delta(\hat{\pi}_l,\pi_*)\leq \epsilon_l$. 
\end{lemma}
\begin{proof}
On $\EE_l$, we have 
\begin{align*}
    &\Delta(\hat{\pi}_l, \hat{\pi}_{l-1})\\ &\leq\hat{\Delta}_{l}(\hat{\pi}_l, \hat{\pi}_{l-1})+\sqrt{\frac{2\norm{\phi_{\pil}-\phi_{\pill}}_{A(w^{(\ell)})^{-1}}^2\log(1/\delta_l)}{n_l}}\tag{by event $\EE_l$}\\ &\leq\hat{\Delta}_{l}(\pi_*, \hat{\pi}_{l-1})+\sqrt{\frac{2\norm{\phi_{\pil}-\phi_{\pill}}_{A(w^{(\ell)})^{-1}}^2\log(1/\delta_l)}{n_l}}\tag{by minimality of $\pil$}\\
    &\leq\Delta(\pi_*, \hat{\pi}_{l-1})+\sqrt{\frac{2\norm{\phi_{\pill}-\phi_{\pi_*}}_{A(w^{(\ell)})^{-1}}^2\log(1/\delta_l)}{n_l}}+\sqrt{\frac{2\norm{\phi_{\pil}-\phi_{\pill}}_{A(w^{(\ell)})^{-1}}^2\log(1/\delta_l)}{n_l}}\tag{by event $\EE_l$}\\
    &\leq \Delta(\pi_*, \hat{\pi}_{l-1})+\frac{\sqrt{2}}{28}\bigsmile{4\epsilon_{l}+\hat{\Delta}_{l-1}(\pi_*,\pill)+4\epsilon_{l}+\hat{\Delta}_{l-1}(\pil,\pill)}\tag{by Lemma~\ref{lem:1.8}}\\
    &\leq \Delta(\pi_*, \hat{\pi}_{l-1})+\frac{\sqrt{2}}{28}\bigsmile{4\epsilon_{l}+2\epsilon_{l-1}+\frac{5}{4}\Delta(\pi_*,\pill)+4\epsilon_{l}+2\epsilon_{l-1}+\frac{5}{4}\Delta(\pil,\pill)}\tag{by Lemma~\ref{lem:diff_emp_true_gap}}\\
    &\leq\Delta(\pi_*,\pill)+\frac{3}{56}\bigsmile{8\epsilon_{l-1}+\frac{5}{4}\Delta(\pil,\pi_*)}.
\end{align*}
Therefore, $\frac{209}{224}\Delta(\pil,\pi_*)\leq \frac{6}{7}\epsilon_{l}$ and $\Delta(\pil,\pi_*)\leq\epsilon_l$, so $\pil\in S_l$.
\end{proof}

\section{Proof of the \textsf{FW-GD} subroutine}\label{sec:proof_FW_GD}

In this section, we aim to prove Theorem~\ref{thm:computational_efficiency}. Specifically, Section~\ref{sec:proof_csc} quantifies the number of oracle calls, and Section~\ref{sec:offline} quantifies the number of offline data needed in order to approximate the expectation over the context distribution. In particular, the size of the history follows directly from Lemma~\ref{lem:history_1} and \ref{lem:history_2}. We will see that $\eta,\gamma_{\max},\gamma_{\min}$ all scale at most polynomially on $|\A|$ and $\epsilon^{-1}$. We leave the convergence analysis of the algorithm in Section~\ref{sec:conv_analysis}. In particular, we will see in Theorem~\ref{thm:FW_gap} that $K_l=\operatorname{poly}(|\A|,\epsilon_l^{-1})$, which shows that the total number of oracle calls is at most $\operatorname{poly}(|\A|,\epsilon^{-1},\log(1/\delta),\log(|\Pi|))$. Combining all results above gives Theorem~\ref{thm:computational_efficiency}. 

\subsection{Proof of computational efficiency}\label{sec:proof_csc}

In this section, we address the technical issues on computational efficiency of our algorithm. Fix an iteration $t$ and let $K_l$ be the number of iterations for \textsf{FW-GD} in the $l$th round. 
\begin{lemma}\label{lem:csc}
Equation~\eqref{eqn:csc} can be computed with $(t+T_{l-1})|\D|$ call to a cost-sensitive classification oracle. 
\end{lemma}

\begin{proof}
We consider the $t$th iteration of the $l$th round for some $n_r$. In this iteration, we compute
\begin{align*}
    [\nabla_\lambda h_l(\lambda^{t},&\gamma^{t},n_r)]_\pi=\sum_{i=1}^{n_l}\frac{r_i}{p^{(\ell)}_{c_i,a_i}+[\gamma^{l-1}]_\pi}\left(\1\{\pi(c_i)=a_i\}-\1\{\pill(c_i)=a_i\}\right)+\frac{\log(1/\delta_l)}{[\gamma^{t}]_\pi n}\\
    &+\E_{c\sim\nu_\D}\bigbrak{\bigsmile{\sum_{a\in\A}\sqrt{(\lambda^{t}\odot \gamma^{t})^\T (t_a^{(c)}+\eta_l)}}\bigsmile{\sum_{a'\in\A}\frac{[\gamma^{t}]_\pi(t_{a'}^{(c)}+\eta_l)_\pi}{\sqrt{(\lambda^{t}\odot \gamma^{t})^\T (t_{a'}^{(c)}+\eta_l)}}}}.
\end{align*}

Define $\gamma_0:=\sqrt{\frac{\log(1/\delta_l)}{|\A|^2\eta_l n_r}}$. Initially, each coordinate of $\gamma^{t}$ is $\gamma_0$. In round $t$ of the algorithm, at most $t$ coordinates of $\gamma$ will change, and these coordinates will be in $\supp(\lambda^{t})$. Also, for any $j\not\in\supp(\lambda^{l-1})$, $\gamma^{l-1}_j=\gamma_0$.  Therefore, let $t_{a}^{(c)}(\cdot,\pill)\in\R^{|\Pi|}$, in round $l$,  
\begin{align*}
    &\underset{\pi\in\Pi\setminus(\operatorname{supp}(\lambda^{t})\cup \supp(\lambda^{l-1}))}{\operatorname{argmax}}\bigbrak{\nabla_\lambda h_{l}(\lambda^{t},\gamma^{t},n_r)}_\pi\\
    &=\underset{\pi\in\Pi\setminus(\operatorname{supp}(\lambda^{t})\cup \supp(\lambda^{l-1}))}{\operatorname{argmax}} \sum_{i=1}^{n_l}\frac{r_i}{p^{(\ell)}_{c_i,a_i}+\gamma_0}\1\{\pi(c_i)=a_i\}+\frac{\log(1/\delta_l)}{\gamma_0 n_r}\\
    &\qquad+\E_{c\sim\nu_\D}\bigbrak{\bigsmile{\sum_{a\in\A}\sqrt{(\lambda^{t}\odot \gamma^{t})^\T (t_a^{(c)}(\pill)+\eta_l)}}\bigsmile{\sum_{a'\in\A}\frac{\gamma_0(t_{a'}^{(c)}(\pill)+\eta_l)_\pi}{\sqrt{(\lambda^{t}\odot \gamma^{t})^\T (t_{a'}^{(c)}(\pill)+\eta_l)}}}}\\
    &= \underset{\pi\in\Pi\setminus(\operatorname{supp}(\lambda^{t})\cup \supp(\lambda^{l-1}))}{\operatorname{argmax}}\sum_{i=1}^{n_l}\frac{r_i}{p^{(\ell)}_{c_i,a_i}+\gamma_0}\1\{\pi(c_i)=a_i\}\\
    &\qquad+\E_{c\sim\nu_\D}\bigbrak{\sum_{a'\in\A}\frac{\sum_{a\in\A}\sqrt{(\lambda^{t}\odot \gamma^{t})^\T (t_a^{(c)}(\pill)+\eta_l)}}{\sqrt{(\lambda^{t}\odot \gamma^{t})^\T (t_{a'}^{(c)}(\pill)+\eta_l)}}\gamma_0t_{a'}^{(c)}(\pill)_\pi}\\
    &= \underset{\pi\in\Pi\setminus(\operatorname{supp}(\lambda^{t})\cup \supp(\lambda^{l-1}))}{\operatorname{argmax}}\sum_{i=1}^{n_l+|\D|}L_i(\pi(c_i))
\end{align*}
which is a cost-sensitive classification problem with cost vector
\[
L_i(a)=
\begin{cases}
    \frac{r_i}{p^{(\ell)}_{c_i,a_i}+\gamma_0}\1\{a=a_i\} &\textrm{ for }i=1,\cdots,n_l\\
    \left(\frac{\gamma_0}{s_{a,c_i}}+\frac{\gamma_0}{s_{\pill(c_i),c_i}}\right)\1\{a\neq \pill(c_i)\}&\textrm{ for }i=n_l+1,\cdots,n_l+|\D|
\end{cases}
\]
where $s_{a,c}=\frac{\sqrt{(\lambda^{t}\odot \gamma^{t})^\T (t_{a}^{(c)}(\pill)+\eta_l)}}{\sum_{a'\in\A}\sqrt{(\lambda^{t}\odot \gamma^{t})^\T (t_{a'}^{(c)}(\pill)+\eta_l)}}$. Note that $s_{a,c}$ is computable since $\lambda^{t}$ has at most $t$ non-zero elements in step $t$. Then, let $\pi^\sharp:=\operatorname{supp}(\lambda^{t})\cup\supp(\lambda^{l-1})$, we have 
\begin{align*}
    &\arg\max_{\pi\in\Pi}\bigbrak{\nabla_\lambda h_{l}(\lambda^{t},\gamma^{t},n_r)}_\pi\\
    &= \arg\max\left\{\underset{\pi\in\Pi^\sharp}{\operatorname{argmax}}\bigbrak{\nabla_\lambda h_{l}(\lambda^{t},\gamma^{t},n_r)}_\pi, \underset{\pi\in\Pi\setminus\Pi^\sharp}{\operatorname{argmax}}\bigbrak{\nabla_\lambda h_{l}(\lambda^{t},\gamma^{t},n_r)}_\pi\right\}.
\end{align*}
The first piece could be found directly since $\operatorname{supp}(\lambda^{t})\cup\supp(\lambda^{l-1})\leq t+T_{l-1}$. The second piece could be computed with $(t+T_{l-1})|\D|$ calls to a constrained cost-sensitive classification oracle, stated in Lemma~\ref{lem:restricted_csc} below. 
\end{proof}

\begin{lemma}\label{lem:restricted_csc}
For any set $B_t\subset\Pi$, we can compute $\underset{\pi\in\Pi\setminus B_t}{\operatorname{argmax}}\bigbrak{\nabla_\lambda h_{l}(\lambda^{t},\gamma^{t},n_r)}_\pi$ using $|B_t|\cdot|\D|$ calls to a constrained cost-sensitive classification oracle defined in Definition~\ref{def:csc}. 
\end{lemma}
\begin{proof}
Algorithm~\ref{alg:c_csc} below shows that we could compute this argmax via the $\textsf{C-AMO}$ oracle. First, by construction of the algorithm, we have that $\pi_e\not\in B_t$, so $\pi_e\in\Pi\setminus B_t$. It remains to show that $\pi_e$ achieves the maximum. We prove this via contradiction. Assume that there is some other $\pi'\ne \pi_e$ that satisfies $\pi'\not\in B_t$ and $\nabla_\lambda[h_l(\lambda,\gamma,n)]_{\pi'}>\nabla_\lambda[h_l(\lambda,\gamma,n)]_{\pi_e}$. By construction of our algorithm, we know that $\nabla_\lambda[h_l(\lambda,\gamma,n)]_{\pi_{k}}$ is non-increasing in $k$. We find the largest $0\leq j\leq i-1$ such that $$\nabla_\lambda[h_l(\lambda,\gamma,n)]_{\pi_{j+1}}\leq\nabla_\lambda[h_l(\lambda,\gamma,n)]_{\pi'}\leq \nabla_\lambda[h_l(\lambda,\gamma,n)]_{\pi_{j}}.$$ First, since $j$ is the largest, we have $\nabla_\lambda[h_l(\lambda,\gamma,n)]_{\pi_{j+1}}<\nabla_\lambda[h_l(\lambda,\gamma,n)]_{\pi'}$, i.e. the first inequality is strict. By assumption that $\pi'\not\in B_t$ and $\pi'\ne \pi_e$, we have $\pi'\ne\pi_k$, $\forall 0\leq k\leq i$. So $\exists c_0\in\D$ such that $\pi'(c_0)\ne \pi_j(c_0)$. Then we get a contradiction since in iteration $j$, at line 6 we should return $\pi'_{c_0}$ instead of $\pi_{j+1}$. Therefore, there does not exist such $\pi'$ and $\pi_e$ achieves the maximum. 
\end{proof}

\begin{algorithm}[!htb]
\caption{Constrained cost-sensitive classification}
\begin{algorithmic}[1]\label{alg:c_csc}
\REQUIRE policy set $\Pi$, set of policies to avoid $B_t$, objective function $h_l$, context history $\D$, tolerance $\epsilon$
\STATE $\pi_0=\argmax{\pi\in\Pi}\,[\nabla_\lambda h_l(\lambda,\gamma,n)]_\pi$, $i=0$
\WHILE{$\pi_i\in B_t$}
\FOR{$c\in\D$}
\STATE compute $\pi'_c=\argmax{\substack{
\pi\in\Pi \\
\pi(c)\ne \pi_i(c)}}\,[\nabla_\lambda h_l(\lambda,\gamma,n)]_\pi$ s.t. $[\nabla_\lambda h_l(\lambda,\gamma,n)]_\pi\leq [\nabla_\lambda h_l(\lambda,\gamma,n)]_{\pi_i}$
\ENDFOR
\STATE $\pi_{i+1}=\argmax{c\in\D}\,[\nabla_\lambda h_l(\lambda,\gamma,n)]_{\pi'_c}$
\STATE $i=i+1$
\ENDWHILE
\STATE $\pi_e=\pi_i$
\ENSURE $\pi_e$
\end{algorithmic}
\end{algorithm}

\begin{lemma}\label{lem:csc_2}
We can compute equation~\eqref{eqn:opt_2} with $K_l|\D|$ calls to a constrained argmax oracle. 
\end{lemma}

\begin{proof}
We follow the proof technique in Lemma~\ref{lem:csc} and break the argmin into two pieces with $\pi\in\supp(\lambda^{l})$ and $\pi\in\Pi\setminus\supp(\lambda^{l})$. We only show how to compute the second piece as the first piece could be compute directly. We know that $\hat{\Delta}_l^{\gamma^l}(\pi,\hat{\pi}_{l-1})=\sum_{i=1}^{n_l}\frac{r_i}{p^{(\ell)}_{c_i,a_i}+[\gamma^l]_\pi}(\1\{\pill(c_i)=a_i\}-\1\{\pi(c_i)=a_i\})$. Then, similar to proof of Lemma~\ref{lem:csc}, let $\gamma_\pi=\gamma_0$ for all $\pi\in\Pi\setminus\supp(\lambda^{l})$, we have
\begin{align*}
    &\argmin{\pi\in\Pi\setminus\supp(\lambda^{l})}\hat{\Delta}_l^{\gamma^l}(\pi,\hat{\pi}_{l-1})+\E_{c\sim\nu_{\mathcal{D}}} \left[ \left(\frac{[\gamma^l]_\pi}{p_{c,a}^{(\ell)}}+\frac{[\gamma^l]_\pi}{s_{a',c}}\right) \1\{\pill(c )\ne \pi(c )\} \right]+\frac{\log(1/\delta_l)}{[\gamma^l]_\pi n_l}\\
    &= \argmin{\pi\in\Pi\setminus\supp(\lambda^{l})}\sum_{i=1}^{n_l}\frac{r_i}{p^{(\ell)}_{c_i,a_i}+[\gamma^l]_\pi}\left(\1\{\pill(c_i)=a_i\}-\1\{\pi(c_i)=a_i\}\right)\\
    &\qquad+\E_{c\sim\nu_{\mathcal{D}}} \left[\left(\frac{[\gamma^l]_\pi}{p_{c,a}^{(\ell)}}+\frac{[\gamma^l]_\pi}{p_{c,a'}^{(\ell)}}\right) \1\{\pill(c )\ne \pi(c )\} \right]\\
    &= \argmin{\pi\in\Pi\setminus\supp(\lambda^{l})}\sum_{i=1}^{n_l}-\frac{r_i}{p^{(\ell)}_{c_i,a_i}+\gamma_0}\1\{\pi(c_i)=a_i\}\\
    &\qquad+\E_{c\sim\nu_{\mathcal{D}}} \left[\left(\frac{\gamma_0}{p_{c,a}^{(\ell)}}+\frac{\gamma_0}{p_{c,a'}^{(\ell)}}\right) \1\{\pill(c )\ne \pi(c )\} \right]\\
    &= \argmin{\pi\in\Pi\setminus\supp(\lambda^{l})}\sum_{i=1}^{n_l}\frac{r_i}{p^{(\ell)}_{c_i,a_i}+\gamma_0}\1\{\pi(c_i)=a_i\}\\
    &\qquad-\E_{c\sim\nu_{\mathcal{D}}} \left[\left(\frac{\gamma_0}{p_{c,a}^{(\ell)}}+\frac{\gamma_0}{p_{c,a'}^{(\ell)}}\right) \1\{\pill(c )\ne \pi(c )\} \right]
\end{align*}
which is a cost-sensitive classification problem with cost vector
\[
L_i(a)=
\begin{cases}
    \frac{r_i}{p^{(\ell)}_{c_i,a_i}+\gamma_0}\1\{a=a_i\} &\textrm{ for }i=1,\cdots,n_l\\
    -\left(\frac{\gamma_0}{p_{c_i,a}^{(\ell)}}+\frac{\gamma_0}{p_{c_i,\pill(c_i)}^{(\ell)}}\right)\1\{a\neq \pill(c_i)\}&\textrm{ for }i=n_l+1,\cdots,n_l+|\D|.
\end{cases}
\]
\end{proof}

\subsection{Quantify the offline data}\label{sec:offline}

We first prove a general result for an empirical process bound of the difference of the expectation and the truth in Lemma~\ref{lem:emp_process}.  

\begin{lemma}\label{lem:emp_process}
Let $m=|\D|$ and define some set $\mc{K}\subset \gamma_{\max}\triangle_\Pi$. Consider some function $u:\C\times \mc{K}\to \R$ with $c,\kappa\mapsto u(c,\kappa)$ and define $\F\triangleq\{c\mapsto u\left(c,\kappa\right): \kappa\in\mc{K}\}$. If 
\begin{enumerate}
    \item $u$ satisfies that for any $c\in\C$ and $\kappa\in\mc{K}$, $u(c,\kappa)\in[0,b]$ where $b<\infty$ is a uniform upper bound;
    \item there exists $L<\infty$ such that $\norm{u(\cdot,\kappa_1)-u(\cdot,\kappa_2)}_\F\leq L\norm{\kappa_1-\kappa_2}_1$.
\end{enumerate}
Then, with probability at least $1-\delta$, \[\sup_{\kappa\in\mc{K}}\left|\E_{c\sim\nu_\D}\bigbrak{u(c,\kappa)}-\E_{c\sim\nu}\bigbrak{u(c,\kappa)}\right|\leq \sqrt{\frac{b^2}{2m}\log\bigsmile{\frac{2}{\delta}}}+\frac{16}{\sqrt{m}}L\gamma_{\max}\sqrt{2k\log(3e|\Pi|/k)}.\]
\end{lemma}
\begin{proof}
By the bounded condition on $u$ we have $\{\E_{c\sim\nu_\D}\bigbrak{u(c,\kappa)}: \kappa\in\mc{K}\}$ satisfies the bounded difference property with parameter $b$. Then we use McDiarmid's inequality to get with probability at least $1-\delta$,
\begin{align*}
    &\sup_{\kappa\in\mc{K}}\left|\E_{c\sim\nu_\D}\bigbrak{u(c,\kappa)}-\E_{c\sim\nu}\bigbrak{u(c,\kappa)}\right|\\
    &\leq \sqrt{\frac{b^2}{2m}\log\left(\frac{2}{\delta}\right)}+\E\left[\sup_{\kappa\in\mc{K}}\left|\E_{c\sim\nu_\D}\bigbrak{u(c,\kappa)}-\E_{c\sim\nu}\bigbrak{u(c,\kappa)}\right|\right].
\end{align*}
Also, note that by definition of $\F$ and classical results on entropy integral \cite{van2000asymptotic},
\[\E\left[\sup_{\kappa\in\mc{K}}\left|\E_{c\sim\nu_\D}\bigbrak{u(c,\kappa)}-\E_{c\sim\nu}\bigbrak{u(c,\kappa)}\right|\right]\leq\frac{8}{\sqrt{n}}\sup_Q \int_{0}^{\infty} \sqrt{\log N(\F, L_2(Q), \epsilon)} d \epsilon,\]
where $N(\F, L_2(Q), \epsilon)$ is the covering number. By condition 2 and property of covering numbers, 
\[\sup_Q N(\F, L_2(Q), \epsilon) \leq  N(\F, \norm{\cdot}_\F, \epsilon) \leq N(\mc{K}, \norm{\cdot}_1, \epsilon/L).\]
Denote $B_1^k$ as the $l_1$ ball with dimension $k$. We know that for $\epsilon\leq 1$, $N(B_1^k, \norm{\cdot}_1, \epsilon)\leq \left(\frac{3}{\epsilon}\right)^k$. Since $\mc{K}\subset\gamma_{\max}\triangle_\Pi^{(k)}\subset \gamma_{\max}B_1^k$, and there are $\binom{\Pi}{k}$ ways to choose such a support $\gamma_{\max}B_1^k$, by union bound over $k$-dimensional subspaces we have 
\begin{align*}
    N(\mc{K}, \norm{\cdot}_1, \epsilon/L)&\leq \binom{\Pi}{k}N(\gamma_{\max}B_1^k, \norm{\cdot}_1, \epsilon/L)\\
    &\leq\binom{\Pi}{k}N(B_1^k, \norm{\cdot}_1, \epsilon/(L\gamma_{\max}))\\
    &\leq\left(\frac{e |\Pi|}{k}\right)^{k}\left(\frac{3L\gamma_{\max}}{\epsilon}\right)^{k} \leq \left(\frac{3L\gamma_{\max}e|\Pi|}{\epsilon k}\right)^k.
\end{align*}
Therefore, 
\begin{align*}
    \sup_Q \int_{0}^{\infty} \sqrt{\log N(\F, L_2(Q), \epsilon)} d \epsilon &\leq \int_{0}^{\infty} \sqrt{\log N(\mc{K}, \norm{\cdot}_1, \epsilon/L)} d \epsilon\\
    &\leq \int_{0}^{L\gamma_{\max}} \sqrt{k\log \left(\frac{3L\gamma_{\max} e|\Pi|}{\epsilon k}\right)} d \epsilon\\
    &= L\gamma_{\max} \int_{0}^{1} \sqrt{k\log \left(\frac{3e|\Pi|}{\epsilon k}\right)} d \epsilon\\
    &\leq L\gamma_{\max} \sqrt{\int_{0}^{1} k\log \left(\frac{3e|\Pi|}{\epsilon k}\right) d \epsilon}\\
    &\leq L\gamma_{\max}\sqrt{2k\log(3e|\Pi|/k)}.
\end{align*}
Combining all results yields
\begin{align*}
    \E\left[\sup_{\lambda\in\triangle_\Pi^{(k)}}\left|\E_{c\sim\nu_\D}\bigbrak{u(c,\kappa)}-\E_{c\sim\nu}\bigbrak{u(c,\kappa)}\right|\right] &\leq \frac{16}{\sqrt{m}}\sup_Q \int_{0}^{\infty} \sqrt{\log N(\F, L_2(Q), \epsilon)} d \epsilon\\
    &\leq \frac{16}{\sqrt{m}}L\gamma_{\max}\sqrt{2k\log(3e|\Pi|/k)}.
\end{align*}
Therefore, our result follows. 
\end{proof}

Then, we take two special kind of $u(c,\kappa)$, and get the bounds for our estimate of the expectation over $\nu$ with the offline history $\D$. 

\begin{lemma}\label{lem:history_1}
Let $m=|\D|$. Then, with probability at least $1-\delta$, we have 
{
\small
\begin{align*}
    &\sup_{(\lambda,\gamma)\in\gamma_{\max}\triangle_\Pi^{(k)}}\left|\E_{c\sim\nu_\D}\bigbrak{\left(\sum_{a\in\A}\sqrt{(\lambda\odot\gamma)^\T (t_{a}^{(c)}+\eta_l)}\right)^2}-\E_{c\sim\nu}\bigbrak{\left(\sum_{a\in\A}\sqrt{(\lambda\odot\gamma)^\T (t_{a}^{(c)}+\eta_l)}\right)^2}\right|\\
    &\leq \sqrt{\frac{|\A|^4\gamma_{\max}^2(1+\eta_l)^2}{2m}\log\left(\frac{2}{\delta}\right)}+\frac{16}{\sqrt{m}}|\A|^2\gamma_{\max}\sqrt{\frac{2k(1+\eta_l)\gamma_{\max}}{\eta_l\gamma_{\min}}\log\left(\frac{3e|\Pi|}{k}\right)}.
\end{align*}
}
\end{lemma}

\begin{proof}
Define $\kappa\in\mc{K}$ such that $\kappa_\pi=\lambda_\pi\gamma_\pi$. Then, $\mc{K}\subset \gamma_{\max}\triangle_\Pi$ since $\sum_{\pi\in\Pi}\kappa_\pi=\sum_{\pi\in\Pi}\lambda_\pi\gamma_\pi\leq \gamma_{\max}$. Then, let $u(c,\kappa)=\left(\sum_{a\in\A} \sqrt{\kappa^\T (t_{a}^{(c)}+\eta_l)}\right)^2$. We aim to use the result of Lemma~\ref{lem:emp_process} to get our bound. First, since for any $\kappa\in\mc{K}$ and any $c\in\D$, $u(c,\kappa)\in [|\A|^2\gamma_{\min}\eta_l, |\A|^2(1+\eta_l)\gamma_{\max}]$, so condition 1 is satisfied. Also, note that $u(c,\kappa)$ is Lipschitz in $\kappa$, i.e. 
\begin{align*}
    &\norm{u(\cdot,\kappa_1)-u(\cdot,\kappa_2)}_\F\\
    &= \sup_{c\in\C}\left|u(c,\kappa_1)-u(c,\kappa_2)\right|\\
    &= \sup_{c\in\C}\left|\left(\sum_{a\in\A} \sqrt{\kappa_1^\T (t_{a}^{(c)}+\eta_l)}\right)^2-\left(\sum_{a\in\A} \sqrt{\kappa_2^\T (t_{a}^{(c)}+\eta_l)}\right)^2\right|\\
    &\leq \sup_{c\in\C} \left|\left(\sum_{a\in\A} \sqrt{\kappa_1^\T (t_a^{(c)}+\eta_l)}+\sqrt{\kappa_2^\T (t_a^{(c)}+\eta_l)}\right)\left(\sum_{a\in\A} \sqrt{\kappa_1^\T (t_a^{(c)}+\eta_l)}-\sqrt{\kappa_2^\T (t_a^{(c)}+\eta_l)}\right)\right|\\
    &= \sup_{c\in\C} \left(\sum_{a\in\A} \sqrt{\kappa_1^\T (t_a^{(c)}+\eta_l)}+\sqrt{\kappa_2^\T (t_a^{(c)}+\eta_l)}\right)\left(\sum_{a\in\A}\frac{\left|(\kappa_1-\kappa_2)^\T t_a^{(c)}\right|}{\sqrt{\kappa_1^\T (t_a^{(c)}+\eta_l)}+\sqrt{\kappa_2^\T (t_a^{(c)}+\eta_l)}}\right)\\
    &\leq \sup_{c\in\C} \left(\sum_{a\in\A} \sqrt{\kappa_1^\T (t_a^{(c)}+\eta_l)}+\sqrt{\kappa_2^\T (t_a^{(c)}+\eta_l)}\right)\left(\sum_{a\in\A}\frac{\norm{\kappa_1-\kappa_2}_1}{\sqrt{\kappa_1^\T (t_a^{(c)}+\eta_l)}+\sqrt{\kappa_2^\T (t_a^{(c)}+\eta_l)}}\right)\\
    &\leq |\A|^2\sqrt{\frac{(1+\eta_l)\gamma_{\max}}{\eta_l\gamma_{\min}}}\norm{\kappa_1-\kappa_2}_1.
\end{align*}
Therefore, condition 2 is satisfied with $L=|\A|^2\sqrt{\frac{(1+\eta_l)\gamma_{\max}}{\eta_l\gamma_{\min}}}$. Plugging in the result in Lemma~\ref{lem:emp_process}, we get  
\begin{align*}
    &\sup_{\lambda\in\triangle_\Pi^{(k)}}\left|\E_{c\sim\nu_\D}\bigbrak{u(c,\kappa)}-\E_{c\sim\nu}\bigbrak{u(c,\kappa)}\right|\\
    &\leq \sqrt{\frac{|\A|^4\gamma_{\max}^2(1+\eta_l)^2}{2m}\log\left(\frac{2}{\delta}\right)}+\frac{16}{\sqrt{m}}|\A|^2\gamma_{\max}\sqrt{\frac{2k(1+\eta_l)\gamma_{\max}}{\eta_l\gamma_{\min}}\log\left(\frac{3e|\Pi|}{k}\right)}. 
\end{align*}
\end{proof}

\begin{lemma}\label{lem:history_2}
For any $\pi\in\Pi$, with probability at least $1-\delta$, 
\begin{align*}
    &\sup_{(\lambda,\gamma)\in\gamma_{\max}\triangle_\Pi}\left|\E_{c\sim\nu_\D}\left[\sum_{a\in\A}\frac{\sum_{a'\in\A}\sqrt{(\lambda\odot\gamma)^\T (t_{a'}^{(c)}+\eta_l)}}{\sqrt{(\lambda\odot\gamma)^\T (t_{a}^{(c)}+\eta_l)}}(\gamma_\pi[t_a^{(c)}]_\pi)\right]\right.\\
    &\left.-\E_{c\sim\nu}\left[\sum_{a\in\A}\frac{\sum_{a'\in\A}\sqrt{(\lambda\odot\gamma)^\T (t_{a'}^{(c)}+\eta_l)}}{\sqrt{(\lambda\odot\gamma)^\T (t_{a}^{(c)}+\eta_l)}}(\gamma_\pi[t_a^{(c)}]_\pi)\right]\right|\\
    &\leq \gamma_{\max}\left(\sqrt{\frac{|\A|^4(1+\eta)\gamma_{\max}}{2\eta\gamma_{\min}m}\log\bigsmile{\frac{2}{\delta}}}+\frac{8|\A|^2\gamma_{\max}}{\sqrt{m}(\eta_l\gamma_{\min})^{3/2}}\sqrt{2k\log(3e|\Pi|/k)}\right).
\end{align*}
\begin{proof}
First, note that 
\[\frac{\sum_{a'\in\A}\sqrt{(\lambda\odot\gamma)^\T (t_{a'}^{(c)}+\eta_l)}}{\sqrt{(\lambda\odot\gamma)^\T (t_{a}^{(c)}+\eta_l)}}(\gamma_\pi[t_a^{(c)}]_\pi)\leq \gamma_{\max}\frac{\sum_{a'\in\A}\sqrt{(\lambda\odot\gamma)^\T (t_{a'}^{(c)}+\eta_l)}}{\sqrt{(\lambda\odot\gamma)^\T (t_{a}^{(c)}+\eta_l)}}[t_a^{(c)}]_\pi.\]
Then, we define $u(c,\kappa)=\sum_{a\in\A}\frac{\sum_{a'\in\A}\sqrt{\kappa^\T (t_{a'}^{(c)}+\eta_l)}}{\sqrt{\kappa^\T (t_{a}^{(c)}+\eta_l)}}[t_a^{(c)}]_\pi$. First, note that for any $c\in\C$ and $\kappa\in\mc{K}$, $u(c,\kappa)\in\Big[0,|\A|^2\frac{\sqrt{(1+\eta)\gamma_{\max}}}{\sqrt{\eta\gamma_{\min}}}\Big]$, so condition 1 in Lemma~\ref{lem:emp_process} is satisfied. Also, 
\begin{align}
    &\norm{u(c,\kappa_1)-u(c,\kappa_2)}_\F= \sup_{c\in\C}|u(c,\kappa_1)-u(c,\kappa_2)|\\
    &=\sup_{c\in\C}\left|\sum_{a\in\A}\frac{\sum_{a'\in\A}\sqrt{\kappa_1^\T(t_{a'}^{(c)}+\eta_l)}}{\sqrt{\kappa_1^\T(t_{a}^{(c)}+\eta_l)}}[t_a^{(c)}]_\pi-\sum_{a\in\A}\frac{\sum_{a'\in\A}\sqrt{\kappa_2^\T(t_{a'}^{(c)}+\eta_l)}}{\sqrt{\kappa_2^\T(t_{a}^{(c)}+\eta_l)}}[t_a^{(c)}]_\pi\right|\nonumber\\
    &= \sup_{c\in\C}\left|\sum_{a \in \A} \left[\frac{\sum_{a'\in\A}\sqrt{\kappa_1^{\top} (t_{a^{\prime}}^{(c)}+\eta_l)}\sqrt{\kappa_2^{\top} (t_{a}^{(c)}+\eta_l)}-\sqrt{\kappa_2^{\top} (t_{a'}^{(c)}+\eta_l)}\sqrt{\kappa_1^{\top} (t_{a}^{(c)}+\eta_l)}}{\sqrt{\kappa_1^{\top} (t_{a}^{(c)}+\eta_l)}\sqrt{\kappa_2^{\top} (t_{a}^{(c)}+\eta_l)}}[t_a^{(c)}]_\pi\right]\right|\nonumber\\
    &\leq \sup_{c\in\C}\sum_{a \in \A} \left[\frac{\sum_{a'\in\A}\left|\sqrt{\kappa_1^{\top} (t_{a^{\prime}}^{(c)}+\eta_l)}\sqrt{\kappa_2^{\top} (t_{a}^{(c)}+\eta_l)}-\sqrt{\kappa_2^{\top} (t_{a'}^{(c)}+\eta_l)}\sqrt{\kappa_1^{\top} (t_{a}^{(c)}+\eta_l)}\right|}{\sqrt{\kappa_1^{\top} (t_{a}^{(c)}+\eta_l)}\sqrt{\kappa_2^{\top} (t_{a}^{(c)}+\eta_l)}}\right].\label{eqn:foo}
\end{align}
Note that by triangular inequality
\begin{align*}
    &\left|\sqrt{\kappa_2^{\top} (t_{a}^{(c)}+\eta_l)}\sqrt{\kappa_1^{\top} (t_{a^{\prime}}^{(c)}+\eta_l)}-\sqrt{\kappa_1^{\top} (t_{a}^{(c)}+\eta_l)}\sqrt{\kappa_2^{\top} (t_{a'}^{(c)}+\eta_l)}\right|\\
    &\leq\bigabs{\sqrt{\kappa_2^{\top} (t_{a}^{(c)}+\eta_l)}-\sqrt{\kappa_1^{\top} (t_{a}^{(c)}+\eta_l)}}\sqrt{\kappa_1^{\top} (t_{a^{\prime}}^{(c)}+\eta_l)}\\
    &\qquad+\sqrt{\kappa_1^{\top} (t_{a}^{(c)}+\eta_l)}\bigabs{\sqrt{\kappa_1^{\top} (t_{a^{\prime}}^{(c)}+\eta_l)}-\sqrt{\kappa_2^{\top} (t_{a^{\prime}}^{(c)}+\eta_l)}}.
\end{align*}
Also note that 
\begin{align*}
    \left|\sqrt{\kappa_2^{\top} (t_{a}^{(c)}+\eta_l)}-\sqrt{\kappa_1^{\top} (t_{a}^{(c)}+\eta_l)}\right| &= \frac{\left|\sum_{\pi\in\Pi}([\kappa_1]_\pi-[\kappa_2]_\pi)(t_a^{(c)}+\eta_l)_\pi\right|}{\sqrt{\kappa_2^{\top} (t_{a}^{(c)}+\eta_l)}+\sqrt{\kappa_1^{\top} (t_{a}^{(c)}+\eta_l)}}\\
    &\leq \frac{1}{2\sqrt{\eta_l\gamma_{\min}}}\norm{\kappa_2-\kappa_1}_1.
\end{align*}
Therefore, \eqref{eqn:foo} is bounded by $|\A|^2\frac{1}{\eta_l\gamma_{\min}}\frac{1}{2\sqrt{\eta_l\gamma_{\min}}}\norm{\kappa_2-\kappa_1}_1$, so condition 2 is satisfied with $L=\frac{|\A|^2}{2(\eta_l\gamma_{\min})^{3/2}}$. Then, by Lemma~\ref{lem:emp_process}, with probability at least $1-\delta$, 
\begin{align*}
    &\sup_{(\lambda,\gamma)\in\gamma_{\max}\triangle_\Pi}\left|\E_{c\sim\nu_\D}\left[\sum_{a}\frac{\sum_{a'\in\A}\sqrt{(\lambda\odot\gamma)^\T (t_{a'}^{(c)}+\eta_l)}}{\sqrt{(\lambda\odot\gamma)^\T (t_{a}^{(c)}+\eta_l)}}(\gamma_\pi[t_a^{(c)}]_\pi)\right]\right.\\
    &\left.-\E_{c\sim\nu}\left[\sum_{a}\frac{\sum_{a'\in\A}\sqrt{(\lambda\odot\gamma)^\T (t_{a'}^{(c)}+\eta_l)}}{\sqrt{(\lambda\odot\gamma)^\T (t_{a}^{(c)}+\eta_l)}}(\gamma_\pi[t_a^{(c)}]_\pi)\right]\right|\\
    &\leq \gamma_{\max}\left(\sqrt{\frac{|\A|^4(1+\eta)\gamma_{\max}}{2\eta\gamma_{\min}m}\log\bigsmile{\frac{2}{\delta}}}+\frac{8|\A|^2\gamma_{\max}}{\sqrt{m}(\eta_l\gamma_{\min})^{3/2}}\sqrt{2k\log(3e|\Pi|/k)}\right).
\end{align*}
\end{proof}
\end{lemma}

\section{Proof of Theorem~\ref{thm:sample_complexity}}\label{sec:complexity}

We first write down Algorithm~\ref{alg:evaluation_oracle} in full detail in Algorithm~\ref{alg:full_algorithm}. 
\begin{algorithm}[!htb]
\footnotesize
\caption{Full CODA Algorithm}
\begin{algorithmic}[1]\label{alg:full_algorithm}
\REQUIRE policies $\Pi = \{ \pi : \mc{C} \rightarrow \mc{A} \}_\pi$, feature map $\phi : \mc{C} \times \mc{A} \rightarrow \R^d$, $\delta\in (0,1)$, historical data $\mc{D} = \{ \nu_s \}_s$
\STATE initiate $\hat{\pi}_0\in\Pi$ arbitrarily, $\lambda_0=\e_{\hat{\pi}_0}$, $\hat{\Delta}_0(\pi)$, $\gamma_0$ appropriately
\FOR{$l=1,2,\cdots$}
\STATE $\epsilon_l=2^{-l}$, $\eta_l=C_1\epsilon_l^2|\A|^{-4}$, $\delta_l=\delta/(l^2|\Pi|^2)$, $K_l$ appropriately
\STATE $t_{a}^{(c)}(\pi')=\{\1\{\pi(c)=a,\pi'(c)\ne a\}+\1\{\pi(c)\ne a,\pi'(c)=a\}\}_{\pi\in \Pi}\in \mathbb{R}^{\Pi}$
\STATE Define $\gamma_{\min}:=\frac{1}{3}\sqrt{\frac{\eta_l\log(1/\delta_l)}{n}}$, $\gamma_{\max}:=\sqrt{\frac{\log(1/\delta_l)}{|\A|^2\eta_l n}}$
\STATE Define
\begin{align}
    h_l(\lambda,\gamma,n)&=\sum_{\pi\in\Pi}\lambda_\pi\left(-\hat{\Delta}_{l-1}^{\gamma^{l-1}}(\pi,\pill)+\frac{\log(1/\delta_l)}{\gamma_\pi n}\right)\nonumber\\
    &\qquad+\E_{c\sim\nu_{\mathcal{D}}}\bigbrak{\bigsmile{\sum_{a\in\A}\sqrt{(\lambda\odot \gamma)^\T (t_a^{(c)}(\pill)+\eta_l)}}^2}.\label{eqn:C_1}
\end{align}
\STATE Let $\lambda^{l}, \gamma^l, n_l=\textsf{FW-GD}(\Pi,|\A|,\pill,\eta_l, K_l,\epsilon_l,\gamma_{\min},\gamma_{\max})$. These are the solutions to 
\begin{align}
    & n_{\ell} := \min\{n\in \mathbb{N}: \max_{\lambda\in\triangle_\Pi}\min_{\gamma\in[\gamma_{\min},\gamma_{\max}]^{|\Pi|}}h_l(\lambda,\gamma,n)\leq \epsilon_{\ell}\}
\end{align}
\STATE Receive contexts $c_1,c_2,\cdots,c_{n_l}\sim\nu$. 
\STATE For each $c_s$, $s=1,2,\cdots,n_l$, pull arms $a_s\sim p^{(\ell)}_{c_s}$ where $p^{(\ell)}_{c_s,a_s}\propto \sqrt{(\lambda^{l}\odot\gamma^l)^\T (t_{a_s}^{(c_s)}(\pill)+\eta_l)}$, and observe rewards $r_s$ where $t_{a_s}^{(c_s)}(\pill)\in\R^{|\Pi|}$
\STATE For each $\pi\in\Pi$, define the IPS estimator $$\hat{\Delta}_l^{\gamma^l}(\pi,\pill)=\sum_{s=1}^{n_l}\frac{r_s}{p^{(\ell)}_{c_s,a_s}+[\gamma^l]_\pi}(\1\{\pill(c_s)=a_s\}-\1\{\pi(c_s)=a_s\})$$
\STATE set 
\begin{align}
    \hat{\pi}_l
    &=\arg\min_{\pi\in\Pi}\hat{\Delta}_l^{\gamma^l}(\pi,\hat{\pi}_{l-1})+\E_{c\sim\nu_{\mathcal{D}}} \left[ \left(\frac{[\gamma^l]_\pi}{p_{c,\pi(c)}^{(\ell)}}+\frac{[\gamma^l]_\pi}{p_{c,\pill(c)}^{(\ell)}}\right) \1\{\pill(c )\ne \pi(c )\} \right]+\frac{\log(1/\delta_l)}{[\gamma^l]_\pi n_l}.
\end{align}
\ENDFOR
\ENSURE $\hat{\pi}_l$
\end{algorithmic}
\end{algorithm}
We aim to show that Algorithm \ref{alg:full_algorithm} achieves the sample complexity lower bound. The two big goals here is to show that $\hat{\pi}_l\in S_l$ for all $l$, which shows that we get the optimal policy, and $n_l$ achieves the sample complexity lower bound. 
\begin{theorem}
With probability at least $1-\delta$, Algorithm~\ref{alg:full_algorithm} returns a policy $\hat{\pi}$ satisfying $V(\pi_{\ast}) - V(\hat{\pi}_{\ell})\leq \epsilon$ in a number of samples not exceeding $O(\rho_{\ast, \epsilon}\log(|\Pi|\log_2(1/\Delta_\epsilon)/\delta)\log_2(1/\Delta_{\epsilon})$ where $\Delta_{\epsilon} := \max\{\epsilon, \min_{\pi\in \Pi} V(\pi_{\ast}) - V(\pi)\}$.
\end{theorem}
\begin{proof}
We first define our key events. Recall
$$\hat{\Delta}_l^{\gamma^l}(\pi,\pill)=\sum_{s=1}^{n_l}\frac{r_s}{p^{(\ell)}_{c_s,a_s}+[\gamma^l]_\pi}(\1\{\pill(c_s)=a_s\}-\1\{\pi(c_s)=a_s\})$$ and $\Delta(\pi,\pi')=V(\pi')-V(\pi)$. 
Define $w(\lambda,\gamma)\in\R^{|\A|\times|\C|}$ with
\begin{equation*}
    [w(\lambda,\gamma)]_{a, c}:=\nu_{c} \cdot p_{c,a}=\nu_{c} \cdot \frac{\sqrt{(\lambda\odot\gamma)^{\top} (t_{a}^{(c)}(\pill)+\eta_l)}}{\sum_{a'\in \A} \sqrt{(\lambda\odot\gamma)^{\top} (t_{a^{\prime}}^{(c)}(\pill)+\eta_l)}}.
\end{equation*}

Then define the events \[\EE_{l}:=\bigcap_{\pi,\pi'\in\Pi}\left\{\left|\hat{\Delta}_{l}^{\gamma^l}(\pi,\pi')-\Delta(\pi,\pi')\right|\leq 2[\gamma^l]_\pi\norm{\phi_{\pi}-\phi_{\pi'}}_{A(w(\lambda^{l},\gamma^l))^{-1}}^2+\frac{2\log(1/\delta_l)}{[\gamma^l]_\pi n_l}\right\},\]
and the good event $\EE=\bigcap_{l=1}^\infty\EE_l$. Lemma~\ref{lem:good_events} shows that $\EE$ happens with probability at least $1-\delta$, and Lemma~\ref{lem:n_l} shows that under this event $\EE$, $$n_l\lesssim \min_{w\in\Omega}\max_{\pi\in\Pi}\frac{\norm{\phi_{\pi_*}-\phi_\pi}_{A(w)^{-1}}^2\log(1/\delta_l)}{\epsilon_l^2+\Delta(\pi,\pi_*)^2}.$$
Therefore, the total number of samples is no more than 
\begin{align*}
    &\sum_{l=1}^{\log_2(1/\Delta_\epsilon)}\min_{w\in\Omega}\max_{\pi\in\Pi}\frac{\norm{\phi_{\pi_*}-\phi_\pi}_{A(w)^{-1}}^2\log(l^2|\Pi|^2/\delta)}{\epsilon_l^2+\Delta(\pi,\pi_*)^2}\\
    &\overset{(i)}{\leq} \sum_{l=1}^{\log_2(1/\Delta_\epsilon)}\min_{w\in\Omega}\max_{\pi\in\Pi\setminus\pi_*}\frac{2\norm{\phi_{\pi_*}-\phi_\pi}_{A(w)^{-1}}^2\log(l^2|\Pi|^2/\delta)}{\epsilon_l^2+\Delta(\pi,\pi_*)^2}\\
    &\overset{(ii)}{\leq} \sum_{l=1}^{\log_2(1/\Delta_\epsilon)}\min_{p^{(c)}\in\triangle_\A,\forall c\in\C}\max_{\pi\in\Pi\setminus\pi_*}\frac{\E_{c\sim\nu}\left[\bigsmile{\frac{1}{p_{\pi_*(c)}^{(c)}}+\frac{1}{p_{\pi(c)}^{(c)}}}\1\{\pi_*(c)\ne\pi(c)\}\right]\log(l^2|\Pi|^2/\delta)}{\Delta(\pi,\pi_*)^2+\epsilon_l^2}\\
    &\lesssim \rho_{\star,\epsilon}(\Pi, v)\log(\log_2(1/\Delta_\epsilon)|\Pi|/\delta)\log_2(1/\Delta_\epsilon).
\end{align*}
where $(i)$ follows from the fact that $\pi_*$ gives zero for the RHS, and $(ii)$ follows from Lemma~\ref{lem:B_3}.
\end{proof}


In what follows, we will fill in the road map to the proof of Lemma~\ref{lem:good_events} and \ref{lem:n_l}. First, Lemma~\ref{lem:EE} controls the estimation error of the gap and shows that $\P(\mathcal{E}_{\ell}) > 1-\delta_{\ell}$, which leads to the high-probability of the good event $\EE$ (Lemma~\ref{lem:good_events}). Lemma~\ref{lem:control_of_UCB} applies the duality machinery in Section~\ref{sec:conv_analysis} and controls the variance term. Lemma~\ref{lem:emp_gaps} applies the result of Lemma~\ref{lem:control_of_UCB} and shows an upper bound for the difference between estimate gap and the true gap, which is a very similar result of Lemma~\ref{lem:diff_emp_true_gap}. Lemma~\ref{lem:est_design_2} is an important lemma showing the analytical solution of $w$ given some $\lambda$ and $\gamma$. With all of these results above, we get Lemma~\ref{lem:n_l} which gives the upper bound on the sample complexity.
 
\begin{lemma}\label{lem:EE}
For any $l>0$, $\pi,\pi'\in\Pi$, with probability at least $1-\delta_l$, 
\[\bigabs{\hat{\Delta}_{l}^{\gamma^l}(\pi,\pi')-\Delta(\pi,\pi')} \leq 2[\gamma^l]_\pi \norm{\phi_\pi-\phi_{\pi'}}_{A(w(\lambda^{l},\gamma^l))^{-1}}^2+\frac{2\log (1 / \delta_l)}{[\gamma^l]_\pi n_l}.\]
\end{lemma}
\begin{proof}
Define \[\hat{V}_l^{\gamma^l}(\pi) := \sum_{s=1}^{n_l}\frac{r_s}{p^{(\ell)}_{c_s,a_s}+[\gamma^l]_\pi}\1\{\pi(c_s)=a_s\},\] so that \[\hat{\Delta}_l^{\gamma^l}(\pi,\pi')=\hat{V}_l^{\gamma^l}(\pi')-\hat{V}_l^{\gamma^l}(\pi).\]

First, note that below.
\begin{align*}
    V(\pi) &= \E_{c\sim \nu}\bigbrak{r(c,\pi(c))}\\
    &= \E_{c\sim\nu}\bigbrak{\E_{a\sim p^{(\ell)}_c}\bigbrak{r(c,a)\frac{\1\{\pi(c)=a\}}{p^{(\ell)}_{c,a}}\bigg|c}}= \E\bigbrak{\frac{1}{t}\sum_{s=1}^t\frac{r_s}{p^{(\ell)}_{c_s,a_s}}\1\{\pi(c_s)=a_s\}}.
\end{align*}
Therefore, 
\begin{align*}
&\left|\mathbb{E}\left[\widehat{V}_l^{\gamma^l}(\pi)-\widehat{V}^{\gamma^l}_l\left(\pi^{\prime}\right)\right]-\left[V(\pi)-V(\pi')\right]\right|\\ &\leq\left|\mathbb{E}\left[\frac{1}{n_l} \sum_{s=1}^{n_l}\left(\frac{1}{p^{(\ell)}_{c_s,a_s}+[\gamma^l]_\pi}-\frac{1}{p^{(\ell)}_{c_s,a_s}}\right)\left(\1\{\pi(c_s)=a_s\}-\1\{\pi'(c_s)=a_s\}\right)\right]\right| \\
&=\left|\mathbb{E}\left[\frac{1}{n_l} \sum_{s=1}^{n_l} \frac{-[\gamma^l]_\pi}{p^{(\ell)}_{c_s,a_s}\left(p^{(\ell)}_{c_s,a_s}+[\gamma^l]_\pi\right)}\left(\1\{\pi(c_s)=a_s\}-\1\{\pi'(c_s)=a_s\}\right)\right]\right| \\
& \leq \mathbb{E}\left[\frac{1}{n_l} \sum_{s=1}^{n_l} \frac{[\gamma^l]_\pi\bigsmile{\1\{\pi'(c_s)=a_s,\pi(c_s)\ne a_s\}+\1\{\pi'(c_s)\ne a_s,\pi(c_s)=a_s\}}}{p^{(\ell)}_{c_s,a_s}\left(p^{(\ell)}_{c_s,a_s}+[\gamma^l]_\pi\right)}\right] \\
&=[\gamma^l]_\pi \mathbb{E}\left[\frac{1}{p^{(\ell)}_{c,a}\left(p^{(\ell)}_{c,a}+[\gamma^l]_\pi\right)\nu_c^2}[\phi_\pi-\phi_{\pi'}]_{a,c}^2\right] \\
&=[\gamma^l]_\pi \sum_{c\in\C}\nu_c\sum_{a\in\A}p^{(\ell)}_{c,a} \frac{1}{p^{(\ell)}_{c,a}\nu_c^2\left(p^{(\ell)}_{c,a}+[\gamma^l]_\pi\right)}[\phi_\pi-\phi_{\pi'}]_{a,c}^2 \\
& \leq [\gamma^l]_\pi \norm{\phi_\pi-\phi_{\pi'}}_{A(w(\lambda^{l},\gamma^l))^{-1}}^2
\end{align*}
where the last inequality follows since $\nu_cp^{(\ell)}_{c,a}=[w(\lambda^{l},\gamma^l)]_{a,c}$. 
Meanwhile, note that $$\frac{r_s}{p^{(\ell)}_{c_s,a_s}+[\gamma^l]_\pi}\left(\1\{\pi(c_s)=a_s\}-\1\{\pi'(c_s)=a_s\}\right)\leq \frac{1}{[\gamma^l]_\pi},$$ and 
\begin{align*}
&\mathbb{E}\left[\bigsmile{\frac{r_s}{p^{(\ell)}_{c_s,a_s}+[\gamma^l]_\pi}\left(\1\{\pi(c_s)=a_s\}-\1\{\pi'(c_s)=a_s\}\right)}^{2}\right] \\
&\leq \mathbb{E}\left[\frac{1}{(p^{(\ell)}_{c_s,a_s}+[\gamma^l]_\pi)^2}\left(\mathbf{1}\left\{\pi\left(c_{s}\right)= a_{s}\right\}-\mathbf{1}\left\{\pi^{\prime}\left(c_s\right)= a_s\right\}\right)^{2}\right]\\
&= \mathbb{E}\left[\frac{1}{(p^{(\ell)}_{c_s,a_s}+[\gamma^l]_\pi)^2\nu_c^2} [\phi_\pi-\phi_{\pi'}]_{a,c}^2\right] \\
& \leq \norm{\phi_\pi-\phi_{\pi'}}_{A(w(\lambda^{l},\gamma^l))^{-1}}^2
\end{align*}
by a similar argument as before. Therefore, by Bernstein's inequality, we have with probability at least $1-\delta$, 
\begin{align*}
\left|\widehat{V}_l^{\gamma^l}(\pi)-\widehat{V}_l^{\gamma^l}\left(\pi^{\prime}\right)-\mathbb{E}\left[\widehat{V}_l^{\gamma^l}(\pi)-\widehat{V}_l^{\gamma^l}\left(\pi^{\prime}\right)\right]\right| &\leq \sqrt{\norm{\phi_\pi-\phi_{\pi'}}_{A(w(\lambda^{l},\gamma^l))^{-1}}^2 \frac{2 \log (1 / \delta)}{n_l}}+\frac{\log (1 / \delta)}{[\gamma^l]_\pi n_l}.
\end{align*}
Combining this with the deviation on expectation gives us 
\begin{align*}
    &\bigabs{\hat{\Delta}_{l}^{\gamma^l}(\pi,\pi')-\Delta(\pi,\pi')}\\
    &\leq[\gamma^l]_\pi \norm{\phi_\pi-\phi_{\pi'}}_{A(w(\lambda^{l},\gamma^l))^{-1}}^2+\sqrt{\norm{\phi_\pi-\phi_{\pi'}}_{A(w(\lambda^{l},\gamma^l))^{-1}}^2 \frac{2 \log (1 / \delta)}{n_l}}+\frac{2\log (1 / \delta)}{[\gamma^l]_\pi n_l}\\
    &\leq2[\gamma^l]_\pi \norm{\phi_\pi-\phi_{\pi'}}_{A(w(\lambda^{l},\gamma^l))^{-1}}^2+\frac{4\log (1 / \delta)}{[\gamma^l]_\pi n_l}.
\end{align*}
\end{proof}

\begin{lemma}\label{lem:good_events}
$\P(\EE)\geq 1-\delta$. 
\end{lemma}

\begin{proof}
By Lemma~\ref{lem:EE} and a union bound over all policies, we have  
\[\mathbb{P}\left(\mathcal{E}_{l} \mid \mathcal{E}_{l-1}, \cdots, \mathcal{E}_{1}\right) \geq 1-\frac{\delta}{l^2}.\]
Since $\EE=\bigcap_{l=0}^\infty\EE_l$, 
\begin{align*}
    \P(\EE^c)&=\P((\cap_{l=0}^\infty\EE_l)^c)=\mathbb{P}\left(\cup_{l=0}^{\infty} \EE_{l}^c\right)=\mathbb{P}\left(\cup_{l=0}^{\infty}\left(\EE_l^c \backslash\left(\cup_{j < l} \EE_{j}^c\right)\right)\right)\\
    &\leq \sum_{l=0}^{\infty} \mathbb{P}\left(\EE_l^c \backslash\left(\cup_{j<l} \EE_{j}^c\right)\right) \leq \sum_{l=0}^{\infty} \mathbb{P}\left(\EE_{l}^c \mid\left(\cap_{j<l} \EE_{j}\right)\right)\leq\sum_{l=0}^{\infty}\frac{\delta}{l^2} \leq \delta. 
\end{align*}
Therefore, $\P(\EE)\geq 1-\delta$. 
\end{proof}

\begin{lemma}\label{lem:control_of_UCB}
Under $\EE$, we have for any $\pi\in\Pi$,
\[[\gamma^l]_\pi\norm{\phi_\pi-\phi_{\pill}}_{A(w(\lambda^{l},\gamma^l))^{-1}}^2+\frac{\log(1/\delta_l)}{[\gamma^l]_\pi n_l}\leq \frac{1}{6}\epsilon_{l}+\frac{1}{64}\hat{\Delta}_{l-1}^{\gamma^{l-1}}(\pi,\pill).\]
\end{lemma}
\begin{proof}
We know that the choice of $n_l$ ensures \[h_l(\lambda^{l},\gamma^l,n_l)\leq \epsilon_l.\] Also, by Theorem~\ref{thm:FW_gap} we have
{
\small
\begin{align*}
    \frac{1}{3}\epsilon_l &\geq \max_{\pi\in\Pi} \left(-\frac{1}{8}\hat{\Delta}_{l-1}^{\gamma^{l-1}}(\pi,\pill)+8[\gamma^l]_\pi\norm{\phi_\pi-\phi_{\pill}}_{A(w(\lambda^{l},\gamma^l))^{-1}}^2+\frac{8\log(1/\delta_l)}{[\gamma^l]_\pi n_l}\right)-h_l(\lambda^{l},\gamma^l,n_l). 
\end{align*}
}
Combining the above two displays gives us  
\begin{align*}
    \epsilon_l&\geq h_l(\lambda^{l},\gamma^l,n_l)\\
    &\geq \max_{\pi\in\Pi}\left(-\frac{1}{8}\hat{\Delta}_{l-1}^{\gamma^{l-1}}(\pi,\pill)+8[\gamma^l]_\pi \norm{\phi_{\pill}-\phi_\pi}_{A(w(\lambda^{l},\gamma^l))^{-1}}^2+\frac{8\log(1/\delta_l)}{[\gamma^l]_\pi n_l}\right)-\frac{1}{3}\epsilon_l.
\end{align*}
Therefore, for any $\pi\in \Pi$,
\[[\gamma^l]_\pi\norm{\phi_\pi-\phi_{\pill}}_{A(w(\lambda^{l},\gamma^l))^{-1}}^2+\frac{\log(1/\delta_l)}{[\gamma^l]_\pi n_l}\leq \frac{1}{6}\epsilon_{l}+\frac{1}{64}\hat{\Delta}_{l-1}^{\gamma^{l-1}}(\pi,\pill).\]
\end{proof}

\begin{lemma}\label{lem:emp_gaps}
Under $\EE$, for all $l\in\N$, the following holds: 
\begin{enumerate}
    \item $|\widehat{\Delta}_{l-1}^{\gamma^{l-1}}\left(\pi, \pill\right)-\Delta\left(\pi, \pi_{*}\right)| \leq 2\epsilon_{l-1}+ \frac{1}{4}\Delta(\pi,\pi_*).$
    \item $\hat{\pi}_l\in S_l:=\{\pi\in\Pi: \Delta(\pi,\pi_*)\leq \epsilon_l\}$. 
\end{enumerate}
\end{lemma}

\begin{proof}
We prove this by induction. First, in round $l=0$, this holds since our rewards are bounded by 1. Then, assume that in round $l-1$, we have $\pill\in S_{l-1}$ and 
$$
|\widehat{\Delta}_{l-2}^{\gamma_{l-2}}\left(\pi, \hat{\pi}_{l-2}\right)-\Delta\left(\pi, \pi_{*}\right)| \leq 2 \epsilon_{l-2}+ \frac{1}{4}\Delta(\pi,\pi_*).
$$
Then, on round $l$, 
\begin{align*}
    &|\widehat{\Delta}_{l-1}^{\gamma^{l-1}}\left(\pi, \hat{\pi}_{l-1}\right)-\Delta\left(\pi, \pi_*\right)|\\
    &= |\widehat{\Delta}_{l-1}^{\gamma^{l-1}}\left(\pi, \pill\right)-\Delta\left(\pi, \pill\right)-\Delta\left(\pill, \pi_*\right)|\\
    &\leq 2[\gamma^{l-1}]_\pi\norm{\phi_{\pi}-\phi_{\hat{\pi}_{l-1}}}_{A(w(\lambda^{l-1},\gamma^{l-1}))^{-1}}^2+\frac{2\log(1/\delta_{l-1})}{[\gamma^{l-1}]_\pi n_{l-1}}+ \epsilon_{l-1}\tag{from event $\EE$ and inductive hypothesis}\\
    &\leq \frac{2}{3}\epsilon_l+\frac{1}{64}\hat{\Delta}_{l-2}^{\gamma_{l-2}}(\pi,\hat{\pi}_{l-2})+\frac{1}{64}\hat{\Delta}_{l-2}^{\gamma_{l-2}}(\pill,\hat{\pi}_{l-2})+\epsilon_{l-1}\tag{from Lemma~\ref{lem:control_of_UCB}}\\
    &\leq \frac{5}{3}\epsilon_{l-1}+\frac{1}{64}\left(2\epsilon_{l-2}+\frac{5}{4}\Delta(\pi,\pi_*)+2\epsilon_{l-2}+\frac{5}{4}\Delta(\pill,\pi_*)\right)\tag{from inductive hypothesis}\\
    &\leq \frac{5}{3}\epsilon_{l-1}+\frac{1}{64}\left(2\epsilon_{l-2}+\frac{5}{4}\Delta(\pi,\pi_*)+2\epsilon_{l-2}+\frac{5}{4}\epsilon_{l-1}\right)\\
    &\leq 2\epsilon_{l-1}+\frac{1}{4}\Delta(\pi,\pi_*).
\end{align*}

Also, 
\begin{align*}
    \Delta(\hat{\pi}_l, \hat{\pi}_{l-1}) &\leq\hat{\Delta}_{l}^{\gamma^l}(\hat{\pi}_l, \hat{\pi}_{l-1})+[\gamma^l]_{\hat{\pi}_l}\norm{x_{\hat{\pi}_l}-\phi_{\pill}}_{A(w(\lambda^{l},\gamma^l))^{-1}}^2+\frac{\log(1/\delta_l)}{[\gamma^l]_{\pil} n_l}\tag{from $\EE$}\\
    &\leq\hat{\Delta}_{l}^{\gamma^l}(\pi_*, \hat{\pi}_{l-1})+[\gamma^l]_{\pi_*}\norm{\phi_{\pi_*}-\phi_{\pill}}_{A(w(\lambda^{l},\gamma^l))^{-1}}^2+\frac{\log(1/\delta_l)}{[\gamma^l]_{\pi_*} n_l}\tag{eqn~\eqref{eqn:opt_2}, the minimum}\\
    &\leq\Delta(\pi_*, \hat{\pi}_{l-1})+2[\gamma^l]_{\pi_*}\norm{\phi_{\pi_*}-\phi_{\pill}}_{A(w(\lambda^{l},\gamma^l))^{-1}}^2+\frac{2\log(1/\delta_l)}{[\gamma^l]_{\pi_*} n_l}\tag{from $\EE$}\\
    &\leq \Delta\left(\pi_{*}, \hat{\pi}_{l-1}\right)+\frac{1}{3}\epsilon_{l}+\frac{1}{32}\hat{\Delta}_{l-1}^{\gamma^{l-1}}\left(\pi_{*}, \hat{\pi}_{l-1}\right)\tag{from Lemma~\ref{lem:control_of_UCB}}\\
    &\leq \Delta\left(\pi_{*}, \hat{\pi}_{l-1}\right)+\frac{1}{3}\epsilon_{l}+\frac{1}{32}\left(2\epsilon_{l-1}+\frac{5}{4}\Delta\left(\pi_{*}, \pi_*\right)\right)\tag{from the above}.
\end{align*}
Therefore, 
\begin{align*}
    \Delta(\pil,\pi_*) &= \Delta(\hat{\pi}_l, \hat{\pi}_{l-1}) - \Delta\left(\pi_{*}, \hat{\pi}_{l-1}\right) \\
    &\leq \frac{1}{3}\epsilon_{l} + \frac{1}{16} 2\epsilon_{l}\\
    &\leq \epsilon_{\ell}
\end{align*}
Therefore, $\Delta(\pil,\pi_*)\leq\epsilon_l$, so $\pil\in S_l$.

\end{proof}

\begin{lemma}\label{lem:est_design_2}
For any $\lambda\in\triangle_\Pi$, $\gamma\in\R^{|\Pi|}$, and $\pi'\in\Pi$, we have 
\[\min_{w\in\Omega}\sum_{\pi\in \Pi} \lambda_{\pi}\gamma_\pi\|\phi_{\pi}-\phi_{\pi'}\|_{A(w)^{-1}}^2=\E_{c\sim\nu}\bigbrak{\bigsmile{\sum_{a\in\A}\sqrt{(\lambda\odot \gamma)^\T t_a^{(c)}(\pi')}}^2}.\]
where $w_{a,c}=\nu_c p_a^{(c)}$ and $p_a^{(c)}\propto \sqrt{ \sum_{\pi \in \Pi} \lambda_\pi \gamma_\pi(\1\{ \pi'(c) = a, \pi(c) \neq a \} + \1\{ \pi'(c) \neq a, \pi(c) = a \}) }$ and $\odot$ denotes element-wise multiplication. 
\end{lemma}

\begin{proof}
For any $\lambda\in\triangle_\Pi$, 
\begin{align*}
    &\min_{w\in\Omega}\sum_{\pi\in \Pi} \lambda_{\pi}\gamma_\pi\|\phi_{\pi}-\phi_{\pi'}\|_{A(w)^{-1}}^2\\
    &= \min_{w\in\Omega}\sum_{\pi \in \Pi} \sum_{a,c} \frac{\lambda_{\pi}\gamma_\pi}{w_{a,c}} (\phi_{\pi}-\phi_{\pi'})^\top e_{a,c}e_{a,c}^{\top} (\phi_{\pi}-\phi_{\pi'})\\
    &= \min_{p_1,\dots,p_{|\C|} \in \triangle_\A} \sum_{\pi \in \Pi} \sum_{a,c} \frac{\lambda_{\pi}\gamma_\pi}{\nu_c p_{c,a}} (\phi_{\pi}-\phi_{\pi'})^\top e_{a,c}e_{a,c}^{\top} (\phi_{\pi}-\phi_{\pi'})   \\
    &=  \sum_c  \min_{p_c \in \triangle_\A}  \sum_a\sum_{\pi \in \Pi}  \frac{\lambda_{\pi}\gamma_\pi}{\nu_c p_{c,a}} (\phi_{\pi}-\phi_{\pi'})^\top e_{a,c}e_{a,c}^{\top} (\phi_{\pi}-\phi_{\pi'})   \\
    &=  \sum_c \frac{1}{\nu_c} \min_{p_c \in \triangle_\A}  \sum_a  \frac{1}{p_{c,a}}  \left( \sum_{\pi \in \Pi} \lambda_{\pi}\gamma_\pi (\phi_{\pi}-\phi_{\pi'})^\top e_{a,c}e_{a,c}^{\top} (\phi_{\pi}-\phi_{\pi'}) \right) \\
    &=  \sum_c \frac{1}{\nu_c}  \left( \sum_{a\in\A} \sqrt{ \sum_{\pi \in \Pi} \lambda_{\pi}\gamma_\pi (\phi_{\pi}-\phi_{\pi'})^\top e_{a,c}e_{a,c}^{\top} (\phi_{\pi}-\phi_{\pi'}) } \right)^2  \\
    &=  \sum_c \frac{1}{\nu_c}  \left( \sum_{a\in\A} \sqrt{ \sum_{\pi \in \Pi} \lambda_{\pi}\gamma_\pi \nu_c^2 (\1\{ \pi'(c) = a, \pi(c) \neq a \} + \1\{ \pi'(c) \neq a, \pi(c) = a \})} \right)^2  \\
    &= \sum_c \nu_c  \left( \sum_{a\in\A} \sqrt{ \sum_{\pi \in \Pi} \lambda_{\pi}\gamma_\pi (\1\{ \pi'(c) = a, \pi(c) \neq a \} + \1\{ \pi'(c) \neq a, \pi(c) = a \}) } \right)^2 \\
    &= \E_{c\sim\nu}\bigbrak{\bigsmile{\sum_{a\in\A}\sqrt{(\lambda\odot\gamma)^\T t_a^{(c)}(\pi')}}^2}.\\
\end{align*}
Note that the minimizer
\begin{align*}
    p_{c,a} &= \frac{\sqrt{ \sum_{\pi \in \Pi} \lambda_{\pi}\gamma_\pi (\phi_\pi-\phi_{\pi'})^\top e_{a,c}e_{a,c}^{\top} (\phi_\pi-\phi_{\pi'}) } }{\sum_{a'} \sqrt{ \sum_{\pi \in \Pi} \lambda_{\pi}\gamma_\pi (\phi_\pi-\phi_{\pi'})^\top e_{a',c}e_{a',c}^{\top} (\phi_\pi-\phi_{\pi'}) } } \\
    &\propto \sqrt{ \sum_{\pi \in \Pi} \lambda_\pi\gamma_\pi (\1\{ \pi'(c) = a, \pi(c) \neq a \} + \1\{ \pi'(c) \neq a, \pi(c) = a \}) }.
\end{align*}
\end{proof}

\begin{lemma}\label{lem:n_l}
Under $\EE$, the choice for $n_l$ in the algorithm satisfies \[n_l\lesssim \min_{w\in\Omega}\max_{\pi\in\Pi}\frac{\norm{\phi_{\pi_*}-\phi_\pi}_{A(w)^{-1}}^2\log(1/\delta_l)}{\epsilon_l^2+\Delta(\pi)^2}.\]
\end{lemma}
\begin{proof}
{
\begingroup
\allowdisplaybreaks
\small
\begin{align*}
    &h_l(\lambda^{l},\gamma^l,n_l)\\
    &=\sum_{\pi\in\Pi}[\lambda^{l}]_\pi\cdot\left(-\hat{\Delta}_{l-1}^{\gamma^{l-1}}(\pi,\pill)+\frac{\log(1/\delta_l)}{[\gamma^l]_\pi n}\right)+\E_{c\sim\nu_\D}\bigbrak{\bigsmile{\sum_{a\in\A}\sqrt{(\lambda^{l}\odot \gamma^l)^\T (t_a^{(c)}+\eta_l)}}^2}\\
    &\leq \max_{\lambda\in\triangle_\Pi}\min_{\gamma}\sum_{\pi\in\Pi}\lambda_\pi\cdot\left(-\hat{\Delta}_{l-1}^{\gamma^{l-1}}(\pi,\pill)+\frac{\log(1/\delta_l)}{\gamma_\pi n}\right)+\E_{c\sim\nu}\bigbrak{\bigsmile{\sum_{a\in\A}\sqrt{(\lambda\odot \gamma)^\T (t_a^{(c)}+\eta_l)}}^2}+\frac{1}{4}\epsilon_l\tag{by Theorem~\ref{lem:SP}, the saddle point argument}\\
    &\leq \max_{\lambda\in\triangle_\Pi}\min_{\gamma}\sum_{\pi\in\Pi}\lambda_\pi\cdot\left(-\hat{\Delta}_{l-1}^{\gamma^{l-1}}(\pi,\pill)+\frac{\log(1/\delta_l)}{\gamma_\pi n}\right)+\E_{c\sim\nu}\bigbrak{\bigsmile{\sum_{a\in\A}\sqrt{(\lambda\odot \gamma)^\T t_a^{(c)}}}^2}+\frac{1}{2}\epsilon_l\tag{by Lemma~\ref{lem:E_9}, controlling the bias}\\ 
    &= \max_{\lambda\in\triangle_\Pi}\min_{w\in\Omega}\min_{\gamma\in\R_+^{|\Pi|}}\sum_{\pi\in\Pi}\lambda_\pi\cdot\left(-\hat{\Delta}_{l-1}^{\gamma^{l-1}}(\pi,\pill)+\gamma_\pi\norm{\phi_{\pill}-\phi_\pi}_{A(w)^{-1}}^2+\frac{\log(1/\delta_l)}{\gamma_\pi n}\right)+\frac{1}{2}\epsilon_l\tag{by Lemma~\ref{lem:est_design_2}, the definition of $w$}\\
    &= \min_{w\in\Omega}\max_{\pi\in\Pi} \min_{\gamma>0}-\frac{1}{8}\hat{\Delta}_{l-1}^{\gamma^{l-1}}(\pi,\pill)+8\gamma\norm{\phi_{\pill}-\phi_\pi}_{A(w)^{-1}}^2+8\frac{\log(1/\delta_l)}{\gamma n_l}+\frac{1}{2}\epsilon_l\tag{by Lemma~\ref{lem:strong_duality}, the strong duality}\\
    &\leq\min_{w\in\Omega} \max_{\pi\in\Pi}\min_{\gamma}\left(-\frac{3}{32}\Delta(\pi,\pi_*)+8\gamma \norm{\phi_{\pill}-\phi_\pi}_{A(w)^{-1}}^2+8\frac{\log(1/\delta_l)}{\gamma n_l}\right)+\frac{3}{4}\epsilon_l\tag{by Lemma~\ref{lem:emp_gaps}}\\
    &\leq\min_{w\in\Omega} \max_{\pi\in\Pi}\left(-\frac{3}{32}\Delta(\pi,\pi_*)+16\sqrt{\frac{\norm{\phi_{\pill}-\phi_\pi}_{A(w)^{-1}}^2\log(1/\delta_l)}{n_l}}\right)+\frac{3}{4}\epsilon_l\\
    &\leq\min_{w\in\Omega} \max_{\pi\in\Pi}\left(-\frac{3}{32}\Delta(\pi,\pi_*)+16\sqrt{\frac{\norm{\phi_{\pi_*}-\phi_\pi}_{A(w)^{-1}}^2\log(1/\delta_l)}{n_l}}\right.\\
    &\left.\qquad+16\sqrt{\frac{\norm{\phi_{\pi_*}-\phi_{\pill}}_{A(w)^{-1}}^2\log(1/\delta_l)}{n_l}}\right)+\frac{3}{4}\epsilon_l\\
    &\leq\min_{w\in\Omega} \max_{\pi\in\Pi}\left(-\frac{3}{32}\Delta(\pi,\pi_*)+16\sqrt{\frac{\norm{\phi_{\pi_*}-\phi_\pi}_{A(w)^{-1}}^2\log(1/\delta_l)}{n_l}}\right.\\
    &\left.\qquad+16\sqrt{\max_{\pi'\in S_{l-1}}\frac{\norm{\phi_{\pi_*}-\phi_{\pi'}}_{A(w)^{-1}}^2\log(1/\delta_l)}{n_l}}\right)+\frac{3}{4}\epsilon_l.
\end{align*}
\endgroup
}
which is less than $\epsilon_l$ whenever
\begin{equation}
    n_l\gtrsim\min_{w\in\Omega}\max_{\pi\in\Pi}\frac{\norm{\phi_{\pi_*}-\phi_\pi}_{A(w)^{-1}}^2\log(1/\delta_l)}{\epsilon_l^2+\Delta(\pi)^2}.\label{eqn:n_l}
\end{equation}
\end{proof}

\section{Convergence analysis of \textsf{FW-GD}}\label{sec:conv_analysis}

\subsection{Statement of the convergence results}
In this section, we will characterize the performance of Algorithm~\ref{alg:full_algorithm}, a.k.a. Algorithm \ref{alg:evaluation_oracle}. 
Our goal is to show two results: the duality gap converges to zero, and our algorithm converges to the saddle point. It is known that Frank-Wolfe algorithm directly deals with the duality gap \cite{pedregosa2020linearly}, so we will define our primal and dual problem in what follows. Since we are computing $n_l$ via binning, in each inner loop $n$ is fixed. Then, we define our dual objective the same as \eqref{eqn:C_1} with the shorthand notation $h_l(\lambda,\gamma):=h_l(\lambda,\gamma,n)$. We formulate our primal objective as
\begin{align}
    \PP_l(w(\lambda,\gamma),\gamma)&:=\max_{\pi\in\Pi} \left(-\hat{\Delta}_{l-1}^{\gamma^{l-1}}(\pi,\pill)+\gamma_\pi\norm{\phi_\pi-\phi_{\pill}}_{A(w(\lambda,\gamma))^{-1}}^2+\frac{\log(1/\delta_l)}{\gamma_\pi n}\right)\label{eqn:PP_l},
\end{align}

where $w(\lambda,\gamma)\in \R^{|\A|\times |\C|}$ such that
\begin{equation}
    [w(\lambda,\gamma)]_{a, c}=\nu_{c} \cdot p_{c,a}=\nu_{c} \cdot \frac{\sqrt{(\lambda\odot\gamma)^{\top} (t_{a}^{(c)}+\eta)}}{\sum_{a'\in \A} \sqrt{(\lambda\odot\gamma)^{\top} (t_{a^{\prime}}^{(c)}+\eta)}}.
\end{equation}

Then we will show those two results. First, Theorem~\ref{thm:FW_gap} bounds the duality gap of the primal and dual objective. Second, Theorem~\ref{lem:SP} shows that Algorithm \ref{alg:evaluation_oracle} converges to a saddle point.

\begin{theorem}\label{thm:FW_gap}
For any $l\in\N$, with the number of $\textsf{FW-GD}$ iterations $K_l=O(L^2\epsilon_l^{-2})$ where $L=|\A|^2\frac{((1+\eta_l)\gamma_{\max})^{5/2}}{\eta_l^{3/2}\gamma_{\min}^2}$, we have
\begin{align*}
    \big|\PP_l(w(\lambda^{l},\gamma^l),\gamma^l)-h_l(\lambda^{l},\gamma^l)\big|\leq \epsilon_l.
\end{align*}
Moreover, $K_l$ depends at most polynomially on $|\A|,\epsilon_l^{-1},\log(1/\delta_l)$. 
\end{theorem}
\begin{proof}
First, Lemma~\ref{lem:gg} shows that for any $\lambda$, $\gamma$, and $n$,  $h_l(\lambda,\gamma,n)=\inner{\lambda}{\nabla_\lambda h_l(\lambda,\gamma,n)}$. Therefore, at some iteration $t$, the Frank-Wolfe gap
\[g_t=\inner{\nabla_\lambda h_{l}(\lambda^{t},\gamma^{t})}{\e_{\pi_t}-\lambda^{t}}=\max_{\pi\in\Pi}[\nabla_\lambda h_{l}(\lambda^{t},\gamma^{t})]_\pi-h_{l}(\lambda^{t},\gamma^{t}).\] 
Lemma~\ref{lem:PP} shows that with a small choice of the regularization parameter the primal objective is close to the maximum component of the gradient, i.e. $|\PP_l(w(\lambda^l,\gamma^l),\gamma^l)-\max_{\pi\in\Pi}[\nabla_\lambda h_l(\lambda^l, \gamma^l)]_\pi| \leq \frac{\epsilon_l}{2}$. Also, 
Lemma~\ref{lem:saddle_point} shows that if $t\geq L^2\epsilon_l^{-2}$ is large enough, the Frank-Wolfe gap is bounded by $\epsilon_l$. 
Combining these two lemmas, for $t\geq L^2\epsilon_l^{-2}$, we have 
\begin{align*}
    &|\PP_l(w(\lambda^{l},\gamma^l),\gamma^l)-h_l(\lambda^{l},\gamma^l)|\\
    &\leq |\PP_l(w(\lambda^{l},\gamma^l),\gamma^l)-\max_{\pi\in\Pi}[\nabla_\lambda h_l(\lambda^{l}, \gamma^l)]_\pi| + |h_{l}(\lambda^{l},\gamma^l)-\max_{\pi\in\Pi}[\nabla_\lambda h_{l}(\lambda^{l},\gamma^l)]_\pi|\\
    &\leq |\PP_l(w(\lambda,\gamma),\gamma)-\max_{\pi\in\Pi}[\nabla_\lambda h_l(\lambda, \gamma)]_\pi|+g_l\\
    &\leq \frac{\epsilon_l}{2}+\frac{\epsilon_l}{2}=\epsilon_l.
\end{align*}
Finally, we conclude that $K_l=\operatorname{poly}(|\A|,\epsilon_l^{-1},\log(1/\delta_l))$ since $\gamma_{\max}=O(|\A|^{-1}\eta_l^{-1/2})$, $\gamma_{\min}=O(\sqrt{\eta_l})$, and $\eta_l=O(|\A|^{-4}\epsilon_l^2)$ all depends polynomially on $|\A|$ and $\epsilon_l^{-1}$. This shows Theorem \ref{thm:FW_gap}. 
\end{proof}

We now have the second main result of this section.

\begin{theorem}\label{lem:SP}
For any $l$, with $K_l=\operatorname{poly}(|\A|,\epsilon_l^{-1},\log(1/\delta_l))$ and the size of the history $\D\geq\operatorname{poly}(|\A|,\epsilon^{-1},\log(1/\delta),\log(|\Pi|))$, Algorithm \ref{alg:evaluation_oracle} converges to a saddle point, i.e. 
\[\Big|\max _{\lambda\in \Delta_{\Pi}} \min _{\gamma\in [\gamma_{\min},\gamma_{\max}]^{\Pi}} h_l(\lambda,\gamma)-h_l(\lambda^{l},\gamma^l)\Big|\leq \epsilon_l.\]
\end{theorem}

\begin{proof} Note that

\begingroup
\allowdisplaybreaks
\begin{align*}
    &\PP_{l}(w(\lambda^{l},\gamma^l),\gamma^l)\\
    &=\max _{\pi\in\Pi}\left[-\hat{\Delta}_{l-1}^{\gamma^{l-1}}(\pi,\pill)+\frac{\log(1 / \delta_l)}{[\gamma^l]_{\pi} n}+[\gamma^l]_{\pi}\left\|\phi_{\pill}-\phi_{\pi}\right\|_{A(w(\lambda^{l},\gamma^l))^{-1}}^{2}\right]\\
    &\geq \max_{\pi\in\Pi}\min_{\gamma}\left[-\hat{\Delta}_{l-1}^{\gamma^{l-1}}(\pi,\pill)+\frac{\log(1 / \delta_l)}{\gamma_{\pi} n}+\gamma_{\pi}\left\|\phi_{\pill}-\phi_{\pi}\right\|_{A(w(\lambda^{l},\gamma^{l}))^{-1}}^{2}\right]\\
    &\geq \min_{w\in\Omega}\max_{\pi\in\Pi}\min_{\gamma}\left[-\hat{\Delta}_{l-1}^{\gamma^{l-1}}(\pi,\pill)+\frac{\log(1 / \delta_l)}{\gamma_{\pi} n}+\gamma_{\pi}\left\|\phi_{\pill}-\phi_{\pi}\right\|_{A(w)^{-1}}^{2}\right]\\
    &= \max_{\lambda\in\triangle_\Pi}\min_{w\in\Omega}\min_{\gamma\in[\gamma_{\min},\gamma_{\max}]^{\Pi}}\sum_{\pi\in\Pi}\lambda_{\pi}\left(-\hat{\Delta}_{l-1}^{\gamma^{l-1}}(\pi,\pill)+\frac{\log(1 / \delta_l)}{\gamma_{\pi} n}+\gamma_{\pi}\left\|\phi_{\pill}-\phi_{\pi}\right\|_{A(w)^{-1}}^{2}\right)\tag{by Lemma~\ref{lem:strong_duality}, strong duality}\\
    &=\max _{\lambda\in \Delta_{\Pi}} \min _{\gamma\in [\gamma_{\min},\gamma_{\max}]^{\Pi}} \sum_{\pi} \lambda_{\pi} \cdot\left(-\hat{\Delta}_{l-1}^{\gamma^{l-1}}(\pi,\pill)+\frac{\log (1 / \delta_l)}{\gamma_{\pi} n}\right)\\
    &\qquad\qquad\qquad+\mathbb{E}_{c\sim\nu}\left[\left(\sum_{a} \sqrt{(\lambda\odot\gamma)^{\top} t_{a}^{(c)}}\right)^{2}\right]\tag{by Lemma~\ref{lem:est_design_2}}\\
    &\geq\max_{\lambda\in\triangle_\Pi}\min_{\gamma\in [\gamma_{\min},\gamma_{\max}]^{\Pi}} \sum_{\pi\in\Pi}\lambda_\pi\cdot\left(-\hat{\Delta}_{l-1}^{\gamma^{l-1}}(\pi,\pill)+\frac{\log(1/\delta_l)}{\gamma_\pi n}\right)\\
    &\qquad\qquad\qquad+\E_{c\sim\nu}\bigbrak{\bigsmile{\sum_{a\in\A}\sqrt{(\lambda\odot \gamma)^\T (t_a^{(c)}+\eta_l)}}^2}-\frac{1}{2}\epsilon_l\tag{by Lemma~\ref{lem:E_9}}\\
    &\geq \min _{\gamma \in [\gamma_{\min},\gamma_{\max}]^{\Pi}} \sum_{\pi}\left[\lambda^{l}\right]_{\pi} \cdot\left(-\hat{\Delta}_{l-1}^{\gamma^{l-1}}(\pi,\pill)+\frac{\log(1 / \delta_l)}{\gamma_{\pi} n}\right)\\
    &\qquad\qquad\qquad+\E_{c\sim\nu}\bigbrak{\bigsmile{\sum_{a\in\A}\sqrt{(\lambda^{l}\odot \gamma)^\T (t_a^{(c)}+\eta_l)}}^2}-\frac{1}{2}\epsilon_l\\
    &\geq \min _{\gamma \in [\gamma_{\min},\gamma_{\max}]^{\Pi}} \sum_{\pi}\left[\lambda^{l}\right]_{\pi} \cdot\left(-\hat{\Delta}_{l-1}^{\gamma^{l-1}}(\pi,\pill)+\frac{\log(1 / \delta_l)}{\gamma_{\pi} n}\right)\\
    &\qquad\qquad\qquad+\E_{c\sim\nu_\D}\bigbrak{\bigsmile{\sum_{a\in\A}\sqrt{(\lambda^{l}\odot \gamma)^\T (t_a^{(c)}+\eta_l)}}^2}-\frac{3}{4}\epsilon_l\tag{by Lemma~\ref{lem:history_1}, controlling the history}\\
    &\geq \sum_{\pi}\left[\lambda^{l}\right]_{\pi} \cdot\left(-\hat{\Delta}_{l-1}^{\gamma^{l-1}}(\pi,\pill)+\frac{\log(1 / \delta_l)}{[\gamma^l]_{\pi} n}\right)\\
    &\qquad\qquad\qquad+\E_{c\sim\nu_\D}\bigbrak{\bigsmile{\sum_{a\in\A}\sqrt{(\lambda^{l}\odot \gamma^l)^\T (t_a^{(c)}+\eta_l)}}^2}-\epsilon_l\tag{by Lemma~\ref{lem:gd_conv}, the GD convergence}\\
    &= h_l(\lambda^{l},\gamma^l)-\epsilon_l. 
\end{align*}
\endgroup

In other words, 
\[\PP_l(w(\lambda^{l},\gamma^l),\gamma^l)\geq \max_{\lambda\in\Delta_\Pi}\min_{\gamma\in[\gamma_{\min},\gamma_{\max}]^\Pi} h_l(\lambda,\gamma)\geq h_l(\lambda^{l},\gamma^l)-\epsilon_l.\]
On the other hand, by Theorem~\ref{thm:FW_gap}, we have $\PP_l(w(\lambda^{l},\gamma^l),\gamma^l)\leq h_l(\lambda^{l},\gamma^l)+\epsilon_l$. Therefore, we have \[\max_{\lambda\in\Delta_\Pi}\min_{\gamma\in[\gamma_{\min},\gamma_{\max}]^\Pi} h_l(\lambda,\gamma)\in\Big[h_l(\lambda^{l},\gamma^l)-\epsilon_l,h_l(\lambda^{l},\gamma^l)+\epsilon_l\Big]\] and so we have our result. 
\end{proof}

\subsection{Technical proofs}

\subsubsection{Guarantees on \texorpdfstring{$\gamma$}{gamma}}

We first provides some guarantees of $\gamma$ and the convergence of the \textsf{GD} subroutine. 

\begin{lemma}
Consider a fixed $n$. Let $\gamma^*=\arg\min_{\gamma}h_l(\lambda,\gamma,n)$. Then we have for all $i$, \[[\gamma^*]_i\in\bigbrak{\frac{1}{3}\sqrt{\frac{\eta_l\log(1/\delta_l)}{n}},\min\left\{\sqrt{\frac{\log(1/\delta_l)}{2n\E_c[\1\{\pi(c)\ne\pi^*(c)\}]}},\sqrt{\frac{\log(1/\delta_l)}{|\A|^2\eta_l n}}\right\}}.\]
\end{lemma}
\begin{proof}
\begin{align*}
    &[\nabla_\gamma h_l(\lambda,\gamma)]_\pi\\
    &=\E_c\bigbrak{\bigsmile{\sum_{a\in\A}\sqrt{(\lambda\odot\gamma)^\T (t_a^{(c)}+\eta_l)}}\cdot\bigsmile{\sum_{a'\in\A}\frac{\lambda_\pi([t_{a'}^{(c)}]_\pi+\eta_l)}{\sqrt{(\lambda\odot\gamma)^\T (t_{a'}^{(c)}+\eta_l)}}}}-\frac{\lambda_\pi\log(1/\delta_l)}{\gamma_\pi^2 n}\\
    &\geq \E_c\bigbrak{\bigsmile{\sum_{a\in\A}\sqrt{\lambda_\pi([t_{a}^{(c)}]_\pi+\eta_l)}}^2}-\frac{\lambda_\pi\log(1/\delta_l)}{\gamma_\pi^2 n}\\
    &\geq |\A|^2\eta_l\lambda_\pi+2\lambda_\pi\E_c[\1\{\pi(c)\ne\pi^*(c)\}]-\frac{\lambda_\pi\log(1/\delta_l)}{\gamma_\pi^2 n},
\end{align*}
where the first to second line follows from Cauchy-Schwartz - $(\sum_{a} x_a)\sum_{a} \left(\tfrac{y_a}{x_a}\right) \geq (\sum_a \sqrt{y_a})^2$.
We first solve $\frac{\lambda_\pi\log(1/\delta_l)}{\gamma_\pi^2 n}<|\A|^2\eta_l\lambda_\pi$ and get $\gamma_\pi>\sqrt{\frac{\log(1/\delta_l)}{|\A|^2\eta_l n}}$. We also solve $\frac{\lambda_\pi\log(1/\delta_l)}{\gamma_\pi^2 n}<2\lambda_\pi\E_c[\1\{\pi(c)\ne\pi^*(c)\}]$ and get $\gamma_\pi<\sqrt{\frac{\log(1/\delta_l)}{2n\E_c[\1\{\pi(c)\ne\pi^*(c)\}]}}$. 
Therefore, the $\pi$th component of the gradient is always positive whenever $\gamma_\pi>\min\left\{\sqrt{\frac{\log(1/\delta_l)}{2n\E_c[\1\{\pi(c)\ne\pi^*(c)\}]}},\sqrt{\frac{\log(1/\delta_l)}{|\A|^2\eta_l n}}\right\}$. 
Therefore, the minimum $\gamma$ should have $\gamma_\pi\leq \min\left\{\sqrt{\frac{\log(1/\delta_l)}{2n\E_c[\1\{\pi(c)\ne\pi^*(c)\}]}},\sqrt{\frac{\log(1/\delta_l)}{|\A|^2\eta_l n}}\right\}$. 
On the other hand, let $s=\arg\min_{\pi}\gamma_\pi$. 
Then, \[\eta_l\gamma_{s} \leq (\lambda \odot \gamma)^{\top} (t_{a}^{(c)}+\eta_l)=\left(\lambda \odot (t_{a}^{(c)}+\eta_l)\right)^{\top} \gamma \leq\left\|\lambda \odot (t_{a}^{(c)}+\eta_l)\right\|_{1} \cdot\|\gamma\|_{\infty}.\] 
Then 
\begin{align*}
    \sum_{a\in\A}\sqrt{(\lambda\odot\gamma)^\T (t_a^{(c)}+\eta_l)}&\leq \sum_{a\in\A}\sqrt{\left\|\lambda \odot (t_{a}^{(c)}+\eta_l)\right\|_{1}} \cdot\sqrt{\|\gamma\|_{\infty}}.
\end{align*}
Note that 
\begin{align*}
    \left(\sum_{a \in \A} \sqrt{\left\|\lambda \odot (t_{a}^{(c)}+\eta_l)\right\|_{1}}\right)^2 &= \left(\sum_{a \in \A} \sqrt{\lambda^\T(t_a^{(c)}+\eta_l)}\right)^2\\
    &\leq \bigsmile{\sum_{a\in\A}\lambda^\T(t_a^{(c)}+\eta_l)}|\A|\\
    &\leq |\A|(1+\eta_l).
\end{align*}
Since for any $\pi$, $\sum_{a'\in\A}[t_{a'}^{(c)}]_\pi\leq 2$, so
\begin{align*}
    [\nabla_\gamma h_l(\lambda,\gamma)]_\pi&\leq \sqrt{|\A|(1+\eta_l)\norm{\gamma}_\infty}\cdot\frac{(2+\eta_l)\lambda_\pi}{\sqrt{\eta_l\gamma_s}}-\frac{\lambda_\pi\log(1/\delta_l)}{\gamma_\pi^2 n}.
\end{align*}
Let $\pi=s$, then by the fact that $\norm{\gamma}_\infty\leq \sqrt{\frac{\log(1/\delta_l)}{|\A|^2\eta_l n}}$, we have 
\begin{align*}
    [\nabla_\gamma h_l(\lambda,\gamma)]_s&\leq \sqrt{|\A|(1+\eta_l)}\bigsmile{\frac{\log(1/\delta_l)}{|\A|^2\eta_l n}}^{1/4}\cdot\frac{(2+\eta_l)\lambda_s}{\sqrt{\eta_l\gamma_s}}-\frac{\lambda_s\log(1/\delta_l)}{\gamma_s^2 n}.
\end{align*}
We solve $\sqrt{|\A|(1+\eta_l)}\bigsmile{\frac{\log(1/\delta_l)}{|\A|^2\eta_l n}}^{1/4}\cdot\frac{(2+\eta_l)\lambda_s}{\sqrt{\eta_l\gamma_s}}-\frac{\lambda_s\log(1/\delta_l)}{\gamma_s^2 n}<0$. Then we get $$\gamma_s<(1+\eta_l)^{-1/3}(2+\eta_l)^{-2/3}\sqrt{\frac{\eta_l\log(1/\delta_l)}{n}}.$$ 
Since $(1+\eta_l)^{-1/3}(2+\eta_l)^{-2/3}>\frac{1}{3}$ whenever $\eta_l\leq 1$, the $s$th component of the gradient is negative whenever $\gamma_s<\frac{1}{3}\sqrt{\frac{\eta_l\log(1/\delta_l)}{n}}$. Therefore, $\min_{\pi}\gamma_{\pi} \geq \frac{1}{3}\sqrt{\frac{\eta_l\log(1/\delta_l)}{n}}$. 
\end{proof}

\subsubsection{Convergence of Frank-Wolfe gap}
Lemma~\ref{lem:recursive} and \ref{lem:saddle_point} shows that the Frank-Wolfe gap is small. The proof technique follows from the general Frank-Wolfe analysis. 

\begin{lemma}\label{lem:recursive}
For any $\xi\in[0,1]$, any $t$, with $L=|\A|^2\frac{((1+\eta_l)\gamma_{\max})^{5/2}}{\eta_l^{3/2}\gamma_{\min}^2}$, we have $h_l(\lambda^{t+1},\gamma^{t+1})\geq h_l(\lambda^{t},\gamma^{t})+\xi g_t-\frac{1}{2}\xi^2L-\kappa_t$. 
\end{lemma}
\begin{proof}
By $L$-Lipschitz gradient condition of $-h_{\ell}$ in $\lambda$ given in Lemma~\ref{lem:f_Lipschitz} we have 
\begin{align*}
    -h_l(\lambda^{t+1},\gamma^{t+1})&\leq -h_l(\lambda^{t}, \gamma^{t+1}) - \inner{\nabla_{\lambda}h_l(\lambda^{t}, \gamma^{t+1})}{\lambda^{t+1}-\lambda^{t}}+\frac{L}{2}\norm{\lambda^{t+1}-\lambda^{t}}_1^2.
\end{align*}
Therefore, 
\begin{align*}
    h_l(\lambda^{t+1},\gamma^{t+1})&\geq
    h_l(\lambda^{t}, \gamma^{t+1}) + \inner{\nabla_{\lambda}h_l(\lambda^{t}, \gamma^{t+1})}{\lambda^{t+1}-\lambda^{t}}-\frac{L}{2}\norm{\lambda^{t+1}-\lambda^{t}}_1^2.
\end{align*}
Plugging in $\lambda^{t+1}=(1-\beta_t)\lambda^{t}+\beta_t\e_{\pi_t}$ as in line 8 of Algorithm~\ref{alg:FW}, we have 
\begin{align*}
    &h_l((1-\beta_t)\lambda^{t}+\beta_t\e_{\pi_t}, \gamma^{t+1})\\
    &\geq h_l(\lambda^{t},\gamma^{t+1}) +\inner{\nabla_{\lambda}h_l(\lambda^{t}, \gamma^{t+1})}{(1-\beta_t)\lambda^{t}+\beta_t\e_{\pi_t}-\lambda^{t}}-\frac{L}{2}\norm{(1+\beta_t)\lambda^{t}-\beta_t\e_{\pi_t}-\lambda^{t}}_1^2\\
    &=h_l(\lambda^{t},\gamma^{t+1})+\beta_t \inner{\nabla_{\lambda}h_l(\lambda^{t}, \gamma^{t+1})}{\e_{\pi_t}-\lambda^{t}}-\frac{L\beta_t^2}{2}\norm{\e_{\pi_t}-\lambda^{t}}_1^2\\
    &=h_l(\lambda^{t},\gamma^{t+1})+\beta_t g_t-\frac{L\beta_t^2}{2}\norm{\e_{\pi_t}-\lambda^{t}}_1^2.
\end{align*}
Choose $\beta_t:=\arg\max_{\xi\in[0,1]}\{\xi g_t-\frac{\xi^2 L}{2}\norm{\e_{\pi_t}-\lambda^{t}}_1^2\}$. Plugging in this expression gives us 
\begin{align*}
    h_l(\lambda^{t+1}, \gamma^{t+1})&\geq h_l(\lambda^{t},\gamma^{t+1}) +\beta_t \inner{\nabla_{\lambda}h_l(\lambda^{t}, \gamma^{t+1})}{\e_{\pi_t}-\lambda^{t}}-\frac{L\beta_t^2}{2}\norm{\e_{\pi_t}-\lambda^{t}}_1^2\\
    &= h_l(\lambda^{t},\gamma^{t+1}) +\max_{\xi\in[0,1]} \{\xi g_t-\frac{\xi^2 L}{2}\norm{\e_{\pi_t}-\lambda^{t}}_1^2\}\\
    &\geq h_l(\lambda^{t},\gamma^{t+1}) +\xi g_t-\frac{\xi^2 L}{2}
\end{align*}
for any $\xi\in[0,1]$ since $\norm{\e_{\pi_t}-\lambda^{t}}_1^2\leq 1$. Also, by construction of $\gamma^{t+1}$ and Lemma~\ref{lem:gd_conv}, we have 
\begin{equation*}
    h_l(\lambda^{t},\gamma^{t+1})\geq\min_{\gamma}h_l(\lambda^{t},\gamma)\geq h_l(\lambda^{t},\gamma^{t})-\kappa_t.
\end{equation*}
Therefore, our result follows. 
\end{proof}
\begin{lemma}\label{lem:saddle_point}
We have for any $t$, with $L=|\A|^2\frac{((1+\eta_l)\gamma_{\max})^{5/2}}{\eta_l^{3/2}\gamma_{\min}^2}$, $\min_{i\in[1,t]} g_i\leq \frac{L}{\sqrt{t+1}}$. 
\end{lemma}
\begin{proof}
With Lemma~\ref{lem:recursive}, we have
\begin{align*}
    h_l(\lambda^{t+1},\gamma^{t+1},n_{r})&\geq h_l(\lambda^{t},\gamma^{t},n_r)+\xi g_t-\frac{1}{2}\xi^2L-\kappa_t.
\end{align*}
Plugging in the choice $\xi=\min\{\frac{g_t}{L},1\}$, we have $h_l(\lambda^{t+1},\gamma^{t+1},n_{r})\geq h_l(\lambda^{t},\gamma^{t},n_r)+\frac{g_t}{2}\min\{\frac{g_t}{L},1\}-\kappa_t$. Summing this up from 0 to $t$ gives us 
\begin{align*}
    h_l(\lambda^{t+1},\gamma^{t+1},n_{r})-h_l(\lambda_0,\gamma_0,n_r)&\geq\sum_{i=0}^{t}\frac{g_i}{2}\min\{\frac{g_i}{L},1\}-\delta_i\\
    &\geq (t+1)g_t^*\min\{\frac{g_t^*}{L},1\}-\sum_{i=0}^t\delta_i. 
\end{align*}
where $g_t^*=\min_{i=0,\cdots,t}g_i$. Then, as long as $\sum_{i=0}^t\delta_i\leq \epsilon_l$, by the fact that $h_l(\lambda^{t+1},\gamma^{t+1})-h_l(\lambda_0,\gamma_0)\leq \max_{\lambda\in\triangle_\Pi}\min_{\gamma}h_l(\lambda,\gamma)-h_l(\lambda_0,\gamma_0)<\infty$. Therefore, we have $\min_{i\in[1,t]} g_i\leq \frac{L}{\sqrt{t+1}}$. 
\lalit{Note that 
$$\frac{\gamma_{\max}}{\gamma_{\min}}=(1+\eta_l)^{1/3}(2+\eta_l)^{2/3}\frac{1}{|\A|\eta_l}\leq 2|\A|^3\epsilon_l^{-2}.$$
Therefore, Plugging in $L$ gives us
\begin{align*}
    \min_{i\in[1,t]} g_i &\leq |\A|^2\frac{((1+\eta_l)\gamma_{\max})^{5/2}}{\eta_l^{3/2}\gamma_{\min}^2} \frac{1}{\sqrt{t+1}}\leq |\A|^2\frac{(1+\eta_l)^{5/2}}{\eta_l^{3/2}}4|\A|^6\epsilon_l^{-4}\sqrt{\gamma_{\max}}\frac{1}{\sqrt{t+1}}\\
    &\leq 4|\A|^8\epsilon_l^{-4}\frac{1}{\sqrt{t+1}}.
\end{align*}
}
\end{proof}

\subsubsection{Connect the Frank-Wolfe gap to the duality gap}

Lemma~\ref{lem:PP} shows that the primal objective is approximately the maximum component of the gradient of the dual objective, which simplifies our Frank-Wolfe gap expression. 

\begin{lemma}\label{lem:PP}
Consider some $\lambda\in\triangle_\Pi$, $\gamma\in\R_+^{|\Pi|}$, and $n\in\N$. For $\eta_l<|\A|^{-4}\epsilon_l^{2}$, we have $|\PP_l(w(\lambda^{l},\gamma^l),\gamma^l)-\max_{\pi\in\Pi}[\nabla_\lambda h_l(\lambda^{l}, \gamma^l)]_\pi|\leq \epsilon_l$. 
\end{lemma}
\begin{proof}
Observe that for any $\pi,\pi'\in\Pi$ and any $\gamma$, 
{
\small
\begin{align*}
    &\gamma_\pi\norm{\phi_{\pi'}-\phi_\pi}_{A(w(\lambda,\gamma))^{-1}}^2\\
    &=\gamma_\pi\sum_{a,c}\frac{\nu_c^2}{[w(\lambda,\gamma)]_{a,c}}\bigsmile{\1\{\pi'(c)=a,\pi(c)\ne a\}+\1\{\pi'(c)\ne a,\pi(c)=a\}}\\
    &=\gamma_\pi \sum_{c}\nu_c\sum_{a}\bigsmile{\frac{\nu_c}{[w(\lambda,\gamma)]_{a,c}}\bigsmile{\1\{\pi'(c)=a,\pi(c)\ne a\}+\1\{\pi'(c)\ne a,\pi(c)=a\}}}\\
    &= \gamma_\pi\E_{c\sim\nu}\left[\sum_{a}\frac{\sum_{a'\in\A}\sqrt{(\lambda\odot\gamma)^\T (t_{a'}^{(c)}+\eta_l)}}{\sqrt{(\lambda\odot\gamma)^\T (t_{a}^{(c)}+\eta_l)}}\bigsmile{\1\{\pi'(c)=a,\pi(c)\ne a\}+\1\{\pi'(c)\ne a,\pi(c)=a\}}\right]\\
    &= \E_{c\sim\nu}\left[\sum_{a}\frac{\sum_{a'\in\A}\sqrt{(\lambda\odot\gamma)^\T (t_{a'}^{(c)}+\eta_l)}}{\sqrt{(\lambda\odot\gamma)^\T (t_{a}^{(c)}+\eta_l)}}(\gamma_\pi[t_a^{(c)}]_\pi)\right].
\end{align*}
}
Therefore, 
\begin{align*}
    &\PP_l(w(\lambda^{l},\gamma^l),\gamma^l)\\
    &= 
    \max_{\pi\in\Pi}\bigbrace{-\hat{\Delta}_{l-1}^{\gamma^{l-1}}(\pi)+[\gamma^l]_\pi\norm{\phi_{\pi}-\phi_{\pill}}_{A(w(\lambda^{l},\gamma^l))^{-1}}^2+\frac{\log(1/\delta_l)}{[\gamma^l]_\pi n}}\\
    &= \max_{\pi\in\Pi}\bigbrace{-\hat{\Delta}_{l-1}^{\gamma^{l-1}}(\pi)+\E_{c\sim\nu}\left[\sum_{a}\frac{\sum_{a'\in\A}\sqrt{(\lambda^{l}\odot\gamma^l)^\T (t_{a'}^{(c)}+\eta_l)}}{\sqrt{(\lambda^{l}\odot\gamma^l)^\T (t_{a}^{(c)}+\eta_l)}}([\gamma^l]_\pi[t_a^{(c)}]_\pi)\right]+\frac{\log(1/\delta_l)}{[\gamma^l]_\pi n}}.
\end{align*}
Lemma~\ref{lem:history_2} guarantees that we could replace the expectation over context to history of contexts $\nu_\D$ without incurring much error. In particular, for a sufficiently large history $\D$, it guarantees 
\begin{align*}
    &\max_{\pi\in\Pi}\left|\E_{c\sim\nu_\D}\left[\sum_{a}\frac{\sum_{a'\in\A}\sqrt{(\lambda^{l}\odot\gamma^l)^\T (t_{a'}^{(c)}+\eta_l)}}{\sqrt{(\lambda^{l}\odot\gamma^l)^\T (t_{a}^{(c)}+\eta_l)}}([\gamma^l]_\pi[t_a^{(c)}]_\pi)\right]\right.\\
    &\left.-\E_{c\sim\nu}\left[\sum_{a}\frac{\sum_{a'\in\A}\sqrt{(\lambda^{l}\odot\gamma^l)^\T (t_{a'}^{(c)}+\eta_l)}}{\sqrt{(\lambda^{l}\odot\gamma^l)^\T (t_{a}^{(c)}+\eta_l)}}([\gamma^l]_\pi[t_a^{(c)}]_\pi)\right]\right|\leq \frac{\epsilon_l}{2}.
\end{align*}
On the other hand, 
\begin{align*}
    &\max_{\pi\in\Pi}\bigbrace{-\hat{\Delta}_{l-1}^{\gamma^{l-1}}(\pi)+\E_{c\sim\nu_\D}\left[\sum_{a}\frac{\sum_{a'\in\A}\sqrt{(\lambda^{l}\odot\gamma^l)^\T (t_{a'}^{(c)}+\eta_l)}}{\sqrt{(\lambda^{l}\odot\gamma^l)^\T (t_{a}^{(c)}+\eta_l)}}([\gamma^l]_\pi[t_a^{(c)}]_\pi)\right]+\frac{\log(1/\delta_l)}{[\gamma^l]_\pi n}}\\
    &= \max_{\pi\in\Pi}\left\{[\nabla_{\lambda}h_l(\lambda^{l},\gamma^l)]_\pi-\E_{c\sim\nu_\D}\left[\sum_{a}\frac{\sum_{a'\in\A}\sqrt{(\lambda^{l}\odot\gamma^l)^\T (t_{a'}^{(c)}+\eta_l)}}{\sqrt{(\lambda^{l}\odot\gamma^l)^\T (t_{a}^{(c)}+\eta_l)}}[\gamma^l]_\pi\eta_l\right]\right\}.
\end{align*}
Note that when $\gamma_\pi\in[\gamma_{\min},\gamma_{\max}]$,
\begin{align*}
    \E_{c\sim\nu_\D}\left[\sum_{a}\frac{\sum_{a'\in\A}\sqrt{(\lambda\odot\gamma)^\T (t_{a'}^{(c)}+\eta_l)}}{\sqrt{(\lambda\odot\gamma)^\T (t_{a}^{(c)}+\eta_l)}}\gamma_\pi\eta_l\right]\in \bigbrak{0,|\A|^2\sqrt{\frac{\gamma_{\max}(1+\eta_l)}{\gamma_{\min}\eta_l}}\gamma_{\max}\eta_l}. 
\end{align*}
Therefore, for $\eta_l<|\A|^{-4}\epsilon_l^{2}$, 
\[\left|\E_{c\sim\nu_\D}\left[\sum_{a}\frac{\sum_{a'\in\A}\sqrt{(\lambda^{l}\odot\gamma^l)^\T (t_{a'}^{(c)}+\eta_l)}}{\sqrt{(\lambda^{l}\odot\gamma^l)^\T (t_{a}^{(c)}+\eta_l)}}[\gamma^l]_\pi\eta_l\right]\right|\leq\frac{\epsilon_l}{2}.\]
Therefore, we have our results. 
\end{proof}

\subsection{Convergence of gradient descent}\label{sec:conv_gd}

In this subsection we show convergence for gradient descent. 

\begin{algorithm}[!htb]
\caption{\textsf{GD}}
\begin{algorithmic}[1]\label{alg:gd}
\REQUIRE $\lambda^t$, $n$, $\kappa_t$
\STATE define $\iota^t=\epsilon_l^3t^{-3}|\A|^{-6}$
\STATE clip $\lambda$ and define $\td{\lambda}=\clip(\lambda,\iota_t)$
\STATE run gradient descent of on $\gamma$ for $h_l(\td{\lambda},\gamma,n)$ over $\supp(\td{\lambda})$ and output $\gamma^t$
\ENSURE $\gamma^t$
\end{algorithmic}
\end{algorithm}
\lalit{Make the steps be on the superscript. Define}

We will first state the main result of this section. 

\begin{lemma}\label{lem:gd_conv}
With the number of iterations $T=O\big(\frac{L_\gamma}{\iota_t}+\frac{1}{\kappa_t\iota_t}\big)$ with $L_\gamma=|\A|^2\frac{((1+\eta_l)\gamma_{\max})^{3/2}}{\eta_l^{3/2}\gamma_{\min}^2}+\frac{2\log(1/\delta_l)}{n\gamma_{\min}^3}$, we have $h_l(\lambda,\gamma^{t},n)-\min_{\gamma}h_l(\lambda,\gamma,n)\leq \kappa_t$. 
\end{lemma}

\begin{proof}[Proof sketch]

Lemma~\ref{lem:clip_lambda} shows that this clipping does not affect the function value that much. Since we do not assume our function to be convex for $\gamma$, we will show that the stationary point is unique and the gradient is strictly positive around the stationary point. Lemma~\ref{lem:local_strongly_convex} first shows that our function is locally strongly convex around any stationary point. In particular, if we are at a point where the $L_1$ norm of the gradient is less than $\lambda_{\min}$, we are locally strongly convex. Lemma~\ref{lem:gamma_Lipschitz} shows our gradient is Lipschitz with respect to the $L_1$ norm. Then, Lemma~\ref{lem:stationary_point_convergence} then shows that the gradient descent algorithm converges to a stationary point. It is the classical argument for gradient descent algorithm on non-convex objectives \cite{jin2021nonconvex}.
\begin{lemma}\label{lem:stationary_point_convergence}
For any $K$, with $L_\gamma=|\A|^2\frac{((1+\eta_l)\gamma_{\max})^{3/2}}{\eta_l^{3/2}\gamma_{\min}^2}+\frac{2\log(1/\delta_l)}{n\gamma_{\min}^3}$, 
\[\min_{k\leq K}\norm{\nabla_{\gamma} h_l(\lambda,\gamma_k,n)}_1^2\leq 2L_\gamma\frac{h_l(\lambda,\gamma_0,n)-\min_{\gamma}h_l(\lambda,\gamma,n)}{K}.\]
\end{lemma}

With this lemma, we have for a sufficiently large $K$, the minimum gradient can be made arbitrarily small. In particular, for $K\geq L_\gamma\lambda_{\min}^{-1}$ we have that the minimum gradient has $L_1$-norm less than $\lambda_{\min}$, and thus we are in a neighborhood of our stationary point by Lemma~\ref{lem:local_strong_convex_eps}. After that, it takes $O(\frac{1}{\kappa_t\lambda_{\min}})$ steps to converge to a point whose value is at most $\kappa_t$ away from the value of the stationary point. The results in \cite{milnor1997topology} coupled with Lemma~\ref{lem:local_strongly_convex} ensure that our stationary point is unique. Intuitively, if we have two locally strongly convex stationary points, there must be a ``hill" between them, which also corresponds to a stationary point, but we have shown that all stationary points must be ``holes" due to local strong convexity, so the stationary point has to be unique.  Thanks to the clipping, we can lower bound $\lambda_{\min}$ by $\iota_t$, so the total number of steps is $\frac{L}{\lambda_{\min}}+\frac{1}{\kappa_t\lambda_{\min}}=\frac{L}{\iota_t}+\frac{1}{\kappa_t\iota_t}$ which matches the result in Lemma~\ref{lem:gd_conv}. 

\end{proof}
\begin{lemma}\label{lem:clip_lambda}
For some iterate $t$, let $\iota_t=\epsilon_l^{3}t^{-3}|\A|^{-6}$ and denote $\td{\lambda}:=\operatorname{clip}(\lambda,\iota_t)$ where $[\operatorname{clip}(\lambda,\epsilon)]_\pi:=\lambda_\pi\1\{\lambda_\pi\geq \epsilon\}$. Then, for any $\gamma$, we have 
\[\Big|h_l(\td{\lambda},\gamma,n)-h_l(\lambda,\gamma,n)\Big|\leq \kappa_t.\]
\end{lemma}
\begin{proof}
For the first term in $h_l$, in the case where $\lambda_\pi\geq \iota_t$, $h_l(\lambda,\gamma,n) = h_l(\td{\lambda},\gamma,n)$. When $0<\lambda_\pi<\iota_t$. We see that
\begin{align*}
    \sum_{\pi\in\Pi,\lambda_\pi< \iota_t}\lambda_\pi\left(-\hat{\Delta}_{l-1}^{\gamma^{l-1}}(\pi,\pill)+\tfrac{\log(1/\delta_l)}{\gamma_\pi n}\right)< t\epsilon\bigsmile{\frac{1}{\gamma_{\min}}+\frac{1}{\gamma_{\min}}}=\frac{2t\iota_t}{\gamma_{\min}}.
\end{align*}
Then we focus on the expectation part of $h_l(\lambda,\gamma,n)$. Note that
\begin{align*}
    \sqrt{(\lambda\odot\gamma)^\T(t_a^{(c)}+\eta_l)} &= \sqrt{\sum_{\pi,\lambda_\pi\geq\iota_t}\lambda_\pi\gamma_\pi[t_a^{(c)}+\eta_l]_\pi+\sum_{\pi,\lambda_\pi<\iota_t}\lambda_\pi\gamma_\pi[t_a^{(c)}+\eta_l]_\pi}\\
    &= \sqrt{(\td{\lambda}\odot\gamma)^\T(t_a^{(c)}+\eta_l)+\sum_{\pi,\lambda_\pi<\iota_t}\lambda_\pi\gamma_\pi[t_a^{(c)}+\eta_l]_\pi}\\
    &\leq \sqrt{(\td{\lambda}\odot\gamma)^\T(t_a^{(c)}+\eta_l)+t\iota_t\gamma_{\max}}\\
    &\leq \sqrt{(\td{\lambda}\odot\gamma)^\T(t_a^{(c)}+\eta_l)}+\sqrt{t\iota_t\gamma_{\max}}.
\end{align*}
Therefore, 
\begin{align*}
    &\E\bigbrak{\bigsmile{\sum_{a\in\A}\sqrt{(\lambda\odot\gamma)^\T(t_a^{(c)}+\eta_l)}}^2}-\E\bigbrak{\bigsmile{\sum_{a\in\A}\sqrt{(\td{\lambda}\odot\gamma)^\T(t_a^{(c)}+\eta_l)}}^2}\\
    &=\E\left[\bigsmile{\sum_{a\in\A}\sqrt{(\lambda\odot\gamma)^\T(t_a^{(c)}+\eta_l)}+\sqrt{(\td{\lambda}\odot\gamma)^\T(t_a^{(c)}+\eta_l)}}\right.\\
    &\qquad\left.\bigsmile{\sum_{a\in\A}\sqrt{(\lambda\odot\gamma)^\T(t_a^{(c)}+\eta_l)}-\sqrt{(\td{\lambda}\odot\gamma)^\T(t_a^{(c)}+\eta_l)}}\right]\\
    &\leq |\A|\sqrt{\gamma_{\max}}|\A|\sqrt{t\iota_t\gamma_{\max}}\\
    &= |\A|^2\gamma_{\max}\sqrt{t\iota_t}.
\end{align*}
Combining two displays above and plugging in $\gamma_{\min}$ and $\gamma_{\max}$ gives 
\begin{align*}
    \Big|h_l(\td{\lambda},\gamma,n)-h_l(\lambda,\gamma,n)\Big|&\leq \frac{2t\iota_t}{\gamma_{\min}}+|\A|\sqrt{\frac{t\iota_t}{\eta_l}}\\
    &= \frac{2t\iota_t|\A|\epsilon_{l}^{-1}}{\sqrt{\eta_l}}+|\A|\sqrt{\frac{t\iota_t}{\eta_l}}.
\end{align*}
Let RHS be $\kappa_t$ and solve for $\iota_t$ we get $\iota_t\leq \min\{\frac{\sqrt{\eta_l}\kappa_t\epsilon_l}{2t|\A|},\frac{\eta_l\kappa_t}{|\A|^2t}\}$. Plugging in $\eta_l=|\A|^{-4}\epsilon_l^2$ gives the result. 
\end{proof}

\begin{lemma}\label{lem:G_11}
Suppose $\gamma^{t}$ satisfies that $h_l(\td{\lambda},\gamma^{t},n)-\min_{\gamma}h_l(\td{\lambda},\gamma,n)\leq \kappa_t$, then we also have $h_l(\lambda,\gamma^{t},n)-\min_{\gamma}h_l(\lambda,\gamma,n)\leq \kappa_t$, i.e. $\gamma^{t}$ satisfies the desired property. 
\end{lemma}
\begin{proof}
Let $\td{\gamma}_*=\arg\min_{\gamma}h_l(\td{\lambda},\gamma,n)$ and $\gamma_*=\arg\min_{\gamma}h_l(\lambda,\gamma,n)$. The result follows from applying Lemma~\ref{lem:clip_lambda} twice on $h_l(\td{\lambda},\gamma^{t},n)$ and $h_l(\td{\lambda},\gamma_*,n)$. In particular, 
\begin{align*}
    h_l(\lambda,\gamma^{t},n)&\leq h_l(\td{\lambda},\gamma^{t},n)+\kappa_t\tag{Lemma~\ref{lem:clip_lambda}}\\
    &\leq h_l(\td{\lambda},\td{\gamma}_*,n)+2\kappa_t\tag{convergence of GD}\\
    &\leq h_l(\td{\lambda},\gamma_*,n)+2\kappa_t\tag{minimality of $\td{\gamma}_*$}\\
    &\leq h_l(\lambda,\gamma_*,n)+3\kappa_t\tag{Lemma~\ref{lem:clip_lambda}}\\
    &= \min_{\gamma}h_l(\lambda,\gamma,n)+3\kappa_t.
\end{align*}
\end{proof}
\subsection{Guarantees for strong concavity and local strong convexity}

The following series of lemmas show that our optimization problem is strongly concave in $\lambda$ and local strongly convex around the minimum $\gamma$, as well as explicitly constructing the Lipschitz constants. These serve as the conditions for convergence of the Frank-Wolfe and gradient descent algorithms. 

\begin{lemma}\label{lem:concave}
$h_l(\lambda,\gamma,n)$ is a concave function of $\lambda$.
\end{lemma}
\begin{proof}
Note that 
\begin{align*}
    \E\bigbrak{\left(\sum_{a\in\A} \sqrt{(\lambda\odot\gamma)^\T (t_a^{(c)}+\eta_l)}\right)^2}&=\E\bigbrak{\sum_{a\in\A}\sum_{a'\in\A}\sqrt{(t_{a'}^{(c)}+\eta_l)^\T(\lambda\odot\gamma)(\lambda\odot\gamma)^\T (t_a^{(c)}+\eta_l)}}.
\end{align*}
we know that $\lambda\mapsto (t_{a'}^{(c)}+\eta_l)^\T(\lambda\odot\gamma)$ and $\lambda\mapsto (\lambda\odot\gamma)^\T (t_a^{(c)}+\eta_l)$ are concave, the square root function is concave and non-decreasing, and sum of concave functions is concave. Therefore, $h_l(\lambda,\gamma,n)$ is concave in $\lambda$ by property of concave functions. 
\end{proof}

\begin{lemma}\label{lem:f_Lipschitz}
Consider some $\lambda$, $\gamma$ and $n$. For any $\lambda_1,\lambda_2\in\triangle_\Pi$, with $L=|\A|^2\frac{((1+\eta_l)\gamma_{\max})^{5/2}}{\eta_l^{3/2}\gamma_{\min}^2}$,
\[f(\lambda_2,\gamma,n)\leq f(\lambda_1,\gamma,n)+\nabla_\lambda f(\lambda_1,\gamma,n)^\T (\lambda_2-\lambda_1)+L\|\lambda_2-\lambda_1\|_1^2,\]
where $f(\lambda,\gamma,n)$ could be either $h_l(\lambda,\gamma,n)$ or $-h_l(\lambda,\gamma,n)$.
\end{lemma}
\begin{proof}
The proof for the negative case is exactly the same as the positive case, so we focus on $f(\lambda,\gamma,n)=h_l(\lambda,\gamma,n)$. We take the gradient of $h_l$ with respect to $\lambda$ and get
\begin{align*}
    \left[\nabla_\lambda h_l(\lambda,\gamma,n)\right]_\pi&=-\hat{\Delta}_{l-1}^{\gamma^{l-1}}(\pi,\hat{\pi}_{l-1})+\frac{\log(1/\delta_l)}{\gamma_\pi n}\\
    &\qquad+\E_{c\sim\nu_\D}\bigbrak{\bigsmile{\sum_{a\in\A}\sqrt{(\lambda\odot \gamma)^\T (t_a^{(c)}+\eta_l)}}\bigsmile{\sum_{a'\in\A}\frac{\gamma_\pi(t_{a'}^{(c)}+\eta_l)_\pi}{\sqrt{(\lambda\odot \gamma)^\T (t_{a'}^{(c)}+\eta_l)}}}}.
\end{align*}
By Lemma~\ref{lem:gg}, for any $\lambda\in\triangle_\Pi$, we have $\inner{\lambda}{\nabla_{\lambda} h_l(\lambda,\gamma,n)}= h_l(\lambda,\gamma,n)$. If we use the shortcut $f(\lambda):=h_l(\lambda,\gamma,n)$, we have
\[ f(\lambda_2)-  f(\lambda_1)-\nabla_{\lambda} f(\lambda_1)^\T (\lambda_2-\lambda_1)= f(\lambda_2)-\nabla_{\lambda} f(\lambda_1)^\T\lambda_2=(\nabla f(\lambda_2)-\nabla f(\lambda_1))^\T\lambda_2.\] 
Note that
{
\small
\begin{align}
    &(\nabla_\lambda f(\lambda_2)-\nabla_\lambda f(\lambda_1))^\T\lambda_2\nonumber\\
    &=\sum_{\pi\in\Pi}[\lambda_2]_\pi\E_{c\sim\nu_\D}\left[\bigsmile{\sum_{a\in\A}\sqrt{(\lambda_2\odot\gamma)^\T(t_a^{(c)}+\eta_l)}}\bigsmile{\sum_{a'\in\A}\frac{\gamma_\pi\cdot(t_{a'}^{(c)}+\eta_l)_\pi}{\sqrt{(\lambda_2\odot\gamma)^\T(t_{a'}^{(c)}+\eta_l)}}}\right.\nonumber\\
    &\qquad\left.-\bigsmile{\sum_{a\in\A}\sqrt{(\lambda_1\odot\gamma)^\T(t_a^{(c)}+\eta_l)}}\bigsmile{\sum_{a'\in\A}\frac{\gamma_\pi\cdot(t_{a'}^{(c)}+\eta_l)_\pi}{\sqrt{(\lambda_1\odot\gamma)^\T(t_{a'}^{(c)}+\eta_l)}}}\right]\nonumber\\
    &= \E_{c\sim\nu_\D}\left[\sum_{a^{\prime} \in \A}(\lambda_{2}\odot\gamma)^{\top} (t_{a^{\prime}}^{(c)}+\eta_l)\right. \nonumber\\
    &\cdot\left.\sum_{a \in \A} \frac{\sqrt{(\lambda_{1}\odot\gamma)^{\top} (t_{a^{\prime}}^{(c)}+\eta_l)}\sqrt{(\lambda_{2}\odot\gamma)^{\top} (t_{a}^{(c)}+\eta_l)}-\sqrt{(\lambda_{2}\odot\gamma)^{\top} (t_{a'}^{(c)}+\eta_l)}\sqrt{(\lambda_{1}\odot\gamma)^{\top} (t_{a}^{(c)}+\eta_l)}}{\sqrt{(\lambda_{2}\odot\gamma)^{\top} (t_{a^{\prime}}^{(c)}+\eta_l)}\sqrt{(\lambda_{1}\odot\gamma)^{\top} (t_{a^{\prime}}^{(c)}+\eta_l)}}\right]\nonumber\\
    &\leq \E_{c\sim\nu_\D}\left[\sum_{a^{\prime} \in \A}(\lambda_{2}\odot\gamma)^{\top} (t_{a^{\prime}}^{(c)}+\eta_l) \right.\nonumber\\
    &\cdot\left.\sum_{a \in \A} \frac{\left|\sqrt{(\lambda_{1}\odot\gamma)^{\top} (t_{a^{\prime}}^{(c)}+\eta_l)}\sqrt{(\lambda_{2}\odot\gamma)^{\top} (t_{a}^{(c)}+\eta_l)}-\sqrt{(\lambda_{2}\odot\gamma)^{\top} (t_{a'}^{(c)}+\eta_l)}\sqrt{(\lambda_{1}\odot\gamma)^{\top} (t_{a}^{(c)}+\eta_l)}\right|}{\sqrt{(\lambda_{2}\odot\gamma)^{\top} (t_{a^{\prime}}^{(c)}+\eta_l)}\sqrt{(\lambda_{1}\odot\gamma)^{\top} (t_{a^{\prime}}^{(c)}+\eta_l)}}\right]\nonumber\\
    &\leq \sum_{a^{\prime} \in \A}\frac{(1+\eta_l)\gamma_{\max}}{\eta_l\gamma_{\min}} \cdot\E_{c\sim\nu_\D}\left[\sum_{a \in \A} \left|\sqrt{(\lambda_{2}\odot\gamma)^{\top} (t_{a}^{(c)}+\eta_l)}\sqrt{(\lambda_{1}\odot\gamma)^{\top} (t_{a^{\prime}}^{(c)}+\eta_l)}\right.\right.\nonumber\\
    &\qquad\left.\left.-\sqrt{(\lambda_{1}\odot\gamma)^{\top} (t_{a}^{(c)}+\eta_l)}\sqrt{(\lambda_{2}\odot\gamma)^{\top} (t_{a^{\prime}}^{(c)}+\eta_l)}\right|\right]\label{eqn:par}
\end{align}
}
Note that by triangular inequality
\begin{align*}
    &\left|\sqrt{(\lambda_{2}\odot\gamma)^{\top} (t_{a}^{(c)}+\eta_l)}\sqrt{(\lambda_{1}\odot\gamma)^{\top} (t_{a^{\prime}}^{(c)}+\eta_l)}-\sqrt{(\lambda_{1}\odot\gamma)^{\top} (t_{a}^{(c)}+\eta_l)}\sqrt{(\lambda_{2}\odot\gamma)^{\top} (t_{a'}^{(c)}+\eta_l)}\right|\\
    &\leq\bigabs{\sqrt{(\lambda_{2}\odot\gamma)^{\top} (t_{a}^{(c)}+\eta_l)}-\sqrt{(\lambda_{1}\odot\gamma)^{\top} (t_{a}^{(c)}+\eta_l)}}\sqrt{(\lambda_{1}\odot\gamma)^{\top} (t_{a^{\prime}}^{(c)}+\eta_l)}\\
    &\qquad+\sqrt{(\lambda_{1}\odot\gamma)^{\top} (t_{a}^{(c)}+\eta_l)}\bigabs{\sqrt{(\lambda_{1}\odot\gamma)^{\top} (t_{a^{\prime}}^{(c)}+\eta_l)}-\sqrt{(\lambda_{2}\odot\gamma)^{\top} (t_{a^{\prime}}^{(c)}+\eta_l)}}.
\end{align*}
Also note that 
\begin{align*}
    &\left|\sqrt{(\lambda_{2}\odot\gamma)^{\top} (t_{a}^{(c)}+\eta_l)}-\sqrt{(\lambda_{1}\odot\gamma)^{\top} (t_{a}^{(c)}+\eta_l)}\right|\\
    &= \frac{\left|\sum_{\pi\in\Pi}((\lambda_2)_\pi-(\lambda_1)_\pi)\gamma_\pi(t_a^{(c)}+\eta_l)_\pi\right|}{\sqrt{(\lambda_{2}\odot\gamma)^{\top} (t_{a}^{(c)}+\eta_l)}+\sqrt{(\lambda_{1}\odot\gamma)^{\top} (t_{a}^{(c)}+\eta_l)}}\\
    &\leq \frac{(1+\eta_l)\gamma_{\max}}{2\sqrt{\eta_l}\gamma_{\min}}\norm{\lambda_2-\lambda_1}_1,
\end{align*}
so \eqref{eqn:par} is bounded by 
\begin{align*}
    &\sum_{a^{\prime} \in \A}\frac{(1+\eta_l)\gamma_{\max}}{\eta_l\gamma_{\min}} \cdot\left(\sum_{a \in \A} 2\cdot\frac{(1+\eta_l)\gamma_{\max}}{2\sqrt{\eta_l}\gamma_{\min}}\norm{\lambda_2-\lambda_1}_1\sqrt{(1+\eta_l)\gamma_{\max}}\right)\\
    &=|\A|^2\frac{((1+\eta_l)\gamma_{\max})^{5/2}}{\eta_l^{3/2}\gamma_{\min}^2}\norm{\lambda_2-\lambda_1}_1.
\end{align*}
\end{proof}

\begin{lemma}\label{lem:gamma_Lipschitz}
Consider some $\lambda$ and $n$. For any $\gamma_1,\gamma_2\in\triangle_\Pi$, with $L_{\gamma}=|\A|^2\frac{((1+\eta_l)\gamma_{\max})^{3/2}}{\eta_l^{3/2}\gamma_{\min}^2}+\frac{2\log(1/\delta_l)}{n\gamma_{\min}^3}$,
\[h_l(\lambda,\gamma_2,n)\leq h_l(\lambda,\gamma_1,n)+\nabla_\gamma h_l(\lambda,\gamma_1,n)^\T (\gamma_2-\gamma_1)+L_{\gamma}\|\gamma_2-\gamma_1\|_1^2. \]
\end{lemma}
\begin{proof}
{
\small
\begin{align*}
    [\nabla_\gamma h_l(\lambda,\gamma)]_\pi &=\E_c\bigbrak{\bigsmile{\sum_{a\in\A}\sqrt{(\lambda\odot\gamma)^\T (t_a^{(c)}+\eta_l)}}\cdot\bigsmile{\sum_{a'\in\A}\frac{\lambda_\pi([t_{a'}^{(c)}]_\pi+\eta_l)}{\sqrt{(\lambda\odot\gamma)^\T (t_{a'}^{(c)}+\eta_l)}}}}-\frac{\lambda_\pi\log(1/\delta_l)}{\gamma_\pi^2 n}.
\end{align*}
}
Then we have similar to the proof of Lemma~\ref{lem:f_Lipschitz}, for any $\gamma$ we have  $h_l(\lambda,\gamma,n)-\nabla_\gamma h_l(\lambda,\gamma,n)^\T\gamma=2\sum_{\pi}\frac{\lambda_\pi\log(1/\delta_l)}{\gamma_\pi^2 n}$, so 
\begin{align*}
    &h_l(\lambda,\gamma_2,n)- h_l(\lambda,\gamma_1,n)-\nabla_\gamma h_l(\lambda,\gamma_1,n)^\T (\gamma_2-\gamma_1)\\
    &= 2\sum_{\pi}\frac{\lambda_\pi\log(1/\delta_l)}{[\gamma_2]_\pi^2 n}-2\sum_{\pi}\frac{\lambda_\pi\log(1/\delta_l)}{[\gamma_1]_\pi^2 n}+(\nabla_\gamma h_l(\lambda,\gamma_2,n)-\nabla_\gamma h_l(\lambda,\gamma_1,n))^\T \gamma_2.
\end{align*}
First, we can follow similar techniques in the proof of Lemma~\ref{lem:f_Lipschitz} to bound the second part and get
{
\small
\begin{align*}
    &(\nabla_\gamma h_l(\lambda,\gamma_2,n)-\nabla_\gamma h_l(\lambda,\gamma_1,n))^\T \gamma_2\\
    &\leq \sum_{a^{\prime} \in \A}(\lambda\odot\gamma_{2})^{\top} (t_{a^{\prime}}^{(c)}+\eta_l) \nonumber\\
    &\quad\cdot\E_{c\sim\nu_\D}\left\{\sum_{a \in \A} \left[\frac{1}{\sqrt{(\lambda\odot\gamma_{2})^{\top} (t_{a^{\prime}}^{(c)}+\eta_l)}\sqrt{(\lambda\odot\gamma_{1})^{\top} (t_{a^{\prime}}^{(c)}+\eta_l)}}\right.\right.\\
    &\left.\left.\quad\cdot\left|\sqrt{(\lambda\odot\gamma_{1})^{\top} (t_{a^{\prime}}^{(c)}+\eta_l)}\sqrt{(\lambda\odot\gamma_{2})^{\top} (t_{a}^{(c)}+\eta_l)}-\sqrt{(\lambda\odot\gamma_{2})^{\top} (t_{a'}^{(c)}+\eta_l)}\sqrt{(\lambda\odot\gamma_{1})^{\top} (t_{a}^{(c)}+\eta_l)}\right|\right]\right\}\nonumber\\
    &\leq \sum_{a^{\prime} \in \A}\frac{(1+\eta_l)\gamma_{\max}}{\eta_l\gamma_{\min}} \cdot\E_{c\sim\nu_\D}\left[\sum_{a \in \A} \left|\sqrt{(\lambda\odot\gamma_{2})^{\top} (t_{a}^{(c)}+\eta_l)}\sqrt{(\lambda\odot\gamma_{1})^{\top} (t_{a^{\prime}}^{(c)}+\eta_l)}\right.\right.\nonumber\\
    &\qquad\left.\left.-\sqrt{(\lambda\odot\gamma_{1})^{\top} (t_{a}^{(c)}+\eta_l)}\sqrt{(\lambda\odot\gamma_{2})^{\top} (t_{a^{\prime}}^{(c)}+\eta_l)}\right|\right].
\end{align*}
}
Also, note that 
\begin{align*}
    &\left|\sqrt{(\lambda\odot\gamma_{2})^{\top} (t_{a}^{(c)}+\eta_l)}-\sqrt{(\lambda\odot\gamma_{1})^{\top} (t_{a}^{(c)}+\eta_l)}\right|\\
    &= \frac{\left|\sum_{\pi\in\Pi}(\lambda_\pi([\gamma_2]_\pi-[\gamma_1]_\pi)(t_a^{(c)})_\pi\right|}{\sqrt{(\lambda\odot\gamma_{2})^{\top} (t_{a}^{(c)}+\eta_l)}+\sqrt{(\lambda\odot\gamma_{1})^{\top} (t_{a}^{(c)}+\eta_l)}}\\
    &\leq \frac{1}{2\sqrt{\eta_l}\gamma_{\min}}\norm{\gamma_2-\gamma_1}_1^2,
\end{align*}
Therefore, similarly we can bound 
\begin{align*}
    &\left|\sqrt{(\lambda\odot\gamma_{2})^{\top} (t_{a}^{(c)}+\eta_l)}\sqrt{(\lambda\odot\gamma_{1})^{\top} (t_{a^{\prime}}^{(c)}+\eta_l)}-\sqrt{(\lambda\odot\gamma_{1})^{\top} (t_{a}^{(c)}+\eta_l)}\sqrt{(\lambda\odot\gamma_{2})^{\top} (t_{a^{\prime}}^{(c)}+\eta_l)}\right|\\
    &\leq \frac{\sqrt{(1+\eta_l)\gamma_{\max}}}{2\sqrt{\eta_l}\gamma_{\min}}\|\gamma_2-\gamma_1\|_1^2.
\end{align*}
For the second term, 
\begin{align*}
    &2\sum_{\pi}\frac{\lambda_\pi\log(1/\delta_l)}{[\gamma_2]_\pi^2 n}-2\sum_{\pi}\frac{\lambda_\pi\log(1/\delta_l)}{[\gamma_1]_\pi^2 n} \\
    &= \frac{2\log(1/\delta_l)}{n}\sum_{\pi}\lambda_\pi\frac{[\gamma_1]_\pi^2-[\gamma_2]_\pi^2}{[\gamma_1]_\pi^2[\gamma_2]_\pi^2}\\
    &\leq \frac{2\log(1/\delta_l)}{n\gamma_{\min}^3}\|\gamma_2-\gamma_1\|_1^2.
\end{align*}
Therefore, we have the result stated above. 
\end{proof}

\begin{lemma}\label{lem:local_strongly_convex}
Consider some fixed $\lambda\in\triangle_\Pi$ and $n$. Assume $\gamma_*$ is a stationary point of $h_l(\lambda,\gamma,n)$, then $h_l(\lambda,\gamma,n)$ is locally strongly convex at $\gamma_*$, i.e. for $L_{\He}=\frac{\lambda_{\min}\log(1/\delta_l)}{\gamma_{\max}^3 n}$, there exists $\epsilon>0$ such that for all $\gamma\in B_\epsilon(\gamma_*)$, $h_l(\lambda,\gamma,n)\geq h_l(\lambda,\gamma_*,n)+\frac{L_{\He}}{2}\norm{\gamma-\gamma_*}^2$. 
\end{lemma}
\begin{proof}
Since $\lambda$ and $n$ are fixed, we use the shortcut $g(\gamma):=h_l(\lambda,\gamma,n)$ in the proof. Denote the Hessian of $g$ as $M$. We aim to show that the Hessian $M\succeq L_{\He} I$ at $\gamma_*$. First, since $\gamma_*$ is a stationary point, $\nabla_\gamma g(\gamma_*)=0$, and so for any $i$, 
\begin{equation}\label{eqn:211}
    \sum_{c\in \mc{D}}\nu_{c_\D}\bigsmile{\sum_{a\in\A}\sqrt{(\lambda\odot\gamma)^\T (t_a^{(c)}+\eta_l)}}\cdot\bigsmile{\sum_{a'\in\A}\frac{\lambda_i([t_{a'}^{(c)}]_i+\eta_l)}{\sqrt{(\lambda\odot\gamma)^\T (t_{a'}^{(c)}+\eta_l)}}}=\frac{\lambda_i\log(1/\delta_l)}{\gamma_i^2 n}.
\end{equation}
Also, we have for $i\neq j$,
\begin{align*}
\frac{\partial^2 g(\gamma)}{\partial \gamma_{i} \gamma_{j}}=&\sum_{c\in \mc{D}}\nu_{c_\D}\left(\sum_{a^{\prime} \in \A} \frac{1}{2}\frac{\lambda_{i}\left[t_{a}^{(c)}+\eta_l\right]_{i}}{\sqrt{(\lambda\odot\gamma)^\T (t_a^{(c)}+\eta_l)}}\right) \cdot\left(\sum_{a \in \A} \frac{\lambda_{j}\left[t_{a}^{(c)}+\eta_l\right]_{j}}{\sqrt{(\lambda\odot\gamma)^\T (t_a^{(c)}+\eta_l)}}\right) \\
&+\left(\sum_{a \in \A} \sqrt{(\lambda\odot\gamma)^\T (t_a^{(c)}+\eta_l)}\right) \cdot\left(\sum_{a^{\prime} \in \A}-\frac{1}{2} \cdot \frac{\left.\lambda_{i} \lambda_{j}\left[t_{a'}^{(c)}+\eta_l\right]_{i}\left[t_{a'}^{(c)}+\eta_l\right]_{j}\right)}{\left((\lambda\odot\gamma)^\T (t_{a'}^{(c)}+\eta_l)\right)^{3 / 2}}\right).
\end{align*}
And 
\begin{align*}
\frac{\partial^2 g(\gamma)}{\partial \gamma_{i}^2}=&\frac{2\lambda_i\log(1/\delta_l)}{\gamma_i^3 n}+\sum_{c\in\D}\nu_{c_\D}\frac{1}{2}\left(\sum_{a^{\prime} \in \A} \frac{\lambda_{i}\left[t_{a}^{(c)}+\eta_l\right]_{i}}{\sqrt{(\lambda\odot\gamma)^\T (t_a^{(c)}+\eta_l)}}\right)^2\\
&-\frac{1}{2}\left(\sum_{a \in \A} \sqrt{(\lambda\odot\gamma)^\T (t_a^{(c)}+\eta_l)}\right) \cdot\left(\sum_{a^{\prime} \in \A} \frac{\lambda_{i}^2 \left[t_{a}^{(c)}+\eta_l\right]_{i}^2}{\left((\lambda\odot\gamma)^\T (t_a^{(c)}+\eta_l)\right)^{3 / 2}}\right).
\end{align*}
Then, for any vector $\mu\in\R^{|\Pi|}$ with $\norm{\mu}=1$, we have 
\begin{align}
    \mu^\T M\mu &= \sum_{i}\sum_{j}\mu_i\mu_jM_{ij}=\sum_{i}\mu_i^2M_{ii}+\sum_{i\ne j}\mu_i\mu_jM_{ij}\nonumber\\
    &= \sum_{i}\mu_i^2\frac{2\lambda_i\log(1/\delta_l)}{\gamma_i^3 n}\\
    &\ +\sum_{c}\nu_c\sum_{i}\sum_{j}\mu_i\mu_j\frac{1}{2}\left(\sum_{a^{\prime} \in \A} \frac{\lambda_{i}\left[t_{a}^{(c)}+\eta_l\right]_{i}}{\sqrt{(\lambda\odot\gamma)^\T (t_a^{(c)}+\eta_l)}}\right) \cdot\left(\sum_{a \in \A} \frac{\lambda_{j}\left[t_{a}^{(c)}+\eta_l\right]_{j}}{\sqrt{(\lambda\odot\gamma)^\T (t_a^{(c)}+\eta_l)}}\right) \nonumber\\
    &\ +\mu_i\mu_j\left(\sum_{a \in \A} \sqrt{(\lambda\odot\gamma)^\T (t_a^{(c)}+\eta_l)}\right) \cdot\left(\sum_{a^{\prime} \in \A}-\frac{1}{2} \cdot \frac{\lambda_{i} \lambda_{j}\left[t_{a'}^{(c)}+\eta_l\right]_{i}\left[t_{a'}^{(c)}+\eta_l\right]_{j}}{\left((\lambda\odot\gamma)^\T (t_{a'}^{(c)}+\eta_l)\right)^{3 / 2}}\right).\label{eqn:190}
\end{align}
In what follows, we will first show that
\begin{align}
    &\sum_{i}\mu_i^2\frac{\lambda_i\log(1/\delta_l)}{\gamma_i^3 n}-\sum_c\nu_c\sum_{i}\sum_{j}\mu_i\mu_j\left(\sum_{a \in \A} \sqrt{(\lambda\odot\gamma)^\T (t_a^{(c)}+\eta_l)}\right)\nonumber\\
    &\qquad\cdot\left(\sum_{a^{\prime} \in \A}\cdot \frac{\lambda_{i} \lambda_{j}\left[t_{a'}^{(c)}+\eta_l\right]_{i}\left[t_{a'}^{(c)}+\eta_l\right]_{j}}{\left((\lambda\odot\gamma)^\T (t_{a'}^{(c)}+\eta_l)\right)^{3 / 2}}\right)\geq 0.\label{eqn:200}
\end{align}
By equation~\ref{eqn:211}, the LHS of \eqref{eqn:190} simplifies to 
\begin{align*}
    &\sum_{c}\nu_c\sum_{i}\mu_i^2\frac{1}{\gamma_i}\left(\sum_{a \in \A} \sqrt{(\lambda\odot\gamma)^\T (t_a^{(c)}+\eta_l)}\right)\left(\sum_{a^{\prime} \in \A}\frac{\lambda_{i} \left[t_{a'}^{(c)}+\eta_l\right]_{i}}{\sqrt{(\lambda\odot\gamma)^\T (t_{a'}^{(c)}+\eta_l)}}\right)\\
    &-\sum_{c}\nu_c\sum_{i}\sum_{j}\mu_i\mu_j\left(\sum_{a \in \A} \sqrt{(\lambda\odot\gamma)^\T (t_a^{(c)}+\eta_l)}\right) \cdot\left(\sum_{a^{\prime} \in \A}\cdot \frac{\lambda_{i} \lambda_{j}\left[t_{a'}^{(c)}+\eta_l\right]_{i}\left[t_{a'}^{(c)}+\eta_l\right]_{j}}{\left((\lambda\odot\gamma)^\T (t_{a'}^{(c)}+\eta_l)\right)^{3 / 2}}\right).
\end{align*}
Therefore, it is sufficient to show that 
{
\small
\begin{align*}
    \sum_{i}\mu_i^2\frac{1}{\gamma_i}\left(\sum_{a^{\prime} \in \A} \frac{\lambda_{i} \left[t_{a'}^{(c)}+\eta_l\right]_{i}}{\sqrt{(\lambda\odot\gamma)^\T (t_{a'}^{(c)}+\eta_l)}}\right)-\sum_{i}\sum_{j}\mu_i\mu_j\left(\sum_{a^{\prime} \in \A} \frac{\lambda_{i} \lambda_{j}\left[t_{a'}^{(c)}+\eta_l\right]_{i}\left[t_{a'}^{(c)}+\eta_l\right]_{j}}{\left((\lambda\odot\gamma)^\T (t_{a'}^{(c)}+\eta_l)\right)^{3 / 2}}\right)\geq 0.
\end{align*}
}
Consider some $a'\in\A$. The LHS of the above simplifies to 
\begin{align*}
    &\sum_{i}\mu_i^2\frac{1}{\gamma_i}\frac{\lambda_{i} \left[t_{a'}^{(c)}+\eta_l\right]_{i}}{\sqrt{(\lambda\odot\gamma)^\T (t_{a'}^{(c)}+\eta_l)}}-\sum_{i}\sum_{j}\mu_i\mu_j \frac{\lambda_{i} \lambda_{j}\left[t_{a'}^{(c)}+\eta_l\right]_{i}\left[t_{a'}^{(c)}+\eta_l\right]_{j}}{\left((\lambda\odot\gamma)^\T (t_{a'}^{(c)}+\eta_l)\right)^{3 / 2}}\\
    &= \frac{1}{\left((\lambda\odot\gamma)^\T (t_{a'}^{(c)}+\eta_l)\right)^{3 / 2}}\left(\sum_i\frac{\mu_i^2}{\gamma_i}\lambda_i\left[t_{a'}^{(c)}+\eta_l\right]_{i}\bigsmile{\sum_{j}\lambda_j\gamma_j\left[t_{a'}^{(c)}+\eta_l\right]_{j}}\right.\\
    &\left.\qquad-\sum_{i}\sum_{j}\mu_i\mu_j\lambda_{i} \lambda_{j}\left[t_{a'}^{(c)}+\eta_l\right]_{i}\left[t_{a'}^{(c)}+\eta_l\right]_{j}\right)\\
    &= \frac{1}{\left((\lambda\odot\gamma)^\T (t_{a'}^{(c)}+\eta_l)\right)^{3 / 2}}\left(\sum_i\sum_j \gamma_i^{-1}\left(\mu_i^2\lambda_i\left[t_{a'}^{(c)}+\eta_l\right]_{i}\lambda_j\gamma_j\left[t_{a'}^{(c)}+\eta_l\right]_{j}\right.\right.\\
    &\left.\left.\qquad-\mu_i\mu_j\lambda_{i} \lambda_{j}\gamma_i\left[t_{a'}^{(c)}+\eta_l\right]_{i}\left[t_{a'}^{(c)}+\eta_l\right]_{j}\right)\right).
\end{align*}
Each summand is 
\begin{align*}
    &\gamma_i^{-1}\bigsmile{\mu_i^2\lambda_i\left[t_{a'}^{(c)}+\eta_l\right]_{i}\lambda_j\gamma_j\left[t_{a'}^{(c)}+\eta_l\right]_{j}-\mu_i\mu_j\lambda_{i} \lambda_{j}\gamma_i\left[t_{a'}^{(c)}+\eta_l\right]_{i}\left[t_{a'}^{(c)}+\eta_l\right]_{j}}\\
    &= \gamma_i^{-1}\mu_i\lambda_{i} \lambda_{j}\left[t_{a'}^{(c)}+\eta_l\right]_{i}\left[t_{a'}^{(c)}+\eta_l\right]_{j}\bigsmile{\mu_i\gamma_j-\mu_j\gamma_i}\\
    &= \gamma_i^{-1}\gamma_j^{-1}\lambda_{i} \lambda_{j}\left[t_{a'}^{(c)}+\eta_l\right]_{i}\left[t_{a'}^{(c)}+\eta_l\right]_{j}(\mu_i\gamma_j)\bigsmile{\mu_i\gamma_j-\mu_j\gamma_i}.
\end{align*}
Exchanging subscripts of $i$ and $j$, we have 
\[\gamma_j^{-1}\gamma_i^{-1}\lambda_{j} \lambda_{i}\left[t_{a'}^{(c)}+\eta_l\right]_{j}\left[t_{a'}^{(c)}+\eta_l\right]_{i}(\mu_j\gamma_i)\bigsmile{\mu_j\gamma_i-\mu_i\gamma_j}.\]
The sum of these two terms is 
\[\gamma_i^{-1}\gamma_j^{-1}\lambda_{i} \lambda_{j}\left[t_{a'}^{(c)}+\eta_l\right]_{i}\left[t_{a'}^{(c)}+\eta_l\right]_{j}\bigsmile{\mu_i\gamma_j-\mu_j\gamma_i}^2\geq 0.\]
Therefore, we proved equation \eqref{eqn:200}. We will show next that 
\begin{align}
    &\sum_{i}\mu_i^2\frac{\lambda_i\log(1/\delta_l)}{\gamma_i^3 n}+\sum_{c}\nu_c\sum_{i}\sum_{j}\mu_i\mu_j\frac{1}{2}\left(\sum_{a^{\prime} \in \A} \frac{\lambda_{i}\left[t_{a}^{(c)}+\eta_l\right]_{i}}{\sqrt{(\lambda\odot\gamma)^\T (t_a^{(c)}+\eta_l)}}\right)\nonumber\\
    &\cdot\left(\sum_{a \in \A} \frac{\lambda_{j}\left[t_{a}^{(c)}+\eta_l\right]_{j}}{\sqrt{(\lambda\odot\gamma)^\T (t_a^{(c)}+\eta_l)}}\right)\geq 0.\label{eqn:220}
\end{align}
By similar calculation, we can obtain that the above simplifies to
\begin{align*}
    &\sum_{c}\nu_c\sum_i \mu_i\gamma_i^{-1}\left(\sum_{a^{\prime} \in \A} \frac{\lambda_{i}\left[t_{a'}^{(c)}+\eta_l\right]_{i}}{\sqrt{(\lambda\odot\gamma)^\T (t_{a'}^{(c)}+\eta_l)}}\right)\\
    &\cdot\bigbrace{\mu_i\sum_{a\in\A}\frac{\sum_j \lambda_j\gamma_j[t_a^{(c)}+\eta_l]_j}{\sqrt{(\lambda\odot\gamma)^\T (t_{a}^{(c)}+\eta_l)}}+\mu_j\gamma_i\sum_{a\in\A}\frac{\sum_j \lambda_j[t_a^{(c)}+\eta_l]_j}{\sqrt{(\lambda\odot\gamma)^\T (t_{a}^{(c)}+\eta_l)}}}.
\end{align*}
We can show that the sum of the above is positive by similar techniques for showing \eqref{eqn:200}. 
Plugging equation~\ref{eqn:200} and \ref{eqn:220} in equation~\ref{eqn:190}, we have that 
\begin{align*}
    \mu^\T M\mu &\geq  \sum_{i}\mu_i^2\frac{\lambda_i\log(1/\delta_l)}{\gamma_i^3 n}\geq \frac{\lambda_{\min}\log(1/\delta_l)}{\gamma_{\max}^3 n},
\end{align*}
so the Hessian is positive-definite. 
\end{proof}

Note that the minimum eigenvalue of the Hessian at the stationary point is $\frac{\lambda_{\min}\log(1/\delta_l)}{\gamma_{\max}^3 n}>0$, we can extend the result in Lemma~\ref{lem:local_strongly_convex} to $\alpha$-stationary points, where $\alpha<\frac{\lambda_{\min}\log(1/\delta_l)}{\gamma_{\max}^3 n}$, and still maintain local strong convexity. 
\begin{lemma}\label{lem:local_strong_convex_eps}
Consider some fixed $\lambda\in\triangle_\Pi$ and $n$. Assume $\gamma_\alpha$ is an $\alpha$-stationary point of $h_l(\lambda,\gamma,n)$, where $\alpha=\frac{\lambda_{\min}\log(1/\delta_l)}{2\gamma_{\max}^3 n}$, then $h_l(\lambda,\gamma,n)$ is locally strongly convex at $\gamma_\alpha$, i.e. for $L_{\He}=\frac{\lambda_{\min}\log(1/\delta_l)}{2\gamma_{\max}^3 n}$, there exists $\epsilon>0$ such that for all $\gamma\in B_\epsilon(\gamma_\alpha)$, $h_l(\lambda,\gamma,n)\geq h_l(\lambda,\gamma_\alpha,n)+\frac{L_{\He}}{2}\norm{\gamma-\gamma_\alpha}^2$. 
\end{lemma}
\begin{proof}
The proof follows almost identically from that of Lemma~\ref{lem:local_strongly_convex}. Note that the $\alpha$-stationary point ensures that $\norm{\nabla_{\gamma} h_l(\lambda,\gamma)}_1 \leq \alpha$, so equation~\ref{eqn:211} is rewritten as
\begin{equation}\label{eqn:212}
    \sum_i\left|\sum_{c\in \mc{D}}\nu_{c_\D}\bigsmile{\sum_{a\in\A}\sqrt{(\lambda\odot\gamma)^\T (t_a^{(c)}+\eta_l)}}\cdot\bigsmile{\sum_{a'\in\A}\frac{\lambda_i([t_{a'}^{(c)}]_i+\eta_l)}{\sqrt{(\lambda\odot\gamma)^\T (t_{a'}^{(c)}+\eta_l)}}}-\frac{\lambda_i\log(1/\delta_l)}{\gamma_i^2 n}\right|\leq \alpha.
\end{equation}
Therefore, for any $\mu$ we can still use the same trick and get 
\[\mu^\T M\mu\geq\sum_{i}\mu_i^2\frac{\lambda_i\log(1/\delta_l)}{\gamma_i^3 n}-\alpha\geq \frac{\lambda_{\min}\log(1/\delta_l)}{2\gamma_{\max}^3 n},\]
so our result follows. 
\end{proof}

\subsection{Proof of strong duality}

In this section, we would like to show that strong duality holds. We first show that the primal problem is convex for $w$. 

\begin{lemma}\label{lem:primal_convex}
The primal problem \eqref{eqn:C_1} is convex for $w$. 
\end{lemma}
\begin{proof}
Note that the primal problem could be written as 
\[\min_{w\in\Omega} c\qquad \textrm{s.t. }\forall\pi\in\Pi,-\Delta(\pi)+\sqrt{\frac{\norm{\phi_{\pi}-\phi_{\pi_*}}_{A(w)^{-1}}^2}{n}}\leq c. \]
Therefore, we consider the function $f(w):=-\Delta(\pi)+\sqrt{\frac{\norm{\phi_{\pi}-\phi_{\pi_*}}_{A(w)^{-1}}^2}{n}}$ for some $\pi\in\Pi$. Note that to show that $f(w)=-\Delta(\pi)+\sqrt{\frac{\norm{\phi_{\pi}-\phi_{\pi_*}}_{A(w)^{-1}}^2}{n}}$ is convex for $w$, it is equivalent to show that $g(w):=\sqrt{\norm{\phi_{\pi}-\phi_{\pi_*}}_{A(w)^{-1}}^2}$ is convex for $w$. Note that 
\begin{align*}
    g(w)&=\sqrt{\sum_{a,c}\nu_c^2w_{a,c}^{-1}(\1\{\pi(c)=a,\pi_*(c)\ne a\}+\1\{\pi(c)\ne a,\pi_*(c)=a\})}\\
    &= \sqrt{\sum_{a,c,t_a^{(c)}=1}\nu_c^2w_{a,c}^{-1}}.
\end{align*}
So restricting to $a,c$ such that $t_a^{(c)}=1$
\begin{align*}
    \frac{\partial g(w)}{\partial w_{a,c}} &= \frac{1}{2\sqrt{\sum_{a,c,t_a^{(c)}=1}\nu_c^2w_{a,c}^{-1}}}\cdot(-\nu_c^2 w_{a,c}^{-2}),
\end{align*}
and 
    \[\frac{\partial^2 g(w)}{\partial w_{a,c}^2} = -\frac{1}{4\bigsmile{\sum_{a,c,t_a^{(c)}=1}\nu_c^2w_{a,c}^{-1}}^{3/2}}\cdot(-\nu_c^2 w_{a,c}^{-2}\cdot-\nu_c^2 w_{a,c}^{-2})+\frac{1}{\sqrt{\sum_{a,c,t_a^{(c)}=1}\nu_c^2w_{a,c}^{-1}}}\cdot \nu_c^2w_{a,c}^{-3}\]
    \[
    \frac{\partial^2 g(w)}{\partial w_{a_1,c_1}\partial w_{a_2,c_2}} = -\frac{1}{4\bigsmile{\sum_{a,c,t_a^{(c)}=1}\nu_c^2w_{a,c}^{-1}}^{3/2}}\cdot(-\nu_{c_1}^2 w_{a_1,c_1}^{-2}\cdot-\nu_{c_2}^2 w_{a_2,c_2}^{-2})
    \]
Denote the Hessian as $M$. Then, for any vector $\mu\in\R^{|\A|\times|\C|}$ with $\norm{\mu}_2=1$, we have 
\begin{align*}
    \mu^\T M\mu &= -\frac{1}{4}\sum_{a,c,t_a^{(c)}=1}\sum_{a',c',t_{a'}^{(c')}=1} \mu_{a,c}\mu_{a',c'} \bigsmile{\sum_{a,c,t_a^{(c)}=1}\nu_c^2w_{a,c}^{-1}}^{-3/2}\nu_{c}^2\nu_{c'}^2 w_{a,c}^{-2} w_{a',c'}^{-2}\\
    &\qquad+\sum_{a,c,t_a^{(c)}=1}\mu_{a,c}^2\nu_c^2w_{a,c}^{-3}\bigsmile{\sum_{a,c,t_a^{(c)}=1}\nu_c^2w_{a,c}^{-1}}^{-1/2}.
\end{align*}
To show that this is nonnegative, it is equivalent to show that 
{
\small
\begin{align*}
    &-\frac{1}{4}\sum_{a,c,t_a^{(c)}=1}\sum_{a',c',t_{a'}^{(c')}=1} \mu_{a,c}\mu_{a',c'} \nu_{c}^2\nu_{c'}^2 w_{a,c}^{-2} w_{a',c'}^{-2}+\sum_{a,c,t_a^{(c)}=1}\mu_{a,c}^2\nu_c^2w_{a,c}^{-3}\bigsmile{\sum_{a',c',t_{a'}^{(c')}=1}\nu_{c'}^2w_{a',c'}^{-1}}\geq 0,
\end{align*}
}
which is equivalent to show that
\begin{align}
    &\sum_{a,c,t_a^{(c)}=1}\sum_{a',c',t_{a'}^{(c')}=1} -\mu_{a,c}\mu_{a',c'} \nu_{c}^2\nu_{c'}^2 w_{a,c}^{-2} w_{a',c'}^{-2}+\mu_{a,c}^2\nu_c^2w_{a,c}^{-3}\nu_{c'}^2w_{a',c'}^{-1}\geq 0.\label{eqn:23}
\end{align}
Note that 
\begin{align*}
    &-\mu_{a,c}\mu_{a',c'} \nu_{c}^2\nu_{c'}^2 w_{a,c}^{-2} w_{a',c'}^{-2}+\mu_{a,c}^2\nu_c^2w_{a,c}^{-3}\nu_{c'}^2w_{a',c'}^{-1}\\
    &= \mu_{a,c}w_{a,c}^{-3}w_{a',c'}^{-2}\nu_c^2\nu_{c'}^2(\mu_{a,c}w_{a',c'}-\mu_{a',c'}w_{a,c})\\
    &= w_{a,c}^{-3}w_{a',c'}^{-3}\nu_c^2\nu_{c'}^2(\mu_{a,c}w_{a',c'})(\mu_{a,c}w_{a',c'}-\mu_{a',c'}w_{a,c}).
\end{align*}
Then, exchanging the label of $a$ and $a'$, we also get a term like $$w_{a',c'}^{-3}w_{a,c}^{-3}\nu_{c'}^2\nu_c^2(\mu_{a',c'}w_{a,c})(\mu_{a',c'}w_{a,c}-\mu_{a,c}w_{a',c'}).$$ The sum of these two terms is 
\begin{align*}
    &w_{a',c'}^{-3}w_{a,c}^{-3}\nu_{c'}^2\nu_c^2(\mu_{a',c'}w_{a,c})(\mu_{a',c'}w_{a,c}-\mu_{a,c}w_{a',c'})\\
    &+w_{a,c}^{-3}w_{a',c'}^{-3}\nu_c^2\nu_{c'}^2(\mu_{a,c}w_{a',c'})(\mu_{a,c}w_{a',c'}-\mu_{a',c'}w_{a,c})\\
    &= w_{a',c'}^{-3}w_{a,c}^{-3}\nu_{c'}^2\nu_c^2(\mu_{a',c'}w_{a,c}-\mu_{a,c}w_{a',c'})(\mu_{a',c'}w_{a,c}-\mu_{a,c}w_{a',c'})\\
    &= w_{a',c'}^{-3}w_{a,c}^{-3}\nu_{c'}^2\nu_c^2(\mu_{a',c'}w_{a,c}-\mu_{a,c}w_{a',c'})^2\geq 0.
\end{align*}
Therefore, equation~\ref{eqn:23} becomes 
\begin{align*}
    &\sum_{a,c,t_a^{(c)}=1}\sum_{\substack{a',c'\\ t_{a'}^{(c')}=1\\ (a',c')>(a,c)}}(w_{a',c'}^{-3}w_{a,c}^{-3}\nu_{c'}^2\nu_c^2(\mu_{a',c'}w_{a,c})(\mu_{a',c'}w_{a,c}-\mu_{a,c}w_{a',c'})\\
    &\qquad+w_{a,c}^{-3}w_{a',c'}^{-3}\nu_c^2\nu_{c'}^2(\mu_{a,c}w_{a',c'})(\mu_{a,c}w_{a',c'}-\mu_{a',c'}w_{a,c}))\\
    &= \sum_{a,c,t_a^{(c)}=1}\sum_{\substack{a',c'\\ t_{a'}^{(c')}=1\\ (a',c')>(a,c)}}w_{a',c'}^{-3}w_{a,c}^{-3}\nu_{c'}^2\nu_c^2(\mu_{a',c'}w_{a,c}-\mu_{a,c}w_{a',c'})^2\geq 0.
\end{align*}
Since the above holds for any vector $\mu$, the Hessian is positive-semidefinite, and so the function $g(w)$ is convex for $w$. 
\end{proof}

\begin{lemma}\label{lem:strong_duality}
In the optimization problem~\ref{eqn:C_1}, the strong duality holds, i.e. 
{
\small
\[\min_{w\in\Omega}\max_{\pi\in\Pi}\left(-\Delta(\pi)+\sqrt{\frac{\norm{\phi_{\pi}-\phi_{\pi_*}}_{A(w)^{-1}}^2}{n}}\right)=\max_{\lambda\in\triangle_\Pi}\min_{w\in\Omega}\sum_{\pi\in\Pi}\lambda_\pi\left(-\Delta(\pi)+\sqrt{\frac{\norm{\phi_{\pi}-\phi_{\pi_*}}_{A(w)^{-1}}^2}{n}}\right).\]
}
\end{lemma}
\begin{proof}
By Lemma~\ref{lem:primal_convex}, the primal problem is convex for $w$, so it is left to check the KKT conditions. Note that the lagrangian is 
\[\mathcal{L}(w,\lambda,c)=c+\sum_{\pi\in\Pi}\lambda_\pi\cdot\left(-\Delta(\pi)+\sqrt{\frac{\norm{\phi_{\pi}-\phi_{\pi_*}}_{A(w)^{-1}}^2}{n}}-c\right).\]
Let $h_\pi(w)=-\Delta(\pi)+\sqrt{\frac{\norm{\phi_{\pi}-\phi_{\pi_*}}_{A(w)^{-1}}^2}{n}}-c$. At an optimal solution $w^*$ and $\lambda^*$, we would like to show that
\begin{align*}
    \sum_{\pi\in\Pi}\lambda_\pi^* h_\pi(w^*)=0.
\end{align*}
We prove this by contradiction. If there is some $\pi$ such that $\lambda_\pi>0$ and $h_\pi(w^*)<0$. Then we could find another $\lambda'\in\triangle_\Pi$ that places zero mass on this $\pi$ and thus get a larger objective, so we get a contradiction. The other conditions follow from the optimality of $w^*$ and $\lambda^*$.  
\end{proof}

\section{Useful lemmas}
In this section, we state several algebraic facts of our function, which serves as the key to derive convergence as well as complexity.

\begin{lemma}\label{lem:B_3}
For any $l$, 
\[\min_{w\in\Omega}\max_{\pi\in\Pi}\frac{\norm{\phi_{\pill}-\phi_\pi}_{A(w)^{-1}}^2}{\Delta(\pi)^2}=\min_{p_c\in\triangle_\A,\forall c\in\C}\max_{\pi\in\Pi}\frac{\E_{c\sim\nu}\left[\bigsmile{\frac{1}{p_{c,\pill(c)}}+\frac{1}{p_{c,\pi(c)}}}\1\{\pill(c)\ne\pi(c)\}\right]}{\Delta(\pi)^2}.\]
\end{lemma}
\begin{proof}
Let $w_{a,c}=\nu_c p_{c,a}$ for some $p_c\in\triangle_\A$. Then, for any $\pi\in\Pi$, 
\begin{align*}
    &\frac{1}{\Delta(\pi)^2}\norm{\phi_{\pill}-\phi_\pi}_{A(w)^{-1}}^2\\
    &=\frac{1}{\Delta(\pi)^2}\sum_{a,c}\frac{\nu_c^2}{w_{a,c}}\bigsmile{\1\{\pill(c)=a,\pi(c)\ne a\}+\1\{\pill(c)\ne a,\pi(c)=a\}}\\
    &=\frac{1}{\Delta(\pi)^2}\sum_{a,c}\frac{\nu_c}{p_{c,a}}\bigsmile{\1\{\pill(c)=a,\pi(c)\ne a\}+\1\{\pill(c)\ne a,\pi(c)=a\}}\\
    &= \frac{1}{\Delta(\pi)^2}\sum_{c}\nu_c\bigsmile{\frac{1}{p_{c,\pill(c)}}+\frac{1}{p_{c,\pi(c)}}}\1\{\pill(c)\ne\pi(c)\}\\
    &=\frac{1}{\Delta(\pi)^2}\E_{c\sim\nu}\left[\bigsmile{\frac{1}{p_{c,\pill(c)}}+\frac{1}{p_{c,\pi(c)}}}\1\{\pill(c)\ne\pi(c)\}\right].
\end{align*}
Therefore, \[\min_{w\in\Omega}\max_{\pi\in\Pi}\frac{\norm{\phi_{\pill}-\phi_\pi}_{A(w)^{-1}}^2}{\Delta(\pi)^2}=\min_{p_c\in\triangle_\A,\forall c\in\C}\max_{\pi\in\Pi}\frac{\E_{c\sim\nu}\left[\bigsmile{\frac{1}{p_{c,\pill(c)}}+\frac{1}{p_{c,\pi(c)}}}\1\{\pill(c)\ne\pi(c)\}\right]}{\Delta(\pi)^2}.\]

\end{proof}

\begin{lemma}\label{lem:gg}
For any $l$, any $\lambda\in\triangle_\Pi$, $\gamma>0$, and any $n$, we have $h_l(\lambda,\gamma,n)=\inner{\lambda}{\nabla_\lambda h_l(\lambda,\gamma,n)}$.
\end{lemma}
\begin{proof}
We first compute 
\begin{align*}
    \left[\nabla_\lambda h_l(\lambda,\gamma,n)\right]_\pi&=-\hat{\Delta}_{l-1}^{\gamma^{l-1}}(\pi,\hat{\pi}_{l-1})+\frac{\log(1/\delta_l)}{\gamma_\pi n}\\
    &+\E_{c\sim\nu_\D}\bigbrak{\bigsmile{\sum_{a\in\A}\sqrt{(\lambda\odot \gamma)^\T (t_a^{(c)}+\eta_l)}}\bigsmile{\sum_{a'\in\A}\frac{\gamma_\pi(t_{a'}^{(c)}+\eta_l)_\pi}{\sqrt{(\lambda\odot \gamma)^\T (t_{a'}^{(c)}+\eta_l)}}}}.
\end{align*}
Then, by the fact that 
\begin{align*}
    &\sum_{\pi\in\Pi}\lambda_\pi\cdot\E_{c\sim\nu_\D}\bigbrak{\bigsmile{\sum_{a\in\A}\sqrt{(\lambda\odot \gamma)^\T (t_a^{(c)}+\eta_l)}}\bigsmile{\sum_{a'\in\A}\frac{\gamma_\pi(t_{a'}^{(c)}+\eta_l)_\pi}{\sqrt{(\lambda\odot \gamma)^\T (t_{a'}^{(c)}+\eta_l)}}}}\\
    &= \E_{c\sim\nu_\D}\bigbrak{\bigsmile{\sum_{a\in\A}\sqrt{(\lambda\odot \gamma)^\T (t_a^{(c)}+\eta_l)}}\bigsmile{\sum_{a'\in\A}\frac{(\lambda\odot\gamma)^\T(t_{a'}^{(c)}+\eta_l)}{\sqrt{(\lambda\odot \gamma)^\T (t_{a'}^{(c)}+\eta_l)}}}}\\
    &= \E_{c\sim\nu_\D}\bigbrak{\bigsmile{\sum_{a\in\A}\sqrt{(\lambda\odot \gamma)^\T (t_a^{(c)}+\eta_l)}}^2},
\end{align*}
we have
\begin{align*}
    &\inner{\lambda}{\nabla_\lambda h_l(\lambda,\gamma,n)}\\
    &= \sum_{\pi\in\Pi}\lambda_\pi\left[\nabla_\lambda h_l(\lambda,\gamma,n)\right]_\pi\\
    &= \sum_{\pi\in\Pi}\lambda_\pi\cdot\bigsmile{-\hat{\Delta}_{l-1}^{\gamma^{l-1}}(\pi,\hat{\pi}_{l-1})+\frac{\log(1/\delta_l)}{\gamma_\pi n}}\\
    &\qquad+\sum_{\pi\in\Pi}\lambda_\pi\E_{c\sim\nu_\D}\bigbrak{\bigsmile{\sum_{a\in\A}\sqrt{(\lambda\odot \gamma)^\T (t_a^{(c)}+\eta_l)}}\bigsmile{\sum_{a'\in\A}\frac{\gamma_\pi(t_{a'}^{(c)}+\eta_l)_\pi}{\sqrt{(\lambda\odot \gamma)^\T (t_{a'}^{(c)}+\eta_l)}}}}\\
    &= \sum_{\pi\in\Pi}\lambda_\pi\cdot\bigsmile{-\hat{\Delta}_{l-1}^{\gamma^{l-1}}(\pi,\hat{\pi}_{l-1})+\frac{\log(1/\delta_l)}{\gamma_\pi n}}+\E_{c\sim\nu_\D}\bigbrak{\bigsmile{\sum_{a\in\A}\sqrt{(\lambda\odot \gamma)^\T (t_a^{(c)}+\eta_l)}}^2}\\
    &= h_l(\lambda,\gamma,n).
\end{align*}
\end{proof}

\begin{lemma}\label{lem:E_9}
For any $\lambda\in\triangle_\Pi$ and $\gamma\in\left[0,\min\left\{\sqrt{\frac{\log(1/\delta_l)}{2n_l\E_c[\1\{\pi(c)\ne\pi^*(c)\}]}},\sqrt{\frac{\log(1/\delta_l)}{|\A|^2\eta_l n_l}}\right\}\right]^{\Pi}$, with $\eta_l=|\A|^{-4}\epsilon_l^{2}$, we have 
\[0\leq\E_{c\sim\nu}\bigbrak{\bigsmile{\sum_{a\in\A}\sqrt{(\lambda\odot \gamma)^\T (t_a^{(c)}+\eta_l)}}^2}-\mathbb{E}_{c\sim\nu}\left[\left(\sum_{a} \sqrt{\left(\lambda \odot \gamma\right)^{\top} t_a^{(c)}}\right)^{2}\right]\leq \epsilon_l.\]
\end{lemma}
\begin{proof} The first inequality is clear since $\eta_l>0$ and $\lambda_\pi,\gamma_\pi\geq 0$ for all $\pi\in\Pi$, so we focus on the upper bound.  Note that
\begin{align}
    &\E_{c\sim\nu}\bigbrak{\bigsmile{\sum_{a\in\A}\sqrt{(\lambda\odot \gamma)^\T (t_a^{(c)}+\eta_l)}}^2}-\mathbb{E}_{c}\left[\left(\sum_{a} \sqrt{\left(\lambda \odot \gamma\right)^{\top} t_a^{(c)}}\right)^{2}\right]\nonumber\\
    &= \E_{c\sim\nu}\bigbrak{\sum_{a\in\A}(\lambda\odot\gamma)^\T(t_a^{(c)}+\eta_l)+\sum_{a_1\in\A}\sum_{a_2\in\A}\sqrt{(\lambda\odot\gamma)^\T(t_{a_1}^{(c)}+\eta_l)(t_{a_2}^{(c)}+\eta_l)^\T(\lambda\odot\gamma)}}\nonumber\\
    &\qquad-\E_{c\sim\nu}\bigbrak{\sum_{a\in\A}(\lambda\odot\gamma)^\T t_a^{(c)}+\sum_{a_1\in\A}\sum_{a_2\in\A}\sqrt{(\lambda\odot\gamma)^\T t_{a_1}^{(c)}{t_{a_2}^{(c)}}^\T(\lambda\odot\gamma)}}.\label{eqn:20}
\end{align}
Note that 
\begin{align*}
    &\E_{c\sim\nu}\bigbrak{\sum_{a_1\in\A}\sum_{a_2\in\A}\sqrt{(\lambda\odot\gamma)^\T(t_{a_1}^{(c)}+\eta_l)(t_{a_2}^{(c)}+\eta_l)^\T(\lambda\odot\gamma)}}\\
    &= \E_{c\sim\nu}\bigbrak{\sum_{a_1\in\A}\sum_{a_2\in\A}\sqrt{(\lambda\odot\gamma)^\T t_{a_1}^{(c)}(t_{a_2}^{(c)})^\T(\lambda\odot\gamma)+\eta_l\lambda^\T\gamma(\lambda\odot\gamma)^\T (t_{a_1}^{(c)}+t_{a_2}^{(c)})+\eta_l^2(\lambda^\T\gamma)^2}}\\
    &\leq \E_{c\sim\nu}\bigbrak{\sum_{a_1\in\A}\sum_{a_2\in\A}\sqrt{(\lambda\odot\gamma)^\T t_{a_1}^{(c)}(t_{a_2}^{(c)})^\T(\lambda\odot\gamma)}}\\
    &\qquad+2|\A|\E_{c\sim\nu}\bigbrak{\sum_{a\in\A}\sqrt{\eta_l\lambda^\T\gamma(\lambda\odot\gamma)^\T t_{a}^{(c)}}}+|\A|^2\eta_l\lambda^\T\gamma.
\end{align*}
Then \eqref{eqn:20} is upper bounded by
\begin{align}
    &\E_{c\sim\nu}\bigbrak{\sum_{a\in\A}\eta_l\lambda^\T\gamma}+2|\A|\E_{c\sim\nu}\bigbrak{\sum_{a\in\A}\sqrt{\eta_l\lambda^\T\gamma(\lambda\odot\gamma)^\T t_{a}^{(c)}}}+|\A|^2\eta_l\lambda^\T\gamma\nonumber\\
    &=|\A|\eta_l\lambda^\T\gamma+|\A|^2\eta_l\lambda^\T\gamma+2|\A|\sqrt{\eta_l\lambda^\T\gamma}\E_{c\sim\nu}\bigbrak{\sum_{a\in\A}\sqrt{\sum_{\pi\in\Pi}\lambda_\pi\gamma_\pi [t_{a}^{(c)}]_\pi}}\nonumber\\
    &=|\A|\eta_l\lambda^\T\gamma+|\A|^2\eta_l\lambda^\T\gamma+2|\A|^2\sqrt{\eta_l\lambda^\T\gamma}\E_{c\sim\nu}\bigbrak{\sum_{a\in\A}\frac{1}{|\A|}\sqrt{\sum_{\pi\in\Pi}\lambda_\pi\gamma_\pi [t_{a}^{(c)}]_\pi}}\nonumber\\
    &=|\A|\eta_l\lambda^\T\gamma+|\A|^2\eta_l\lambda^\T\gamma+2|\A|^2\sqrt{\eta_l\lambda^\T\gamma}\E_{c\sim\nu}\bigbrak{\E_{a\sim\mu}\left[\sqrt{\sum_{\pi\in\Pi}\lambda_\pi\gamma_\pi [t_{a}^{(c)}]_\pi}\right]}\nonumber\\
    &\leq |\A|\eta_l\lambda^\T\gamma+|\A|^2\eta_l\lambda^\T\gamma+2|\A|^2\sqrt{\eta_l\lambda^\T\gamma}\sqrt{\sum_{\pi\in\Pi}\lambda_\pi\gamma_\pi\frac{1}{|\A|}\E_{c\sim\nu}\left[\sum_{a\in\A}[t_{a}^{(c)}]_\pi\right]}\nonumber\\
    &= |\A|\eta_l\lambda^\T\gamma+|\A|^2\eta_l\lambda^\T\gamma+2|\A|^2\sqrt{\eta_l\lambda^\T\gamma}\sqrt{\sum_{\pi\in\Pi}\lambda_\pi\gamma_\pi\frac{1}{|\A|}2\cdot\E_{c\sim\nu}[\1\{\pi(c)\ne\pi^*(c)\}]}.\label{eqn:21}
\end{align}
Since $\gamma_\pi\leq \sqrt{\frac{\log(1/\delta_l)}{2n_l\E_c[\1\{\pi(c)\ne\pi^*(c)\}]}}$, $\gamma_\pi\E_{c\sim\nu}[\1\{\pi(c)\ne\pi^*(c)\}]\leq \sqrt{\frac{\E_c[\1\{\pi(c)\ne\pi^*(c)\}]\log(1/\delta_l)}{2n_l}}\leq \sqrt{\frac{\log(1/\delta_l)}{2n_l}}$. We know from the lower bound argument that 
\begin{align*}
    n_l&\gtrsim \min_{w\in\Omega}\max_{\pi\in\Pi}\frac{\norm{\phi_\pi-\phi_{\pi_*}}_{A(w)^{-1}}^2}{\Delta(\pi)^2+\epsilon_l^2}\log(1/\delta_l)\geq\epsilon_l^{-1}\log(1/\delta_l),
\end{align*}
so $\sqrt{\frac{\log(1/\delta_l)}{2n}}\lesssim \sqrt{\epsilon_l}$. Therefore, \eqref{eqn:21} is upper bounded by 
\begin{align}
    (|\A|+|\A|^2)\eta_l\lambda^\T\gamma+2|\A|^{3/2}\sqrt{\epsilon_l\eta_l\lambda^\T\gamma}.\label{eqn:28}
\end{align}
Since $\eta_l\lambda^\T\gamma\leq \eta_l\gamma_{\max}=\sqrt{\frac{\eta_l\log(1/\delta_l)}{|\A|^2 n_l}}\leq \sqrt{\eta_l}\frac{1}{|\A|}$. Plugging this as well as $\eta_l\leq |\A|^{-4}\epsilon_l^2$ in equation~\ref{eqn:28} gives that the bias is upper bounded by $\epsilon_l$. 

\end{proof}

\end{document}